\documentclass{article}

\PassOptionsToPackage{numbers, compress}{natbib}
\usepackage[final]{neurips_2024}

\usepackage[utf8]{inputenc} % allow utf-8 input
\usepackage[T1]{fontenc}    % use 8-bit T1 fonts
\usepackage{hyperref}       % hyperlinks
\usepackage{url}            % simple URL typesetting
\usepackage{booktabs}       % professional-quality tables
\usepackage{amsfonts}       % blackboard math symbols
\usepackage{nicefrac}       % compact symbols for 1/2, etc.
\usepackage{microtype}      % microtypography
\usepackage{xcolor}         % colors

\usepackage{mathtools}
\usepackage{amsmath}
\usepackage{amssymb}
\usepackage{amsthm}
\usepackage{algorithm}
\usepackage{algorithmic}

\usepackage[capitalize,noabbrev]{cleveref}
\usepackage{hhline}
\usepackage{tabularx} 
\usepackage{lipsum}
\usepackage{graphicx}
\usepackage{subfigure}
\usepackage{comment}
\usepackage{enumerate}
\usepackage{bbm}

\theoremstyle{plain}
\newtheorem{theorem}{Theorem}[section]

\newtheorem{lemma}[theorem]{Lemma}

\newtheorem{assumption}[theorem]{Assumption}

\crefname{assumption}{Assumption}{Assumptions}

\usepackage{amsmath,amsfonts,bm}

\def\1{\bm{1}}

\definecolor{brickred}{rgb}{0.8, 0.25, 0.33}
\definecolor{darkred}{rgb}{0.75,0.0,0.0}
\definecolor{darkblue}{rgb}{0.0,0.0,0.75}

\def\gdphi{{\nabla\Phi}}
\def\gdx{{\nabla_{x}}}
\def\gdy{{\nabla_{y}}}
\def\gdxy{{\nabla_{xx}^2}}
\def\gdxy{{\nabla_{xy}^2}}
\def\gdyy{{\nabla_{yy}^2}}

\def\gdf{{\bar{\nabla}f}}

\def\pr{{\mathrm{Pr}}}

\def\init{{\mathrm{init}}}

\def\rmA{{\mathbf{A}}}
\def\rmB{{\mathbf{B}}}

\def\rmI{{\mathbf{I}}}

\def\rmP{{\mathbf{P}}}
\def\rmQ{{\mathbf{Q}}}

\def\rmU{{\mathbf{U}}}

\def\rmX{{\mathbf{X}}}

\def\vw{{\bm{w}}}
\def\vx{{\bm{x}}}

\def\vz{{\bm{z}}}

\DeclareMathAlphabet{\mathsfit}{\encodingdefault}{\sfdefault}{m}{sl}
\SetMathAlphabet{\mathsfit}{bold}{\encodingdefault}{\sfdefault}{bx}{n}

\def\gD{{\mathcal{D}}}
\def\gE{{\mathcal{E}}}
\def\gF{{\mathcal{F}}}

\def\gH{{\mathcal{H}}}

\newcommand{\E}{\mathbb{E}}

\newcommand{\R}{\mathbb{R}}

\DeclareMathOperator*{\argmin}{arg\,min}

\title{An Accelerated Algorithm for Stochastic Bilevel Optimization under Unbounded Smoothness}

\author{%
	Xiaochuan Gong \quad\quad Jie Hao \quad\quad Mingrui Liu\thanks{Corresponding Author.} \\
	Department of Computer Science\\
	George Mason University\\
	\texttt{\{xgong2, jhao6, mingruil\}@gmu.edu} \\
}

\begin{document}

	\maketitle
	
	% \vspace*{-0.1in}
	\begin{abstract}
    \vspace*{-0.1in} 
		This paper investigates a class of stochastic bilevel optimization problems where the upper-level function is nonconvex with potentially unbounded smoothness and the lower-level problem is strongly convex. These problems have significant applications in sequential data learning, such as text classification using recurrent neural networks. The unbounded smoothness is characterized by the smoothness constant of the upper-level function scaling linearly with the gradient norm, lacking a uniform upper bound. Existing state-of-the-art algorithms require $\widetilde{O}(\epsilon^{-4})$ oracle calls of stochastic gradient or Hessian/Jacobian-vector product to find an $\epsilon$-stationary point. However, it remains unclear if we can further improve the convergence rate when the assumptions for the function in the population level also hold for each random realization almost surely (e.g., Lipschitzness of each realization of the stochastic gradient). To address this issue, we propose a new Accelerated Bilevel Optimization algorithm named AccBO. The algorithm updates the upper-level variable by normalized stochastic gradient descent with recursive momentum and the lower-level variable by the stochastic Nesterov accelerated gradient descent algorithm with averaging. We prove that our algorithm achieves an oracle complexity of $\widetilde{O}(\epsilon^{-3})$ to find an $\epsilon$-stationary point, when the lower-level stochastic gradient has a small variance $O(\epsilon)$. Our proof relies on a novel lemma characterizing the dynamics of stochastic Nesterov accelerated gradient descent algorithm under distribution drift with high probability for the lower-level variable, which is of independent interest and also plays a crucial role in analyzing the hypergradient estimation error over time. Experimental results on various tasks confirm that our proposed algorithm achieves the predicted theoretical acceleration and significantly outperforms baselines in bilevel optimization. The code is available \href{https://github.com/MingruiLiu-ML-Lab/Accelerated-Bilevel-Optimization-Unbounded-Smoothness}{here}.
	\end{abstract}
	
	\vspace*{-0.1in}
	\section{Introduction}
	
	Bilevel optimization receives tremendous attention recently in the machine learning community, due to its applications in meta-learning~\cite{franceschi2018bilevel,rajeswaran2019meta}, hyperparameter optimization~\cite{franceschi2018bilevel,feurer2019hyperparameter}, data hypercleaning~\cite{ji2021bilevel}, continual learning~\cite{borsos2020coresets,hao2023bilevelcoreset}, and reinforcement learning~\cite{konda1999actor}. The bilevel optimization problem has the following formulation:
	\begin{equation}
		\label{eq:bp}
		\begin{aligned}
			\min_{x\in \mathbb{R}^{d_x}} \Phi(x):=f(x, y^*(x)), 
			\ \
			\text{s.t.}, \ \ y^*(x) \in \mathop{\arg \min}_{y\in \mathbb{R}^{d_y}} g(x, y),
			\vspace*{-0.15in}
		\end{aligned}
	\end{equation}%
	where $f$ and $g$ are upper-level and lower-level functions respectively. For example, in meta-learning~\cite{finn2017model,franceschi2018bilevel}, $x$ denotes the layers of neural networks for shared representation learning, $y$ denotes the task-specific head encoded in the last layer, and the formulation~\eqref{eq:bp} aims to learn the a common representation learning encoder $x$ such that it can be quickly adapted to downstream tasks by only updating the task-specific head $y$. In machine learning, people typically consider stochastic optimization setting such that  $f(x,y)=\mathbb{E}_{\xi\sim\mathcal{D}_f}\left[F(x,y;\xi)\right]$ and $g(x,y)=\mathbb{E}_{\zeta\sim\mathcal{D}_g}\left[G(x,y;\zeta)\right]$, where $\mathcal{D}_f$ and $\mathcal{D}_g$ are the underlying unknown data distributions for $f$ and $g$ respectively, and one can access noisy observations of $f$ and $g$ based on sampling from $\mathcal{D}_f$ and $\mathcal{D}_g$. 
	
	There emerges a wave of studies for algorithmic design and analysis for solving the bilevel optimization problem~\eqref{eq:bp} under different assumptions of $f$ and $g$. Most theoretical work assumes the upper-level function is smooth (i.e., gradient is Lipschitz) and nonconvex, and the lower-level function is strongly convex~\cite{ghadimi2018approximation,ji2021bilevel,hong2023two,grazzi2022bilevel,kwon2023fully}. However, as pointed out by~\cite{zhang2019gradient,crawshaw2022robustness}, certain neural networks such as recurrent neural networks~\cite{elman1990finding}, long-short term memory networks~\cite{hochreiter1997long} and transformers~\cite{vaswani2017attention} have smoothness constants that scale with gradient norm, potentially leading to unbounded smoothness constants (i.e., gradient Lipschitz constant can be infinity). Motivated by this, Hao et al.~\cite{hao2024bilevel} designed the first bilevel optimization algorithm to handle the cases  where $f$ is nonconvex with potentially unbounded smoothness and $g$ is strongly convex. The algorithm in~\cite{hao2024bilevel} achieves $\widetilde{O}(\epsilon^{-4})$ oracle complexity for finding an $\epsilon$-stationary point (i.e., a point $x$ such that $\|\nabla \Phi(x)\|\leq \epsilon$). Gong et al.~\cite{gong2024a} proposed an single-loop algorithm under the same setting as in~\cite{hao2024bilevel} and also achieved $\widetilde{O}(\epsilon^{-4})$ oracle complexity. This complexity result is worse than the $\widetilde{O}(\epsilon^{-3})$ oracle complexity under the relatively easier setting where $f$ has a Lipschitz gradient, and each realization of the stochastic oracle calls is Lipschitz with respect to its argument (e.g., almost-sure Lipschitz oracle)~\cite{yang2021provably,khanduri2021near,cutkosky2019momentum,guo2021randomized,hu2023blockwise}. This naturally motivates us to study the following question:
	
	\textbf{Is it possible to improve the $\widetilde{O}(\epsilon^{-4})$ oracle complexity for bilevel optimization problems where the upper-level function is nonconvex with unbounded smoothness and the lower-level function is strongly convex, by assuming that the properties of the function at the population level also hold almost surely for each random realization?}

	In this paper, we give a positive answer to this question by designing a new  algorithm named AccBO with an improved oracle complexity of $\widetilde{O}(\epsilon^{-3})$, when the lower-level stochastic gradient has a small variance $O(\epsilon)$.  Our algorithm is inspired by momentum-based variance reduction techniques used in nonconvex smooth optimization~\cite{cutkosky2019momentum} under the almost-sure Lipschitz stochastic gradient oracle framework. The innovation of AccBO lies in its update rules: it employs normalized stochastic gradient descent with recursive momentum for the upper-level variable and stochastic Nesterov accelerated gradient descent with averaging for the lower-level variable. Our approach differs from existing accelerated bilevel optimization algorithms, such as those proposed by~\cite{yang2021provably,khanduri2021near} in two key ways: (i) while these algorithms use recursive momentum for the upper-level variable update, AccBO utilizes normalized recursive momentum to address the unbounded smoothness of the upper-level function; (ii) for the lower-level variable update, we use stochastic Nesterov accelerated gradient descent with averaging, in contrast to the recursive momentum method used by the other algorithms. The primary challenge in analyzing the convergence rate of AccBO arises from the need to simultaneously control errors from both upper-level and lower-level variables, given the unbounded smoothness, large learning rate, and recursive momentum in the upper-level problem. Our main contributions are summarized as follows.

	\begin{itemize}
		\item We design a new  algorithm named AccBO for solving bilevel optimization problems where the upper-level function is nonconvex with unbounded smoothness and the lower-level function is strongly convex. AccBO leverages normalized recursive momentum for the upper-level variable and Nesterov momentum for the lower-level variable under the stochastic setting to achieve acceleration. To the best of our knowledge, the simultaneous usage of these two techniques in stochastic bilevel optimization is novel and has not been previously explored in the bilevel optimization literature.
		
		\item We prove that the AccBO algorithm requires $\widetilde{O}(\epsilon^{-3})$ oracle calls for finding an $\epsilon$-stationary point, when the variance of the lower-level stochastic gradient is $O(\epsilon)$. This complexity strictly improves the state-of-the-art oracle complexity for unbounded smooth nonconvex upper-level problem and strongly-convex lower-level problem as described in~\cite{hao2024bilevel,gong2024a}~\footnote{Note that even if the lower-level stochastic gradient variance is $O(\epsilon)$, the oracle complexity required in~\cite{hao2024bilevel,gong2024a} is still $\widetilde{O}(\epsilon^{-4})$ since the oracle complexity required for their upper-level problem is already $\widetilde{O}(\epsilon^{-4})$.}. To achieve this result, we introduce novel proof techniques for analyzing the dynamics of stochastic Nesterov accelerated gradient descent under distribution drift with high probability for the lower-level variable, which are crucial for analyzing the hypergradient error and also of independent interest.
		
		\item We empirically verify the effectiveness of our proposed algorithm on various tasks, including deep AUC maximization and data hypercleaning. Our algorithm indeed achieves the predicted theoretical acceleration and significantly outperforms baselines in bilevel optimization.
	\end{itemize}

		\vspace*{-0.1in}

	\section{Related Work}
	
	\textbf{Relaxed Smoothness.} The concept of relaxed smoothness was initially introduced by~\cite{zhang2019gradient}, inspired by the loss landscapes observed in recurrent neural networks and long-short term memory networks. They show that techniques such as gradient clipping and normalization could improve the performance compared with gradient descent in these scenarios. It inspired further investigations that concentrate on various aspects, including improved analysis on gradient clipping and normalization~\cite{zhang2020improved,jin2021non}, adaptive algorithms~\cite{crawshaw2022robustness,li2023convergence,wang2023convergence,faw2023beyond}, federated algorithms~\cite{liu2022communication,crawshaw2023episode,crawshaw2023federated}, generalized assumptions~\cite{crawshaw2022robustness,chen2023generalized}, and recursive momentum based methods with faster rates~\cite{reisizadeh2023variance,liu2023near}. The work of~\cite{hao2024bilevel,gong2024a} considered a relaxed smoothness condition for the upper-level problem in the bilevel optimization setting.

	\textbf{Bilevel Optimization.} Bilevel optimization refers to a special kind of optimization where one problem is embedded within another. It was first introduced by~\cite{bracken1973mathematical}. Early works developed specific bilevel optimization algorithms with asymptotic convergence analysis~\cite{vicente1994descent,anandalingam1990solution,white1993penalty}. Ghadimi and Wang~\cite{ghadimi2018approximation} initiated the study of non-asymptotic convergence for gradient-based methods in bilevel optimization where the upper-level problem is smooth and the lower-level problem is strongly convex. This field saw further advancements with improved complexity results~\cite{hong2023two,ji2021bilevel,chen2021closing,dagreou2022framework,chen2023optimal} and fully first-order algorithms~\cite{kwon2023fully,liu2022bome}. There is a line of work which leverages almost-sure Lipschitz oracle (e.g., stochastic gradient) to obtain improved convergence rates of bilevel optimization algorithms~\cite{yang2021provably,khanduri2021near}. When the lower-level function is not strongly convex, several algorithmic framework and approximation schemes were proposed~\cite{sabach2017first,sow2022constrained,kwon2023penalty,shen2023penalty,liu2020generic,chen2023bilevel1}. The setting considered in this paper is very close to~\cite{hao2024bilevel,gong2024a}, where the upper-level function is nonconvex and unbounded smooth, and the lower-level function is strongly convex. However, the work of~\cite{hao2024bilevel,gong2024a} do not have an accelerated rate $\widetilde{O}(\epsilon^{-3})$ for finding an $\epsilon$-stationary point as established in this paper.
	
	\textbf{Nesterov Accelerated Gradient and Variants.}  
	Nesterov Accelerated Gradient (NAG) method was introduced by~\cite{nesterov1983method} for deterministic convex optimization problems. The stochastic version of NAG (SNAG) was extensively studied in the literature~\cite{assran2020convergence,aybat2020robust,vaswani2022towards,chen2023convergence}. To the best of our knowledge, none of them provide a high probability analysis for SNAG. In the online learning setting, there is a line of work focusing on the perspective of sequential stochastic/online optimization with distributional drift~\cite{besbes2015non,wilson2018adaptive,madden2021bounds,cutler2023stochastic}. While these studies provide valuable insights into adaptive techniques and performance bounds under distributional drift, they do not explore the potential integration of such methods with bilevel optimization problems, nor do they consider the application of SNAG within this framework.
	
		\vspace*{-0.1in}

	\section{Problem Setup and Preliminaries}
	
	Define $\langle\cdot,\cdot\rangle$ and $\|\cdot\|$ as the inner product and the Euclidean norm. Throughout the paper, we use asymptotic notation $\widetilde{O}(\cdot)$, $\widetilde{\Theta}(\cdot)$, $\widetilde{\Omega}(\cdot)$ to hide polylog factors in $\epsilon^{-1}$ and $1/\delta$. Denote $f$: $\mathbb{R}^{d_x}\times\mathbb{R}^{d_y} \rightarrow \mathbb{R}$ as the upper-level function, and $g$: $\mathbb{R}^{d_x}\times\mathbb{R}^{d_y} \rightarrow \mathbb{R}$ as the lower-level function. The hypergradient $\gdphi(x)$ has the following form~\cite{ghadimi2018approximation}:
	% \begin{small}
		\begin{equation}
			\gdphi(x) = \gdx f(x,y^*(x)) - \gdxy g(x,y^*(x))\left[\nabla_{yy}^{2} g(x,y^*(x))\right]^{-1}\gdy f(x,y^*(x)).
		\end{equation}
	% \end{small}%
	To avoid the Hessian inverse computation, we typically use the following Neumann series to approximate the hypergradient~\cite{ghadimi2018approximation,khanduri2021near,hong2023two}. In particular, for the stochastic setting, define
	% \begin{small}
		\begin{equation*}
			\gdf(x,y;\Bar{\xi}) = \gdx F(x,y;\xi) - \frac{Q}{l_{g,1}}\gdxy G(x,y;\zeta^{(0)})\prod_{i=1}^{\mathsf{q}(Q)}\left(I - \frac{\gdyy G(x,y;\zeta^{(i)})}{l_{g,1}}\right)\gdy F(x,y;\xi),
		\end{equation*}
	% \end{small}%
	where $\mathsf{q}(Q)\sim \mathrm{Uniform}\{0,\dots,Q-1\}$, $\bar{\xi}\coloneqq \{\xi,\zeta^{(0)},\dots,\zeta^{(\mathsf{q}(Q))}\}$ and we use the convention that $\prod_{i=1}^{j}A_i=I$ if $j=0$. Then $\mathbb{E}_{\bar{\xi}}[\gdf(x,y;\bar{\xi})]$ is a good approximation of $\gdphi(x)$ if $y$ and $y^*(x)$ are close~\cite{ghadimi2018approximation}.
	
	Throughout the paper, we make the following assumptions.
	\begin{assumption}[$(L_{x,0}, L_{x,1}, L_{y,0}, L_{y,1})$-smoothness~\cite{hao2024bilevel}] \label{ass:relax-smooth}
		Let $z=(x,y)$ and $z'=(x',y')$, there exists $L_{x,0}, L_{x,1}, L_{y,0}, L_{y,1} > 0$ such that for all $z,z'$, if $\|z-z'\| \leq 1/\sqrt{2(L_{x,1}^2+L_{y,1}^2)}$, then $ \|\gdx f(z)-\gdx f(z')\| \leq (L_{x,0}+L_{x,1}\|\gdx f(z)\|)\|z-z'\|$ and $\|\gdy f(z)-\gdy f(z')\| \leq (L_{y,0}+L_{y,1}\|\gdy f(z)\|)\|z-z'\|$.
        % \begin{equation*}
        %     \begin{aligned}
        %         \|\gdx f(z)-\gdx f(z')\| \leq (L_{x,0}+L_{x,1}\|\gdx f(z)\|)\|z-z'\| \\
        %         \|\gdy f(z)-\gdy f(z')\| \leq (L_{y,0}+L_{y,1}\|\gdy f(z)\|)\|z-z'\|
        %     \end{aligned}
        % \end{equation*}
	\end{assumption}
	
	\textbf{Remark}: \cref{ass:relax-smooth} is introduced by~\cite{hao2024bilevel} for describing the bilevel optimization problems with recurrent neural networks. This assumption can be regarded as a block-wise relaxed smoothness assumptions for two blocks $x$ and $y$, which is a variant of the relaxed smoothness assumption~\cite{zhang2019gradient} and the coordinate-wise relaxed smooth assumption~\cite{crawshaw2022robustness}.

	\begin{assumption} \label{ass:f-and-g}
		Suppose the followings hold for objective functions $f$ and $g$: (i)  $f$ is continuously differentiable and $(L_{x,0}, L_{x,1}, L_{y,0}, L_{y,1})$-smooth in $(x,y)$; (ii)  For every $x$, $\|\gdy f(x,y)\| \leq l_{f,0}$ for all $y$; (iii) For every $x$, $g(x, y)$ is $\mu$-strongly-convex in $y$ for $\mu>0$; (iv) $g$ is $l_{g,1}$-smooth jointly in $(x, y)$; (v) $g$ is twice continuously differentiable, and $\gdxy g, \gdyy g$ are $l_{g,2}$-Lipschitz jointly in $(x,y)$.
	\end{assumption}
	
	\textbf{Remark}: \cref{ass:f-and-g} is standard in the bilevel optimization literature~\cite{kwon2023fully,hao2024bilevel,ghadimi2018approximation}. \cref{ass:f-and-g} (i) characterizes the unbounded smoothness of the upper-level function and is empirically observed in recurrent neural networks~\cite{hao2024bilevel}. 
	
	\begin{assumption} \label{ass:noise}
		The following stochastic estimators are unbiased and have the following properties:
		\begin{small}
			\begin{equation*}
				\E_{\xi\sim\gD_f}[\|\gdx F(x,y;\xi)-\gdx f(x,y)\|^2] \leq \sigma_{f,1}^2, \quad
				\E_{\xi\sim\gD_f}[\|\gdy F(x,y;\xi)-\gdy f(x,y)\|^2] \leq \sigma_{f,1}^2,
			\end{equation*}
			\begin{equation*}
				\Pr(\|\gdy G(x,y;\xi)-\gdy g(x,y)\| \geq \lambda) \leq 2\exp(-2\lambda^2/\sigma_{g,1}^2) \quad \forall \lambda>0,
			\end{equation*}
			\begin{equation*}
				\E_{\zeta\sim\gD_g}[\|\gdxy G(x,y;\zeta)-\gdxy g(x,y)\|^2] \leq \sigma_{g,2}^2, \quad
				\E_{\zeta\sim\gD_g}[\|\gdyy G(x,y;\zeta)-\gdyy g(x,y)\|^2] \leq \sigma_{g,2}^2.
			\end{equation*}
		\end{small}%
	\end{assumption}
	
	\textbf{Remark:} \cref{ass:noise} assumes the stochastic oracle for the upper-level problem has bounded variance, which is standard in nonconvex stochastic optimization~\cite{ghadimi2013stochastic,ghadimi2016accelerated,ghadimi2018approximation}. It also assumes the stochastic oracle for the lower-level problem is light-tailed, which is common for the high probability analysis for the lower-level problem~\cite{lan2012optimal,hazan2014beyond}. Note that the same assumption is also made in~\cite{hao2024bilevel,gong2024a} for the bilevel problems with a unbounded smooth upper-level function.
	
	\begin{assumption} \label{ass:individual-noise}
		$F(x,y;\xi)$ and $G(x,y;\zeta)$ satisfy \cref{ass:f-and-g} for every $\xi$ and $\zeta$ almost surely.
	\end{assumption}
	
	\textbf{Remark:} \cref{ass:individual-noise} assumes that each random realization of the upper- and lower-level functions satisfies the same property as in the population level. Note that this condition is the key to achieve an improved $\widetilde{O}(\epsilon^{-3})$ oracle complexity under various settings, including both single-level nonconvex smooth problems~\cite{fang2018spider,cutkosky2019momentum,tran2019hybrid} and bilevel problems with nonconvex smooth upper-level objectives~\cite{yang2021provably,khanduri2021near}. Furthermore, this assumption is shown to be necessary for achieving improved oracle complexity in single-level problems~\cite{arjevani2023lower}.
    % and also used in the bilevel setting for obtaining accelerated rates~\cite{yang2021provably,khanduri2021near}. 

		\vspace*{-0.1in}

	\section{Algorithm and Analysis} 
	\label{sec:algorithmdesign}

	\subsection{Main Challenges and Algorithm Design}
	\textbf{Main Challenges.}  We begin by explaining why existing bilevel optimization algorithms and their corresponding analysis techniques are insufficient in our setting. First, most algorithms developed for bilevel optimization require the upper-level function is smooth (i.e., the gradient of the upper-level function is Lipschitz)~\cite{ghadimi2018approximation,ji2021bilevel,hong2023two,yang2021provably,khanduri2021near,dagreou2022framework,kwon2023fully}. They characterize the estimation error of the optimal solution for the lower-level problem, utilize an approximate hypergradient descent approach and the descent lemma for $L$-smooth functions to prove the convergence. In particular, they demonstrate that a potential function, incorporating both the function value and the bilevel error from the lower-level problem, progressively decreases in expectation. However, when the upper-level function is $(L_{x,0}, L_{x,1}, L_{y,0}, L_{y,1})$-smooth as illustrated in \cref{ass:relax-smooth}, the previous algorithms and analyses relying on $L$-smoothness do not work. The reason is that the hypergradient bias depends on the approximation error of the lower-level variable as well as the hypergradient itself: these elements are statistically dependent and the standard potential function argument with an expectation-based analysis would not work. To address this issue, the work of~\cite{hao2024bilevel,gong2024a} requires a careful high probability analysis in the unbounded smoothness setting and obtains $\widetilde{O}(\epsilon^{-4})$ oracle complexity. Such a requirement of high probability analysis prevents us from leveraging the momentum-based variance reduction technique for updating the lower-level variable. For example, the work~\cite{khanduri2021near} which has $\widetilde{O}(\epsilon^{-3})$ oracle complexity in the smooth case leverages the momentum-based variance reduction technique~\cite{cutkosky2019momentum} for updating the lower-level variable with an expectation-based analysis, but it seems difficult to establish a high probability analysis for the momentum-based variance reduction algorithm in terms of the lower-level variable. Second, the recent work of Hao et al.~\cite{hao2024bilevel} and Gong et al.~\cite{gong2024a} considered that the upper-level function is unbounded smooth and addressed this issue by performing normalized stochastic gradient with momentum for the upper-level variable and periodic updates or stochastic gradient descent for the lower-level variable, but their oracle complexity is not better than $\widetilde{O}(\epsilon^{-4})$. These facts indicate that we need new algorithm design and analysis techniques to get potential acceleration.
	
	\textbf{Algorithm Design.} To obtain potential acceleration and enable a high probability analysis for the lower-level variable, our key idea is to update the upper-level variable by normalized stochastic gradient descent with recursive momentum and update the lower-level variable by the stochastic Nesterov accelerated gradient (SNAG) method. Different from~\cite{hao2024bilevel,gong2024a}, the key innovation of our algorithm design is that we achieve acceleration for both upper-level and lower-level problems simultaneously but without affecting each other. The upper-level update rule can be regarded as a generalization of the acceleration technique (e.g., the momentum-based variance reduction technique)~\cite{cutkosky2019momentum,liu2023near} from single-level to bilevel problems. The main challenge is that we need to deal with the accumulated error of the recursive momentum over time due to the hypergradient bias, which is caused by the inaccurate estimation of the optimal lower-level variable. Therefore we require a very small tracking error between the iterate of the lower-level variable and the optimal lower-level solution defined by the upper-level variable (i.e., $y^*(x)$) at every iteration. This requirement is satisfied by executing SNAG method under the distribution drift for the lower-level problem, where the drift is caused by the change of the upper-level variable over time. Note that we can provide a high probability analysis of the SNAG method under distributional drift, which strictly improves the analysis of SGD under distributional drift in~\cite{cutler2021stochastic} in the small stochastic gradient noise regime. 
	
	\begin{algorithm}[t]
		\caption{\textsc{Stochastic Nesterov Accelerated Gradient Method (SNAG)}} \label{alg:sgd}
		\begin{algorithmic}[1]
			\STATE \textbf{Input:} $x, \Tilde{y}_{-1}, \Tilde{y}_0, \tilde{\alpha}, T_0$ \hfill \# \texttt{SNAG}$(x, \Tilde{y}_0, \tilde{\alpha}, T_0)$
			\FOR{$t=0, 1, \dots, T_0-1$}
			\STATE Sample $\Tilde{\pi}_t$ from distribution $\gD_g$
			\STATE $\Tilde{z}_t = \Tilde{y}_t + \gamma(\Tilde{y}_t-\Tilde{y}_{t-1})$
			\STATE $\Tilde{y}_{t+1} = \Tilde{z}_t - \tilde{\alpha}\gdy G(x,\Tilde{z}_t;\Tilde{\pi}_t)$
			\ENDFOR
		\end{algorithmic}
	\end{algorithm}
	
	\begin{algorithm}[t]
		\caption{\textsc{Accelerated Bilevel Optimization algorithm (AccBO)}} \label{alg:bilevel}
		\begin{algorithmic}[1]
			\STATE \textbf{Input:} $\alpha^{\init}, \alpha, \beta, \gamma, \eta, \tau, I, S, T_0, T$, set $x_0, y_0^{\init} = 0$
			\STATE $y_0 = \texttt{SNAG}(x_0, y_0^{\init}, \alpha^{\init}, T_0)$, and set $y_{-1} = \hat{y}_0 = y_0$ \hfill \textcolor{blue}{\# Warm-start} 
			\FOR{$t=0, 1, \dots, T-1$}
			\STATE Sample $\mathsf{q}(Q) \sim \mathrm{Uniform}\{0,\dots,Q-1\}$ and $\{\zeta_{t,s}^{(0)}, \dots, \zeta_{t,s}^{(\mathsf{q}(Q))}\}_{s=1}^{S} \sim \gD_g$
			\STATE Sample $\{\xi_{t,s}\}_{s=1}^{S} \sim \gD_f$, denote $\Bar{\xi}_t \coloneqq \cup_{s=1}^{S}\{\mathsf{q}(Q), \xi_{t,s}, \zeta_{t,s}^{(0)}, \dots, \zeta_{t,s}^{(\mathsf{q}(Q))}\}$
			\STATE \textcolor{blue}{\# Lower-Level: Stochastic Nesterov Accelerated Gradient Descent with Averaging}
			\STATE \textcolor{blue}{\# Option I: from Line $8\sim 9$ (for quadratic lower-level function)}
			\STATE $z_t = y_t + \gamma(y_t-y_{t-1})$
			\STATE $y_{t+1} = z_t - \alpha\gdy G(x_t,z_t;\pi_t)$, where $\pi_t\sim\gD_g$
			\STATE \textcolor{blue}{\# Option II: from Line $11\sim 20$ (for general strongly convex lower-level function)}
			\IF{$t>0$ and $t$ is a multiple of $I$}
			\STATE Set $y_t^0 = y_t^{-1} = y_t$ 
			\FOR{$j=0,1,\dots,N-1$}
			\STATE $z_t^j = y_t^j + \gamma(y_t^j-y_t^{j-1})$
			\STATE $y_t^{j+1} = z_t^{j} - \alpha\gdy G(x_t,z_t^{j};\pi_t^{j})$, where $\pi_t^{j}\sim\gD_g$
			\ENDFOR
			\STATE $y_{t+1} = y_t^{N+1}$
			\ELSE
			\STATE $y_{t+1} = y_t$
			\ENDIF
			\STATE $\hat{y}_{t+1} = (1-\tau)\hat{y}_t + \tau y_{t+1}$  \hfill \textcolor{blue}{\# Averaging} 
			\STATE \textcolor{blue}{\# Upper-Level: Normalized Stochastic Gradient Descent with Recursive Momentum}
			\STATE $m_t = \beta m_{t-1} + (1-\beta)\gdf(x_t,\hat{y}_t;\Bar{\xi}_t) + \beta(\gdf(x_t,\hat{y}_t;\Bar{\xi}_t) - \gdf(x_{t-1},\hat{y}_{t-1};\Bar{\xi}_t))$ if $t\geq1$ else $m_0 = \gdf(x_0,\hat{y}_0;\Bar{\xi}_0)$
			\STATE $x_{t+1} = x_t - \eta\frac{m_t}{\|m_t\|}$ 
			\ENDFOR
		\end{algorithmic}
	\end{algorithm}
	
	The detailed description of our algorithm is illustrated in Algorithm~\ref{alg:bilevel}. At the very beginning, we run a certain number of iterations of SNAG for the fixed upper-level variable $x_0$ (line $2$) as the warm-start stage, and then update the lower-level variable by SNAG  (line $8\sim 20$) with averaging (line $21$) and update the upper-level variable by normalized stochastic gradient descent with recursive momentum (line $23\sim 24$). Note that we have two options for implementing SNAG. In Option I (line $8\sim 9$), the algorithm simply runs SNAG under distribution drift caused by the sequence $\{x_t\}$. Option I is specifically designed for a particular subset of bilevel optimization problems where the  lower-level function is a quadratic function. Option II (line $11\sim 20$) is designed for a broader range of bilvel optimization problems, accommodating general strongly-convex lower-level functions. In Option II, we run SNAG with periodic updates: the lower-level update is performed for $N$ iterations only when the iteration number $t$ is a multiple of $I$.

	\subsection{Main Results}
	
	We first introduce some useful notations. Let $\sigma(\cdot)$ be the $\sigma$-algebra generated by the random variables in the argument. We define the following filtrations for $t\geq1$: $\gF^{\init} = \sigma(\tilde{\pi}_0,\dots,\tilde{\pi}_{T_0-1})$, $\gF_t = \sigma(\bar{\xi}_0,\dots,\Bar{\xi}_{t-1})$, $\widetilde{\gF}_t^1 = \sigma(\pi_0,\dots,\pi_{t-1})$, and we also define $\widetilde{\gF}_t^2 = \sigma(\pi_t^0,\dots,\pi_t^{N-1})$ when $t$ is a multiple of $I$. We use $\E_t$, $\E_{\gF_t}$ and $\E$ to denote the conditional expectation $\E[\cdot \mid \gF_t]$, the expectation over $\gF_t$ and the total expectation over $\gF_T$ respectively. 
	
	\begin{theorem} \label{thm:main}
		Suppose \cref{ass:relax-smooth,ass:f-and-g,ass:noise,ass:individual-noise} hold. Let $\{x_t\}$ be the iterates produced by \Cref{alg:bilevel}. For any given $\delta\in(0,1)$ and small enough $\epsilon$ (see exact choice in \eqref{eq:thm-eps}), if $\sigma_{g,1} = O(\sqrt{\epsilon})$ as defined in \eqref{eq:sigma-g1}, and we set parameters $\alpha^{\init}, \alpha, \beta, \gamma, \eta, \tau, I, N, S, Q, T_0$ (see exact choices in \eqref{eq:thm-beta}, \eqref{eq:thm-alpha}, \eqref{eq:thm-T0}, \eqref{eq:thm-I}, and \eqref{eq:thm-S}) as
		\begin{small}
			\begin{equation*}
				\alpha^{\init} = \widetilde{\Theta}(\epsilon^4),
				\quad
				\alpha = \widetilde{\Theta}(\epsilon^2),
				\quad
				1-\beta = \widetilde{\Theta}(\epsilon^2),
				\quad
				\eta = \widetilde{\Theta}(\epsilon^2),
				\quad
				\tau = \widetilde{\Theta}(\epsilon),
				\quad
				\gamma = O(1),
			\end{equation*}
			\begin{equation*}
                    % \sigma_{g,1} = O(\sqrt{\epsilon}),
                    % \quad
				T_0 = \widetilde{O}(\epsilon^{-2}),
				\quad
				I = \widetilde{O}(\epsilon^{-1}),
				\quad
				N = \widetilde{O}(\epsilon^{-1}),
				\quad
				Q = \widetilde{O}(1),
				\quad
				S = \widetilde{O}(1).
			\end{equation*}
		\end{small}%
		Then with probability at least $1-2\delta$ over the randomness in $\sigma(\gF^{\init}\cup\widetilde{\gF}_T^1)$ (for Option I) or $\sigma(\gF^{\init}\cup(\cup_{t\leq T}\widetilde{\gF}_t^2))$ (for Option II), \Cref{alg:bilevel} guarantees $\frac{1}{T}\sum_{t=1}^{T}\E\|\gdphi(x_t)\|\leq 20\epsilon$ within $T = \frac{4d_0}{\eta\epsilon} =\widetilde{O}(\epsilon^{-3})$ iterations, where $d_0\coloneqq \Phi(x_0)-\inf_{x}\Phi(x)$ and the expectation is taken over the randomness over $\gF_T$. For Option I, it requires $T_0+SQT= \widetilde{O}(\epsilon^{-3})$ oracle calls of stochastic gradient or Hessian/Jacobian vector product. For Option II, it requires $T_0+\frac{NT}{I}+SQT= \widetilde{O}(\epsilon^{-3})$ oracle calls of stochastic gradient or Hessian/Jacobian vector product.
	\end{theorem}
	
	\textbf{Remark:} Theorem~\ref{thm:main} established an improved $\widetilde{O}(\epsilon^{-3})$ oracle complexity for finding an $\epsilon$-stationary point when the standard deviation of lower-level stochastic gradient is small: $\sigma_{g,1}=O(\sqrt{\epsilon})$. This complexity result strictly improves the $\widetilde{O}(\epsilon^{-4})$ obtained by~\cite{hao2024bilevel,gong2024a} when the upper-level function is nonconvex and unbounded smooth. This complexity result also matches that in the single-level unbounded smooth setting~\cite{liu2023near} and is nearly optimal in terms of the dependency on $\epsilon$~\cite{arjevani2023lower}. The full statement of Theorem~\ref{thm:main} is included in Theorem~\ref{thm:main-appendix}.

	\subsection{Proof Sketch}
	
	In this section, we provide a roadmap of proving Theorem~\ref{thm:main} and the main steps. The detailed proofs can be found in \Cref{sec:proof-sketch} and~\ref{sec:proof-main}. The key idea is to prove two things: (1) the lower-level iterate is very close to the optimal lower-level variable at every iteration; (2) two consecutive iterates of the lower-level iterates are close to each other. In particular, define $y_t^*=y^*(x_t)$, and we aim to prove that $\|\hat{y}_t-y_t^*\|\leq O(\epsilon)$ and $\|\hat{y}_{t+1}-\hat{y}_t\|\leq O(\epsilon^2)$ for every $t$. These two requirements are essential to control the hypergradient estimation error (i.e., $\|m_t- \gdphi(x_t)\|$) caused by inaccurate estimate of the lower-level problem. \cref{lem:averaging} provides the guarantee for the lower-level problem, and \cref{lem:uppererror} characterizes the hypergradient estimation error. Equipped with these two lemmas, we can adapt the momentum-based variance reduction techniques~\cite{cutkosky2019momentum,liu2023near} to the upper-level problem and prove the main theorem.
	
	The main technical contribution of this paper is to provide a general framework for proving the convergence of SNAG under distributional drift in \cref{sec:nesterov}, which can be leveraged as a tool to control the lower-level error in bilevel optimization and derive the \cref{lem:averaging}, as illustrated in \cref{sec:nagbilevel}. In particular, we can regard the change of the upper-level variable $x$ at each iteration as the distributional drift for the lower-level problem: the drift is small due to the normalization operator of the upper-level update rule and also the Lipschitzness of $y^*(x)$. Once we have the general lemma for tracking the minimizer for any fixed distributional drift over time, this lemma can be applied to our algorithm analysis and establish guarantees for the bilevel problem.
	
	\subsubsection{Stochastic Nesterov Accelerated Gradient Descent under Distributional Drift}
	\label{sec:nesterov}
	In this section, we study the sequences of stochastic optimization problems $\min_{w\in\R^d} \phi_t(w)$ indexed by time $t\in\mathbb{N}$. We denote the minimizer and the minimal value of $\phi_t$ as $w_t^*$ and $\phi_t^*$, and we define the \textit{minimizer drift} at time $t$ to be $\Delta_t\coloneqq \|w_t^*-w_{t+1}^*\|$. With a slight abuse of notation~\footnote{The notation in \cref{sec:nesterov} is independent of that in other sections, although there may be incidental overlaps in terminology.}, we consider the SNAG algorithm applied to the sequence $\{\phi_t\}_{t=1}^{T}$, where $T$ is the total number of iterations:
	\begin{equation}
		\label{eq:NAG:single1}
		\begin{aligned}
			z_t &= w_t + \gamma(w_t-w_{t-1})  \\
			w_{t+1} &= w_t + \gamma(w_t-w_{t-1}) - \alpha g_t, 
		\end{aligned}
	\end{equation}
	where $g_t=\nabla \phi_t(z_t;\xi_t)$ is the stochastic gradient evaluated at $z_t$ with random sample $\xi_t$. Define $\varepsilon_t=g_t-\nabla \phi_t(z_t)$ as the stochastic gradient noise at $t$-th iteration. Define $\mathcal{H}_t=\sigma(\xi_1,\ldots,\xi_{t-1})$ as the filtration, which is the $\sigma$-algebra generated by all random variables until $t$-th iteration. We make the following assumption, which is the same as Assumption 3 in~\cite{cutler2023stochastic} for high probability analysis.
	
	\begin{assumption} \label{ass:single-level}
		Function $\phi_t : \R^d \rightarrow \R$ is $\mu$-strongly convex and $L$-smooth for constants $\mu,L > 0$. Also, there exists constants $\Delta,\sigma>0$ such that the drift $\Delta_t^2$ is sub-exponential conditioned on $\gH_t$ with parameter $\Delta^2$ and the noise $\varepsilon_t$ is norm sub-Gaussian conditioned on $\gH_t$ with parameter $\sigma/2$.
	\end{assumption}

	\begin{lemma} \label{lm:option-1-single}
		Suppose \Cref{ass:single-level} holds and let $\{w_t\}$ be the iterates produced by the update rule~\eqref{eq:NAG:single1} with constant learning rate $\alpha\leq 1/25L$, and set $\gamma = \frac{1-\sqrt{\mu\alpha}}{1+\sqrt{\mu\alpha}}$. Define $\theta_t = [(w_t-w_t^*)^{\top}, (w_{t-1}-w_t^*)^{\top}]^{\top} \in \R^{2d}$, and the potential function $V_t$ as
        \begin{equation*}
            V_t = \theta_t^{\top}\rmP\theta_t + \phi_t(w_t) - \phi_t(w_t^*), 
            \quad\text{where}\quad
            \rmP = \frac{1}{2\alpha}
            \begin{bmatrix}
                1 & \sqrt{\mu\alpha}-1 \\
                \sqrt{\mu\alpha}-1 & (1-\sqrt{\mu\alpha})^2
            \end{bmatrix} \otimes \rmI_d.
        \end{equation*}
        Then for any given $\delta\in(0,1)$ and all $t\geq0$, the following holds with probability at least $1-\delta$ over the randomness in $\gH_t$ (here $e$ denotes the base of natural logarithms): 
        % (i) With drift: Let $\phi_t(w) \coloneqq \frac{\mu}{2}\|w-w_t^*\|^2$ and $w\in\mathbb{R}$, then $ V_t \leq \left(1-\frac{\sqrt{\mu\alpha}}{4}\right)^tV_0 + \left(2\alpha\sigma^2 + \frac{72\Delta^2}{\alpha}\right)\ln\frac{eT}{\delta}$. (ii) Without drift: Let $\phi_t(w)\equiv \phi(w)$ be any general functions and $w\in\mathbb{R}^d$ with $\Delta=0$, then $V_t \leq \left(1-\frac{\sqrt{\mu\alpha}}{4}\right)^tV_0 + 2\alpha\sigma^2\ln\frac{eT}{\delta}$. 
        \begin{enumerate}[(i)]
            \item (With drift) Let $\phi_t(w) \coloneqq \frac{\mu}{2}\|w-w_t^*\|^2$, then $V_t \leq \left(1-\frac{\sqrt{\mu\alpha}}{4}\right)^tV_0 + \left(\frac{5\sqrt{\alpha}\sigma^2}{\sqrt{\mu}} + \frac{80\Delta^2}{\alpha}\right)\ln\frac{eT}{\delta}$.
            % \begin{equation*}
            %     V_t \leq \left(1-\frac{\sqrt{\mu\alpha}}{4}\right)^tV_0 + \left(\frac{5\sqrt{\alpha}\sigma^2}{\sqrt{\mu}} + \frac{80\Delta^2}{\alpha}\right)\ln\frac{eT}{\delta}.
            % \end{equation*}
            \item (Without drift) Let $\phi_t(w)\equiv \phi(w)$ be any general strongly convex functions with $\Delta=0$, then $V_t \leq \left(1-\frac{\sqrt{\mu\alpha}}{4}\right)^tV_0 + \frac{5\sqrt{\alpha}\sigma^2}{\sqrt{\mu}} \ln\frac{eT}{\delta}$.
            % \begin{equation*}
            %     V_t \leq \left(1-\frac{\sqrt{\mu\alpha}}{4}\right)^tV_0 + \frac{5\sqrt{\alpha}\sigma^2}{\sqrt{\mu}} \ln\frac{eT}{\delta}.
            % \end{equation*}
        \end{enumerate}
	\end{lemma}
	
	\textbf{Remark}: When $\{\phi_t\}_{t=1}^{T}$ is a sequence of quadratic functions with moving minimizers, \Cref{lm:option-1-single} provides a high probability tracking guarantee for SNAG with distributional drift, which is useful to provide guarantees for Option I in \cref{alg:bilevel}. Note that this guarantee strictly improves the guarantee of stochastic gradient descent with distributional drift (e.g., \cite[Theorem 6]{cutler2023stochastic}) \emph{in the small stochastic gradient noise regime} and therefore is of independent interest. In particular, for small $\alpha$, the decaying factor in the first term is improved from $1-\frac{\mu\alpha}{2}$ to $1-\frac{\sqrt{\mu\alpha}}{4}$, the drift term is improved from $\frac{\Delta^2}{\mu\alpha^2}$ to $\frac{\Delta^2}{\alpha}$, and the variance term becomes a bit worse (from $\alpha\sigma^2$ to $\frac{\sqrt{\alpha}\sigma^2}{\sqrt{\mu}}$). When $\sigma$ is small enough, the variance term becomes insignificant compared with the drift term, then \cref{lm:option-1-single} provides an improved convergence rate with high probability. To the best of our knowledge, such an improved guarantee for SNAG with distributional drift is first shown in this work. When there is no drift, \Cref{lm:option-1-single} also provides a high probability guarantee for SNAG. It holds for any smooth and strongly convex function $\phi$, and it is useful to provide guarantees for Option II of Algorithm~\ref{alg:bilevel}.

    %\textcolor{red}{add more discussions here about the condition number improvement.}
	
	\subsubsection{Application of Stochastic Nesterov Accelerated Gradient to Bilevel Optimization}
	\label{sec:nagbilevel}
	Inspired by~\cite{hao2024bilevel,gong2024a}, we can regard $\phi_t(\cdot)$ as $\phi_t(\cdot):=g(x_t,\cdot)$ in the bilevel setting, and then we have $\Delta_t=\eta l_{g,1}/\mu$ for every $t$ due to the upper-level update rule and the Lipschitzness of $y^*(x)$. Therefore we can focus on the high probability analysis on the lower-level variable without worrying about the randomness from the upper-level. Throughout, we assume \Cref{ass:relax-smooth},~\ref{ass:f-and-g},~\ref{ass:noise} and~\ref{ass:individual-noise} hold. In addition, the failure probability $\delta\in(0,1)$ and $\epsilon>0$ are chosen in the same way as in \Cref{thm:main}.
	
	\begin{lemma}[Warm-start]
		\label{lem:warmstart}
		Let $\{y_t^{\init}\}$ be the iterates produced by line 2 of \Cref{alg:bilevel}. Set $\alpha^{\init}=\widetilde{\Theta}(\epsilon^4)$, $\sigma_{g,1}=(\mu\alpha)^{1/4}\tilde{\sigma}_{g,1}$, and $\phi_t(y)\equiv g(x_0,y)$. Then $\|y_{T_0}^{\init}-y_0^*\| \leq \sqrt{\frac{\mu\alpha}{32}}\frac{\epsilon}{L_0}$ holds with probability at least $1-\delta$ over the randomness in $\widetilde{\mathcal{F}}^{\init}$ (we denote this event as $\gE_{\init}$) in $T_0=\widetilde{O}(\epsilon^{-2})$ iterations.
	\end{lemma}
	
	\textbf{Remark:} \cref{lem:warmstart} shows that for fixed initialization $x_0$, running SNAG for at most $T_0=\widetilde{O}(\epsilon^{-2})$ iterations can guarantee that the Euclidean distance between the lower-level variable $y_{T_0}^{\init}$ and the optimal solution $y^*(x_0)$ is at most $O(\epsilon)$, with high probability.
	
	\begin{lemma}[Option I]
		\label{lem:optionI}
		Under event $\gE_{\init}$, let $\{y_t\}$ be the iterates produced by Option I. Set $\alpha=\widetilde{\Theta}(\epsilon^2)$, $\sigma_{g,1}=(\mu\alpha)^{1/4}\tilde{\sigma}_{g,1}$, and $\phi_t(y) = g(x_t,y) = \frac{\mu}{2}\|y-y_t^*\|^2$. Then for any $t\in[T]$, \Cref{alg:bilevel} guarantees with probability at least $1-\delta$ over the randomness in $\widetilde{\mathcal{F}}_T^1$ (we denote this event as $\gE_y^1$) that $\|y_t-y_t^*\| \leq \epsilon/2L_0$.
	\end{lemma} 
	
	\begin{lemma}[Option II]
		\label{lem:optionII}
		Under event $\gE_{\init}$, let $\{y_t\}$ be the iterates produced by Option II. Set $\alpha=\widetilde{\Theta}(\epsilon^2)$, $N=\widetilde{O}(\epsilon^{-1})$, $I=\widetilde{O}(\epsilon^{-1})$, $\sigma_{g,1}=(\mu\alpha)^{1/4}\tilde{\sigma}_{g,1}$, and $\phi_t(y) = g(x_t,y)$ when $t$ is a multiple of $I$ (i.e., $x_t$ is fixed for each update round of Option II so $g$ can be general functions). Then for any $t\in[T]$, \Cref{alg:bilevel} guarantees with probability at least $1-\delta$ over the randomness in $\sigma(\cup_{t\leq T}\widetilde{\gF}_t^2)$ (we denote this event as $\gE_y^2$) that $\|y_t-y_t^*\| \leq \epsilon/L_0$. 
	\end{lemma}
	
	\textbf{Remark}: \cref{lem:optionI} and \cref{lem:optionII} show that, under event $\gE_{\init}$ and both option I and option II, the algorithm guarantees that each iterate $y_t$ is $O(\epsilon)$-close to the the optimal lower-level variable $y_t^*$ at every iteration $t$ with high probability.
	
	\begin{lemma}[Averaging]
		\label{lem:averaging}
		Under event $\gE_{\init}\cap\gE_y^1$ (Option I) or $\gE_{\init}\cap\gE_y^2$ (Option II), set $\tau=\sqrt{\mu\alpha}$ in the averaging step (line 21 of \Cref{alg:bilevel}). Then for any $t\geq 0$ we have $\|\hat{y}_t-y_t^*\| \leq \frac{2\epsilon}{L_0}$ and  $ \|\hat{y}_{t+1}-\hat{y}_t\| \leq \frac{\mu\epsilon^2}{24L_0^2\sigma_{g,1}} =: \vartheta$.
	\end{lemma}
	
	\textbf{Remark:} \cref{lem:averaging} shows that after performing averaging operations over the sequence $\{y_t\}_{t=1}^{T}$, the averaged sequence enjoys stronger guarantees. First, each averaged iterate $\hat{y}_t$ is still $O(\epsilon)$-close to the optimal lower-level variable $y_t^*$; Second, two consecutive averaged iterates (i.e., $\hat{y}_t$ and $\hat{y}_{t+1}$) is $O(\epsilon^2)$-close to each other. The stronger guarantees are crucial to control the hypergradient estimation error as described in \cref{lem:uppererror}.

	\begin{lemma} \label{lem:uppererror}
		Under event $\gE_{\init}\cap\gE_y^1$ (Option I) or $\gE_{\init}\cap\gE_y^2$ (Option II), define $\epsilon_t = m_t-\E_t[\gdf(x_t,\hat{y}_t;\Bar{\xi}_t)]$, then we have the following averaged cumulative error bound:
		\begin{small}
			\begin{equation*}
				\begin{aligned}
					\frac{1}{T}\sum_{t=0}^{T-1}\E\|\epsilon_t\|
					&\leq \frac{\Bar{\sigma}}{T(1-\beta)} + \sqrt{1-\beta}\Bar{\sigma} + \frac{\Bar{L}_0}{\sqrt{1-\beta}}\sqrt{\frac{2(\eta^2+\vartheta^2)}{S}} + \Bar{L}_1\sqrt{\frac{2(\eta^2+\vartheta^2)}{S(1-\beta)}}\frac{1}{T}\sum_{t=0}^{T-1}\E\|\gdphi(x_t)\|,
				\end{aligned}
			\end{equation*}
		\end{small}%
		where $S$ denotes the batch size, and $\Bar{\sigma}, \Bar{L}_0, \Bar{L}_1$ are defined in \Cref{lm:hyper-var,lm:hyper-estimator}.
	\end{lemma}
	
	\textbf{Remark:} \cref{lem:uppererror} characterizes the upper-level hypergradient estimation error under the good event that the lower-level error can be controlled. One can choose hyperparameters appropriately such that the cumulative error (i.e., LHS) grows only sublinearly in terms of $T$, which is important for establishing the fast convergence of our algorithm.

	\section{Experiments} 
	\label{sec:experiment}
	
	\textbf{Deep AUC Maximization with Recurrent Neural Networks}.\label{sec:auc_max}
	AUC (Area Under the ROC Curve)~\cite{hanley1983method} is a critical metric in evaluating the performance of binary classification models. It measures the ability of the model to distinguish between positive and negative classes, and it is defined as the probability that the prediction score of a positive example is higher than that is a negative example~\cite{hanley1982meaning}. Deep AUC maximization~\cite{liu2019stochastic,ying2016stochastic} can be formulated as a min-max optimization problem~\cite{liu2019stochastic}: $\min_{\vw\in\mathbb{R}^{d}, (a,b)\in \mathbb{R}^2} \max_{\alpha\in \mathbb{R}} f(\vw, a, b, \alpha)\ \coloneqq \mathbb{E}_{\vz}[F(\vw, a, b, \alpha; \vz)],$
	where $F(\vw, a, b, \alpha; \vz) = (1-r)(h(\vw; \vx)-a )^2 \mathbb{I}_{[c=1]} + r (h(\vw;\vx)-b)^2\mathbb{I}_{[c=-1]}+ 2(1+\alpha)(r h(\vw;\vx)\mathbb{I}_{[c=-1]}-(1-r)h(\vw; \vx)\mathbb{I}_{[c=1]}) - r(1-r)\alpha^2$, $\vw$ denotes the model parameter, $\vz=(\vx,c)$ is the random data sample ($\vx$ denote the feature vector and $c\in\{+1,-1\}$ denotes the label), $h(\vw,\vx)$ is the score function defined by a neural network, and $r=\text{Pr}(c=1)$ denotes the ratio of positive samples in the population. This min-max formulation is an special case of the bilevel problem with $g=-f$ in~\eqref{eq:bp}, which can be reformulated as the following:
	\begin{equation}
		\label{eq:deepauc}
		\begin{aligned}
			&\min_{\vw\in\mathbb{R}^{d}, (a,b)\in \mathbb{R}^2} \mathbb{E}_{\vz}[F(\vw, a, b, \alpha^*(\vw,a,b); \vz)]  \quad \text{s.t.,} \quad
			\alpha^*(\vw, a, b) \in \arg\min_{\alpha \in \mathbb{R}} \ -\mathbb{E}_{\vz}[F(\vw, a, b, \alpha; \vz)]
		\end{aligned}
	\end{equation}
	where $(\vw, a, b)$ denotes the upper-level variable, and $\alpha$ denotes the lower-level variable. In this case, the lower-level is a quadratic function in terms of $\alpha$ and is strongly convex, and the upper-level function is non-convex function with potential unbounded smoothness when using a recurrent neural network as the predictive model. 
	
	We aim to perform imbalanced text classification task and maximize the AUC metric. The Deep AUC maximization experiment is performed on imbalanced Sentiment140 \cite{go2009twitter} dataset (under the license of CC BY 4.0), which is a binary text classification task. Specifically, we follow \cite{yuan2021large} to make training set imbalanced with a pre-defined imbalanced ratio ($r$), and leave the test set unchanged. Given $r$, we randomly discard the positive samples (with label 1) in original training set until the portion of positive samples equals to $r$.  The imbalance ratio $r$ is set to 0.2 in our experiment, which means only 
	$20\%$ data is positive in the training set. We use a two-layer recurrent neural network with input dimension=300, hidden dimension=4096, and output dimension=2 for the model prediction.
	
	We compare with some bilevel optimization baselines, including StocBio \cite{ji2021bilevel}, TTSA \cite{hong2023two}, SABA \cite{dagreou2022framework}, MA-SOBA \cite{chen2023optimal}, SUSTAIN \cite{khanduri2021near}, VRBO \cite{yang2021provably} and BO-REP \cite{hao2024bilevel}. We show the training and test AUC result with 25 epochs in (a) (b) of \cref{fig:auc} and running time in (c), (d) of \cref{fig:auc}. Our algorithm AccBO achieves highest AUC score among all the baselines over epochs and running time. The running time figure shows AccBO converges to a good result faster than other baselines. The detailed parameter tuning and selection are included in \cref{sec:detailexp_hc}.
	
	\begin{figure*}[!t]
		\begin{center}
			\subfigure[\scriptsize Training AUC]{\includegraphics[width=0.245\linewidth]{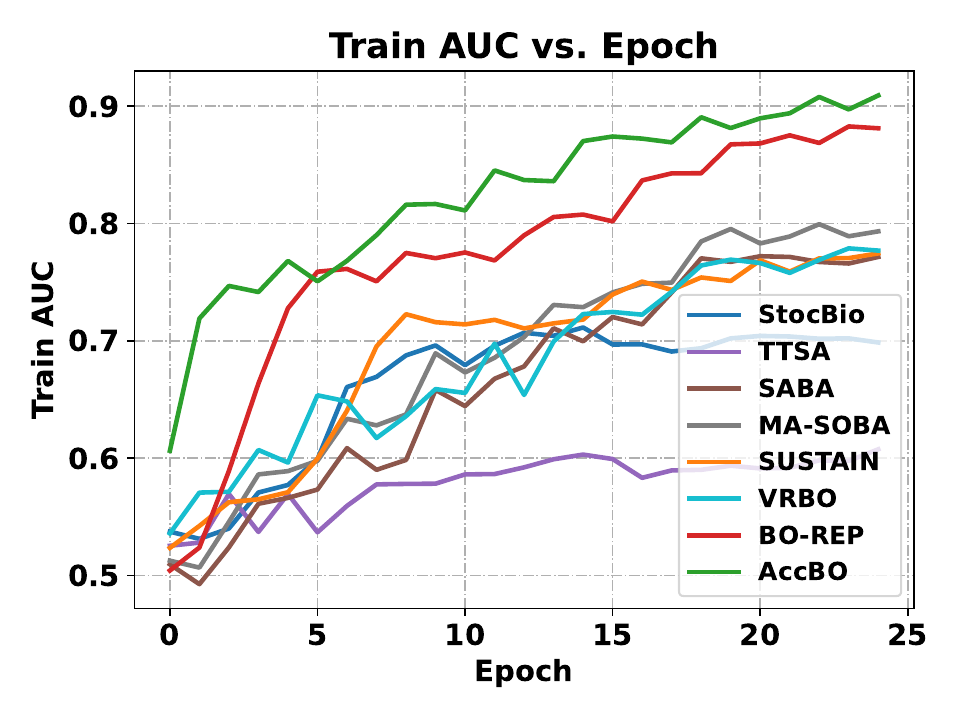}}   
			\subfigure[\scriptsize Test AUC]{\includegraphics[width=0.245\linewidth]{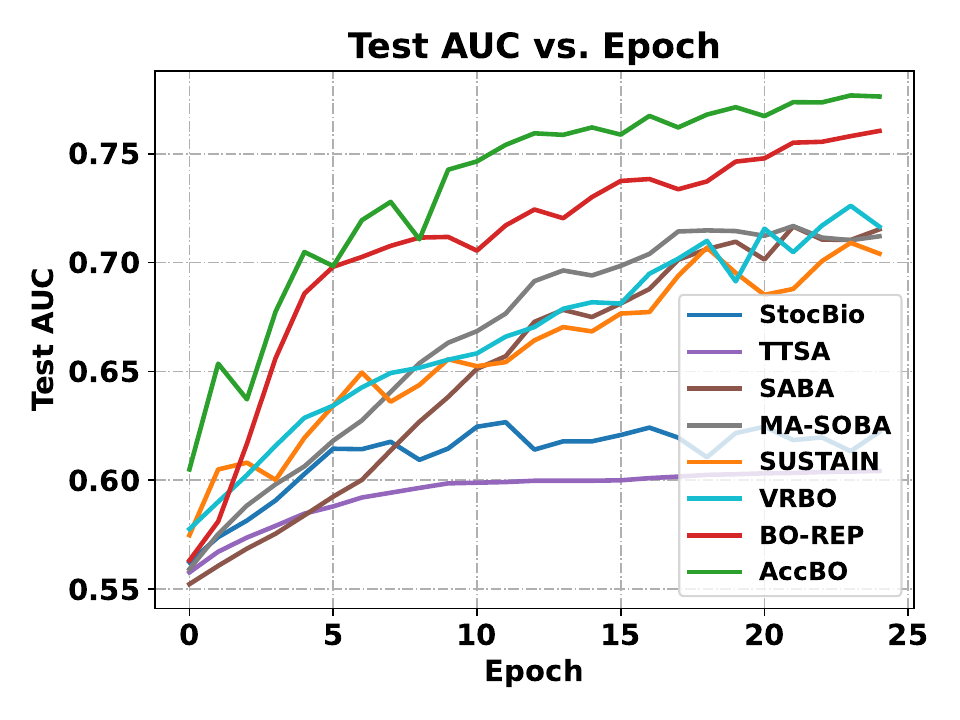}}   
			\subfigure[\scriptsize Training AUC vs. running time]{\includegraphics[width=0.245\linewidth]{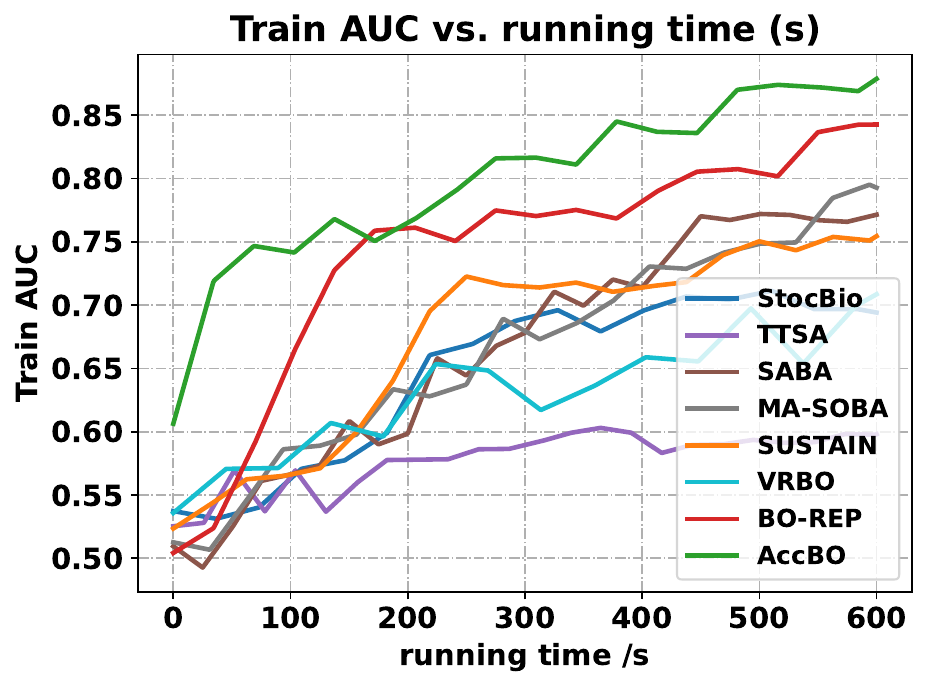}}  
			\subfigure[\scriptsize Test AUC vs. running time]{\includegraphics[width=0.245\linewidth]{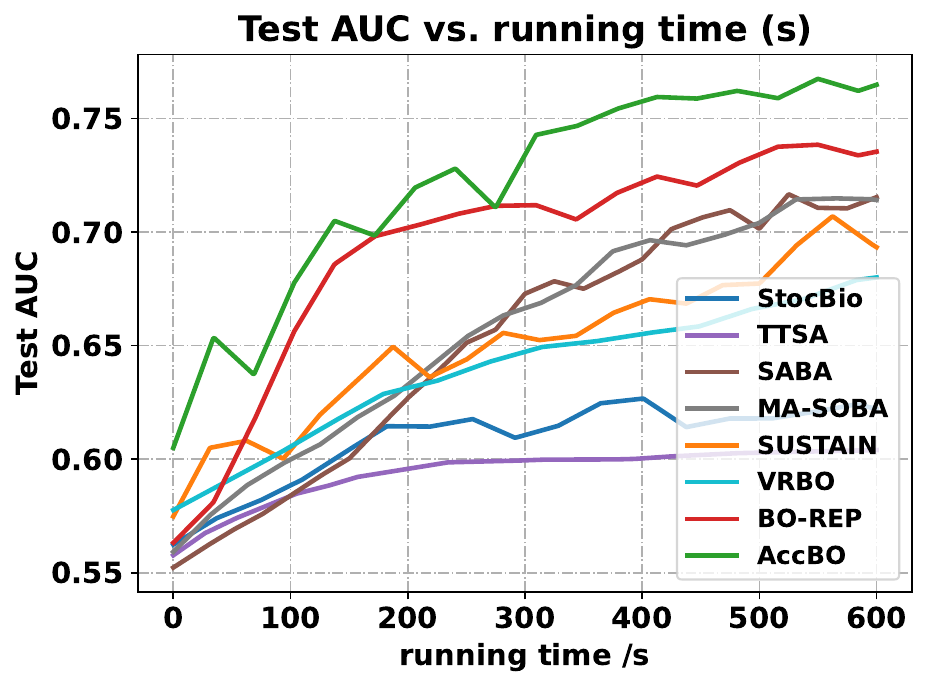}}
		\end{center}
		\caption{Results of bilevel optimization on deep AUC maximization. Figure (a), (b) are the results over epochs, and Figure (c), (d) are the results over running time.}
		\label{fig:auc}
	\end{figure*}

	\begin{figure*}[!t]
		\begin{center}
			\subfigure[\scriptsize Training ACC]{\includegraphics[width=0.245\linewidth]{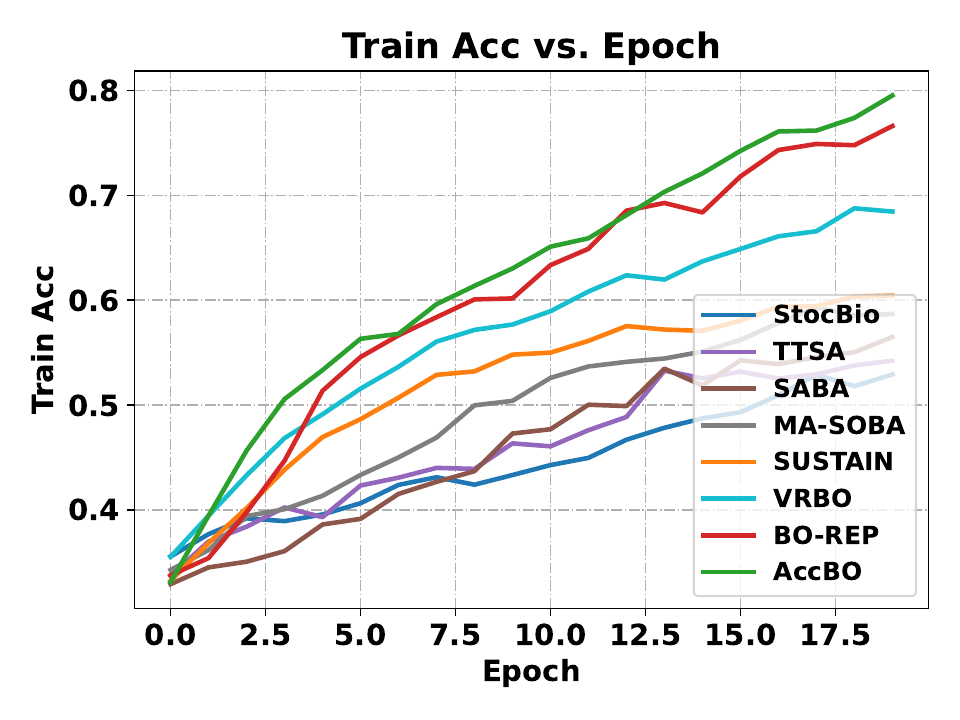}}   
			\subfigure[\scriptsize Test ACC]{\includegraphics[width=0.245\linewidth]{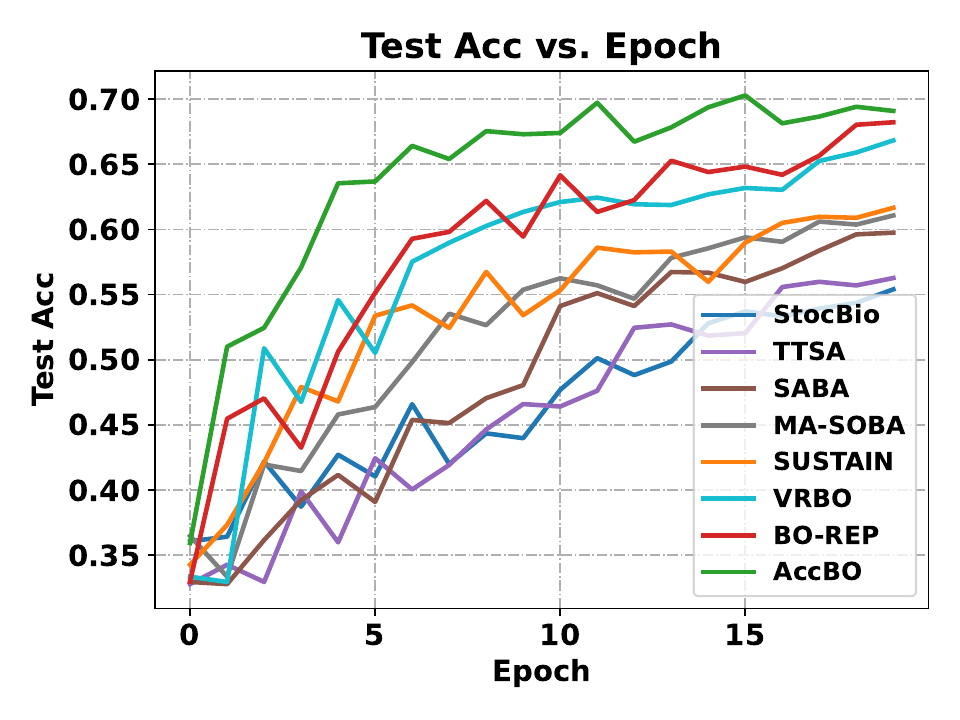}}   
			\subfigure[\scriptsize Training ACC vs. running time]{\includegraphics[width=0.245\linewidth]{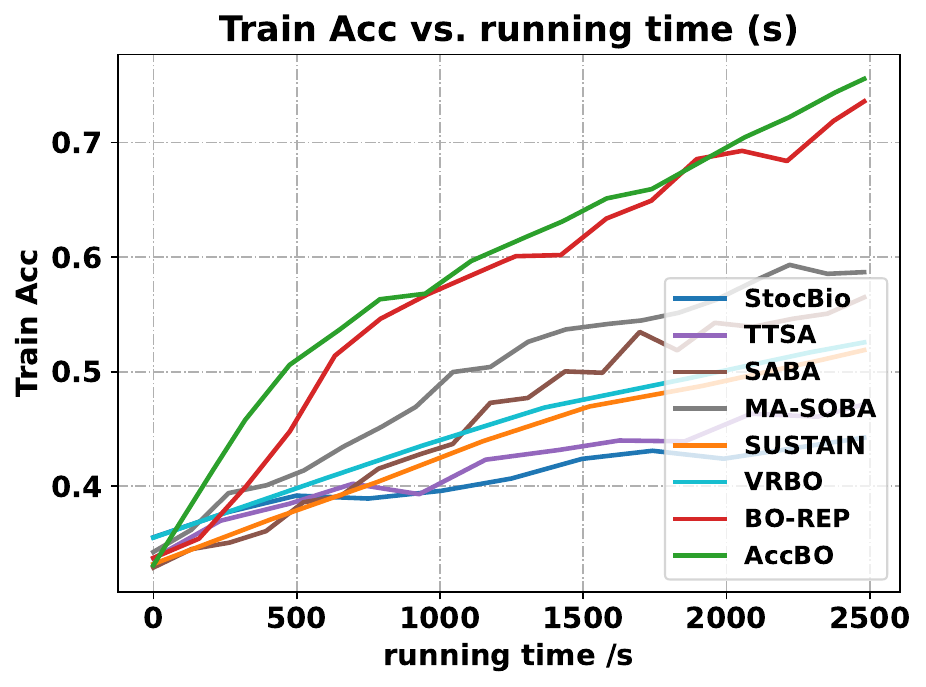}}  
			\subfigure[\scriptsize Test ACC vs. running time]{\includegraphics[width=0.245\linewidth]{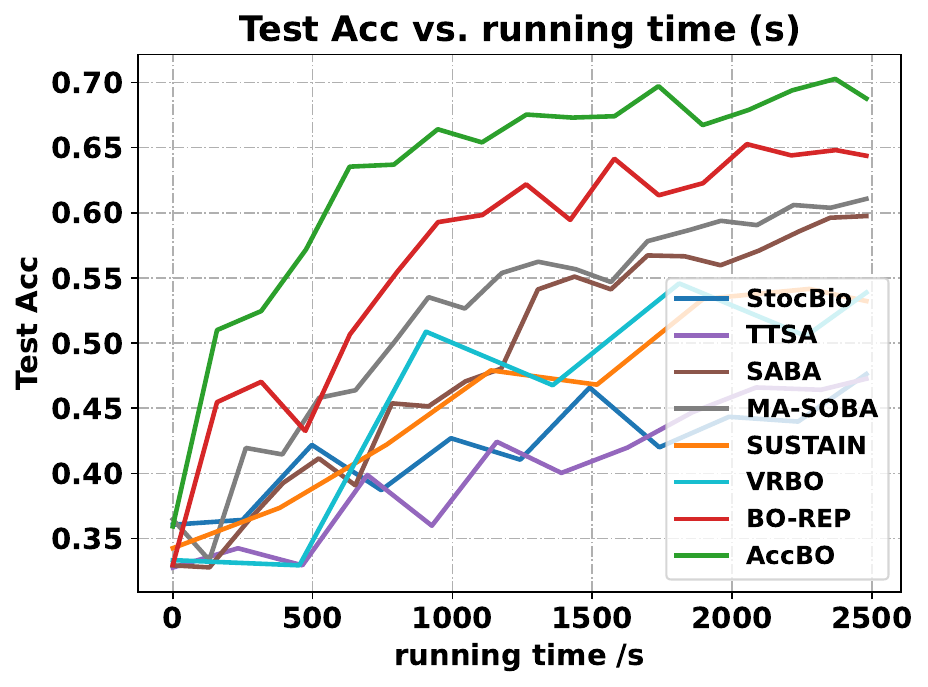}} \\
			\subfigure[\scriptsize Training ACC]{\includegraphics[width=0.245\linewidth]{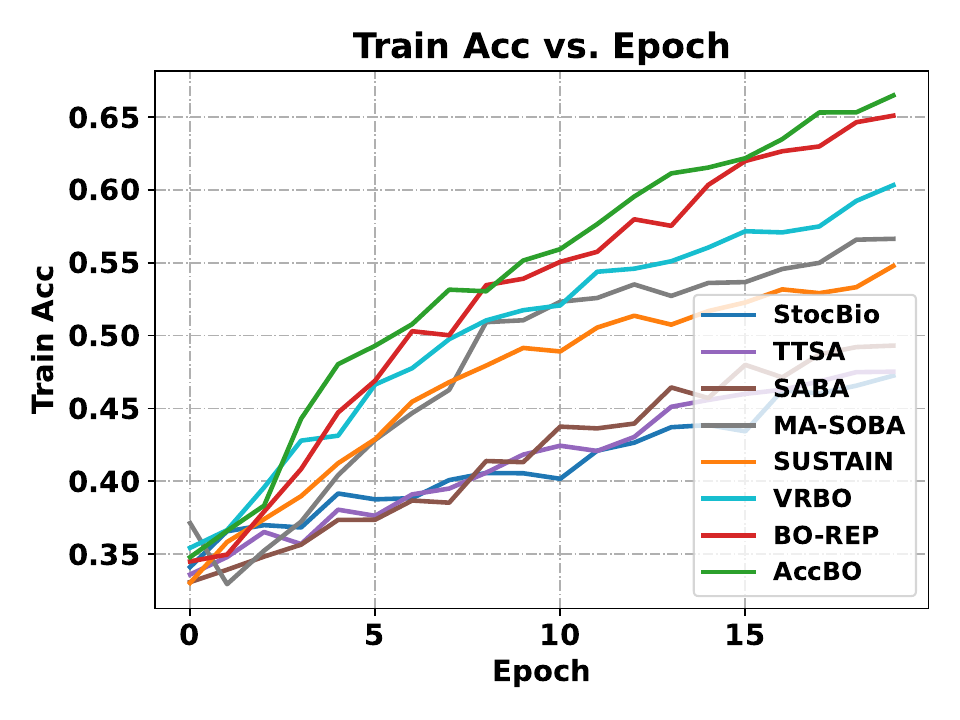}}   
			\subfigure[\scriptsize Test ACC]{\includegraphics[width=0.245\linewidth]{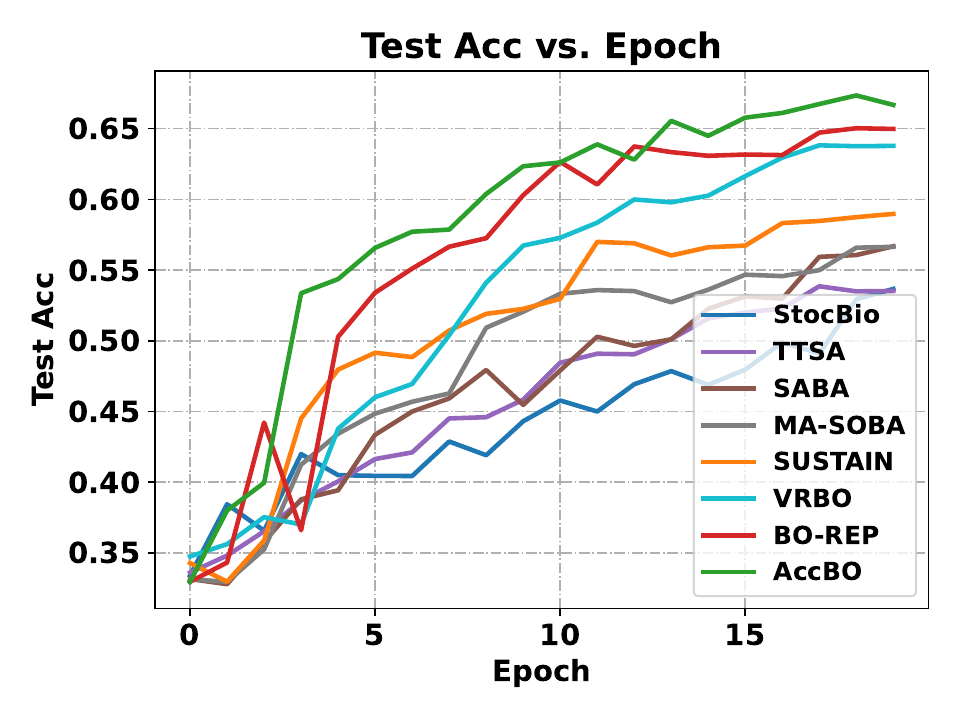}}   
			\subfigure[\scriptsize Training ACC vs. running time]{\includegraphics[width=0.247\linewidth]{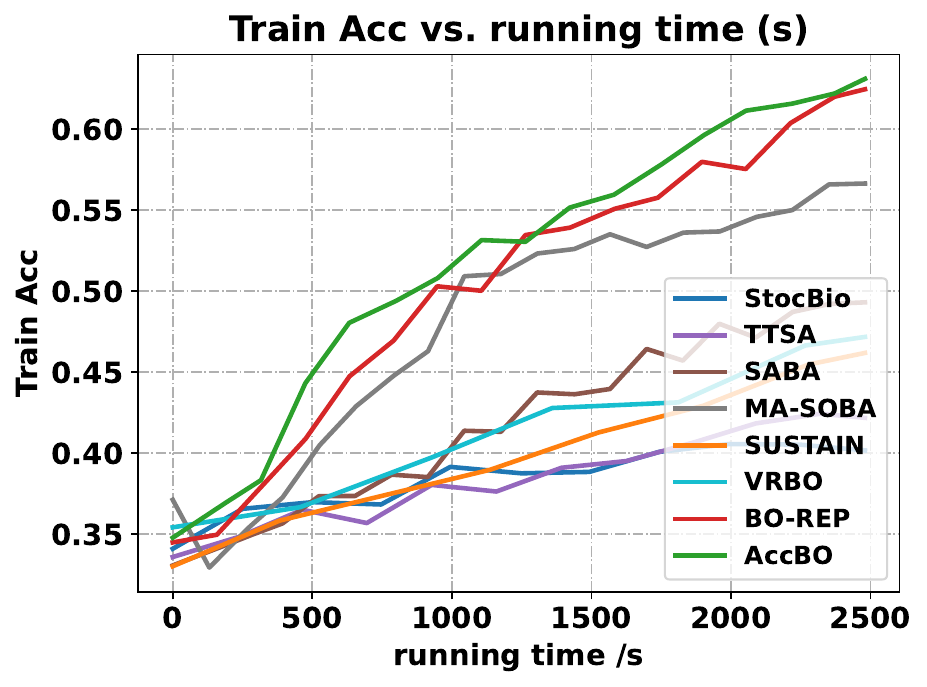}}  
			\subfigure[\scriptsize Test ACC vs. running time]{\includegraphics[width=0.245\linewidth]{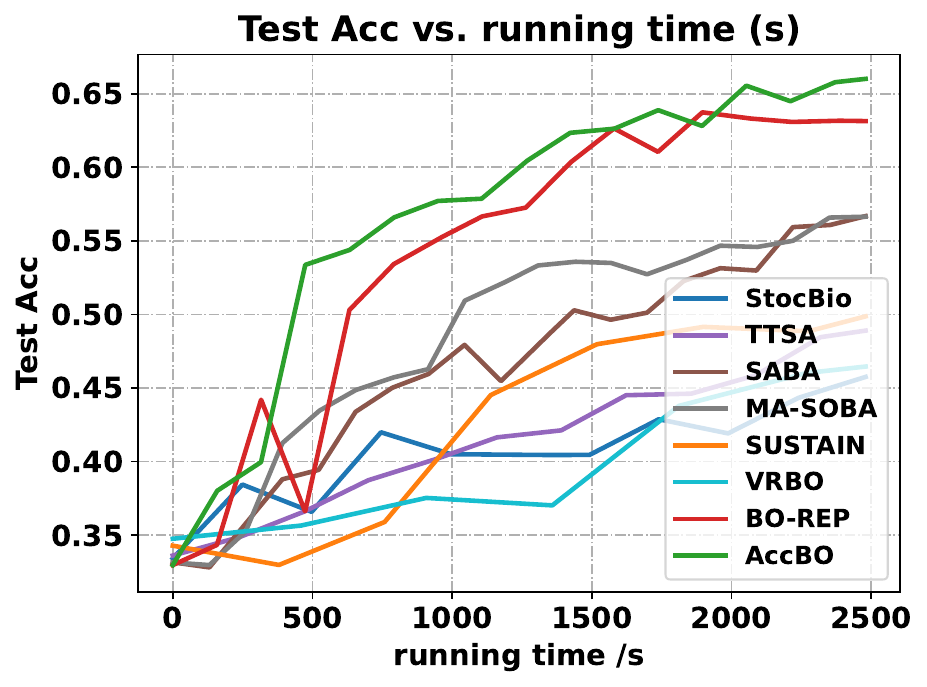}}
		\end{center}
		\caption{Results of bilevel optimization on data hyper-cleaning with $p=0.1$. Figure (a), (b), (c), (d) are the results with noise rate $p=0.1$  where (a), (b) are the results over epochs, and Figure (c), (d) are the results over running time. Figure (e), (f), (g), (h) are the results with noise rate $p=0.2$.}
		\label{fig:hc_p_0.2}
	\end{figure*}
	
	\textbf{Data Hypercleaning.} The Data hypercleaning task tries to learn a set of weights $\boldsymbol\lambda$ for the corrupted training data $\mathcal{D}_{tr}$, such that the model trained on the weighted corrupted training set can achieve good performance on the clean validation set $\mathcal{D}_{val}$, where the corrupted training set $\mathcal{D}_{tr}\coloneqq \{\vx_i, \bar{y_i}\}$ and the label $\bar{y_i}$ is randomly flipped to one of other labels with probability $0<p<1$. The data hyper-cleaning can be formulated as a bilevel optimization problem,
	\begin{equation}\label{eq:hc_formulation}
		\min_{\boldsymbol{\lambda}} \frac{1}{|\mathcal{D}_{\text{val}}|}\sum_{\xi\in \mathcal{D}_{\text{val}}}\mathcal{L}(\vw^*(\boldsymbol{\lambda}); \xi), \ \mathrm{s.t.} \ \vw^*(\boldsymbol{\lambda}) \in \argmin_{\vw} \frac{1}{|\mathcal{D}_{\text{tr}}|}\sum_{\zeta_i\in \mathcal{D}_{\text{tr}} }\sigma(\lambda_i)\mathcal{L}(\vw; \zeta_i) +  c \|\vw\|^2,
	\end{equation}
	where $\vw$ is the model parameter of a neural network, and $\sigma(x)=\frac{1}{1+e^{-x}}$ is the sigmoid function. We perform bilevel optimization algorithms on the noisy text classification dataset Stanford Natural Language Inference (SNLI)~\cite{bowman2015large} (under the license of CC BY 4.0) with a three-layer recurrent neural network with input dimension=300, hidden dimension=4096, and output dimension=3 for the label prediction. Each of sentence-pairs manually labeled as entailment, contradiction, and neutral. Specifically, the label of each training data is randomly flipped to one of the other two labels with probability $p$. We set $p=0.1$ and $p=0.2$ in the experiments, respectively. We compare all the baselines used in the deep AUC maximization experiment. Different from the formulation \eqref{eq:deepauc} for the deep AUC maximization, the lower-level function in \eqref{eq:hc_formulation} is not quadratic function of the lower-level variable. Therefore we choose Option II in \cref{alg:bilevel}, i.e., periodic updates for the lower-level variable.
	The results  are presented in \cref{fig:hc_p_0.2} ($p=0.1$ and $p=0.2$). Our algorithm AccBO exhibits the highest classification accuracy on training and test set among all the bilevel baselines, and also shows a high runtime efficiency.  More detailed parameter tuning and selection can be found in \cref{sec:detailexp_hc}. All the experiments are run on the device of NVIDIA A6000 (48GB memory) GPU and AMD EPYC 7513 32-Core CPU.

	\section{Conclusion}
	\label{sec:conclusion}
	In this paper, we propose a new algorithm named AccBO for solving bilevel optimization problems where the upper-level is nonconvex and unbounded smooth and the lower-level problem is strongly convex. The algorithm achieved $\widetilde{O}(\epsilon^{-3})$ oracle complexity for finding an $\epsilon$-stationary point when the lower-level stochastic gradient variance is $O(\epsilon)$, which matches the rate of the state-of-the-art single-level relaxed smooth optimization~\cite{liu2023near} and is nearly optimal in terms of dependency on $\epsilon$~\cite{arjevani2023lower}.
	
	\textbf{Limitations.} There are two limitations of our work. One limitation of our work is that the convergence analysis for the Option I of our algorithm relies on the lower-level problem being a quadratic function: only under this case the algorithm becomes a single-loop procedure. Another limitation is that we require the lower-level stochastic gradient has variance $O(\epsilon)$. It remains unclear how to design single-loop algorithms for more general lower-level strongly convex functions and get rid of the small stochastic gradient variance assumption for the lower-level variable.
	
	\begin{ack}
We would like to thank the anonymous reviewers for their helpful comments. We would like to thank Tianbao Yang and Qihang Lin for helpful discussions for the quadratic function with distributional drift in the earlier version of our paper. This work has been supported by the Presidential Scholarship, the ORIEI seed funding, and the IDIA P3 fellowship from George Mason University, the Cisco Faculty Research Award, and NSF award \#2436217, \#2425687. The Computations were run on ARGO, a research computing cluster provided by the Office of Research Computing at George Mason University (URL: https://orc.gmu.edu).

	\end{ack}

	\bibliography{ref}

\begin{thebibliography}{75}
\providecommand{\natexlab}[1]{#1}
\providecommand{\url}[1]{\texttt{#1}}
\expandafter\ifx\csname urlstyle\endcsname\relax
  \providecommand{\doi}[1]{doi: #1}\else
  \providecommand{\doi}{doi: \begingroup \urlstyle{rm}\Url}\fi

\bibitem[Anandalingam and White(1990)]{anandalingam1990solution}
G~Anandalingam and DJ~White.
\newblock A solution method for the linear static stackelberg problem using
  penalty functions.
\newblock \emph{IEEE Transactions on automatic control}, 35\penalty0
  (10):\penalty0 1170--1173, 1990.

\bibitem[Arjevani et~al.(2023)Arjevani, Carmon, Duchi, Foster, Srebro, and
  Woodworth]{arjevani2023lower}
Yossi Arjevani, Yair Carmon, John~C Duchi, Dylan~J Foster, Nathan Srebro, and
  Blake Woodworth.
\newblock Lower bounds for non-convex stochastic optimization.
\newblock \emph{Mathematical Programming}, 199\penalty0 (1-2):\penalty0
  165--214, 2023.

\bibitem[Assran and Rabbat(2020)]{assran2020convergence}
Mahmoud Assran and Michael Rabbat.
\newblock On the convergence of nesterov's accelerated gradient method in
  stochastic settings.
\newblock \emph{arXiv preprint arXiv:2002.12414}, 2020.

\bibitem[Aybat et~al.(2019)Aybat, Fallah, Gurbuzbalaban, and
  Ozdaglar]{aybat2019universally}
Necdet~Serhat Aybat, Alireza Fallah, Mert Gurbuzbalaban, and Asuman Ozdaglar.
\newblock A universally optimal multistage accelerated stochastic gradient
  method.
\newblock \emph{Advances in neural information processing systems}, 32, 2019.

\bibitem[Aybat et~al.(2020)Aybat, Fallah, Gurbuzbalaban, and
  Ozdaglar]{aybat2020robust}
Necdet~Serhat Aybat, Alireza Fallah, Mert Gurbuzbalaban, and Asuman Ozdaglar.
\newblock Robust accelerated gradient methods for smooth strongly convex
  functions.
\newblock \emph{SIAM Journal on Optimization}, 30\penalty0 (1):\penalty0
  717--751, 2020.

\bibitem[Besbes et~al.(2015)Besbes, Gur, and Zeevi]{besbes2015non}
Omar Besbes, Yonatan Gur, and Assaf Zeevi.
\newblock Non-stationary stochastic optimization.
\newblock \emph{Operations research}, 63\penalty0 (5):\penalty0 1227--1244,
  2015.

\bibitem[Borsos et~al.(2020)Borsos, Mutny, and Krause]{borsos2020coresets}
Zal{\'a}n Borsos, Mojmir Mutny, and Andreas Krause.
\newblock Coresets via bilevel optimization for continual learning and
  streaming.
\newblock \emph{Advances in neural information processing systems},
  33:\penalty0 14879--14890, 2020.

\bibitem[Bowman et~al.(2015)Bowman, Angeli, Potts, and
  Manning]{bowman2015large}
Samuel~R Bowman, Gabor Angeli, Christopher Potts, and Christopher~D Manning.
\newblock A large annotated corpus for learning natural language inference.
\newblock \emph{arXiv preprint arXiv:1508.05326}, 2015.

\bibitem[Bracken and McGill(1973)]{bracken1973mathematical}
Jerome Bracken and James~T McGill.
\newblock Mathematical programs with optimization problems in the constraints.
\newblock \emph{Operations research}, 21\penalty0 (1):\penalty0 37--44, 1973.

\bibitem[Chen et~al.(2023{\natexlab{a}})Chen, Xu, and Zhang]{chen2023bilevel1}
Lesi Chen, Jing Xu, and Jingzhao Zhang.
\newblock On bilevel optimization without lower-level strong convexity.
\newblock \emph{arXiv preprint arXiv:2301.00712}, 2023{\natexlab{a}}.

\bibitem[Chen et~al.(2021)Chen, Sun, and Yin]{chen2021closing}
Tianyi Chen, Yuejiao Sun, and Wotao Yin.
\newblock Closing the gap: Tighter analysis of alternating stochastic gradient
  methods for bilevel problems.
\newblock \emph{Advances in Neural Information Processing Systems},
  34:\penalty0 25294--25307, 2021.

\bibitem[Chen et~al.(2023{\natexlab{b}})Chen, Xiao, and
  Balasubramanian]{chen2023optimal}
Xuxing Chen, Tesi Xiao, and Krishnakumar Balasubramanian.
\newblock Optimal algorithms for stochastic bilevel optimization under relaxed
  smoothness conditions.
\newblock \emph{arXiv preprint arXiv:2306.12067}, 2023{\natexlab{b}}.

\bibitem[Chen et~al.(2023{\natexlab{c}})Chen, Na, and
  Kolar]{chen2023convergence}
You-Lin Chen, Sen Na, and Mladen Kolar.
\newblock Convergence analysis of accelerated stochastic gradient descent under
  the growth condition.
\newblock \emph{Mathematics of Operations Research}, 2023{\natexlab{c}}.

\bibitem[Chen et~al.(2023{\natexlab{d}})Chen, Zhou, Liang, and
  Lu]{chen2023generalized}
Ziyi Chen, Yi~Zhou, Yingbin Liang, and Zhaosong Lu.
\newblock Generalized-smooth nonconvex optimization is as efficient as smooth
  nonconvex optimization.
\newblock \emph{arXiv preprint arXiv:2303.02854}, 2023{\natexlab{d}}.

\bibitem[Crawshaw et~al.(2022)Crawshaw, Liu, Orabona, Zhang, and
  Zhuang]{crawshaw2022robustness}
Michael Crawshaw, Mingrui Liu, Francesco Orabona, Wei Zhang, and Zhenxun
  Zhuang.
\newblock Robustness to unbounded smoothness of generalized signsgd.
\newblock \emph{Advances in neural information processing systems}, 2022.

\bibitem[Crawshaw et~al.(2023{\natexlab{a}})Crawshaw, Bao, and
  Liu]{crawshaw2023episode}
Michael Crawshaw, Yajie Bao, and Mingrui Liu.
\newblock Episode: Episodic gradient clipping with periodic resampled
  corrections for federated learning with heterogeneous data.
\newblock In \emph{The Eleventh International Conference on Learning
  Representations}, 2023{\natexlab{a}}.

\bibitem[Crawshaw et~al.(2023{\natexlab{b}})Crawshaw, Bao, and
  Liu]{crawshaw2023federated}
Michael Crawshaw, Yajie Bao, and Mingrui Liu.
\newblock Federated learning with client subsampling, data heterogeneity, and
  unbounded smoothness: A new algorithm and lower bounds.
\newblock In \emph{Thirty-seventh Conference on Neural Information Processing
  Systems}, 2023{\natexlab{b}}.

\bibitem[Cutkosky and Orabona(2019)]{cutkosky2019momentum}
Ashok Cutkosky and Francesco Orabona.
\newblock Momentum-based variance reduction in non-convex sgd.
\newblock \emph{Advances in neural information processing systems}, 32, 2019.

\bibitem[Cutler et~al.(2021)Cutler, Drusvyatskiy, and
  Harchaoui]{cutler2021stochastic}
Joshua Cutler, Dmitriy Drusvyatskiy, and Zaid Harchaoui.
\newblock Stochastic optimization under time drift: iterate averaging,
  step-decay schedules, and high probability guarantees.
\newblock \emph{Advances in neural information processing systems},
  34:\penalty0 11859--11869, 2021.

\bibitem[Cutler et~al.(2023)Cutler, Drusvyatskiy, and
  Harchaoui]{cutler2023stochastic}
Joshua Cutler, Dmitriy Drusvyatskiy, and Zaid Harchaoui.
\newblock Stochastic optimization under distributional drift.
\newblock \emph{Journal of Machine Learning Research}, 24\penalty0
  (147):\penalty0 1--56, 2023.

\bibitem[Dagr{\'e}ou et~al.(2022)Dagr{\'e}ou, Ablin, Vaiter, and
  Moreau]{dagreou2022framework}
Mathieu Dagr{\'e}ou, Pierre Ablin, Samuel Vaiter, and Thomas Moreau.
\newblock A framework for bilevel optimization that enables stochastic and
  global variance reduction algorithms.
\newblock \emph{Advances in Neural Information Processing Systems},
  35:\penalty0 26698--26710, 2022.

\bibitem[Elman(1990)]{elman1990finding}
Jeffrey~L Elman.
\newblock Finding structure in time.
\newblock \emph{Cognitive science}, 14\penalty0 (2):\penalty0 179--211, 1990.

\bibitem[Fang et~al.(2018)Fang, Li, Lin, and Zhang]{fang2018spider}
Cong Fang, Chris~Junchi Li, Zhouchen Lin, and Tong Zhang.
\newblock Spider: Near-optimal non-convex optimization via stochastic
  path-integrated differential estimator.
\newblock \emph{Advances in neural information processing systems}, 31, 2018.

\bibitem[Faw et~al.(2023)Faw, Rout, Caramanis, and Shakkottai]{faw2023beyond}
Matthew Faw, Litu Rout, Constantine Caramanis, and Sanjay Shakkottai.
\newblock Beyond uniform smoothness: A stopped analysis of adaptive sgd.
\newblock \emph{arXiv preprint arXiv:2302.06570}, 2023.

\bibitem[Feurer and Hutter(2019)]{feurer2019hyperparameter}
Matthias Feurer and Frank Hutter.
\newblock Hyperparameter optimization.
\newblock In \emph{Automated Machine Learning}, pages 3--33. Springer, Cham,
  2019.

\bibitem[Finn et~al.(2017)Finn, Abbeel, and Levine]{finn2017model}
Chelsea Finn, Pieter Abbeel, and Sergey Levine.
\newblock Model-agnostic meta-learning for fast adaptation of deep networks.
\newblock In \emph{International conference on machine learning}, pages
  1126--1135. PMLR, 2017.

\bibitem[Franceschi et~al.(2018)Franceschi, Frasconi, Salzo, Grazzi, and
  Pontil]{franceschi2018bilevel}
Luca Franceschi, Paolo Frasconi, Saverio Salzo, Riccardo Grazzi, and
  Massimiliano Pontil.
\newblock Bilevel programming for hyperparameter optimization and
  meta-learning.
\newblock In \emph{International conference on machine learning}, pages
  1568--1577. PMLR, 2018.

\bibitem[Ghadimi and Lan(2013)]{ghadimi2013stochastic}
Saeed Ghadimi and Guanghui Lan.
\newblock Stochastic first-and zeroth-order methods for nonconvex stochastic
  programming.
\newblock \emph{SIAM Journal on Optimization}, 23\penalty0 (4):\penalty0
  2341--2368, 2013.

\bibitem[Ghadimi and Lan(2016)]{ghadimi2016accelerated}
Saeed Ghadimi and Guanghui Lan.
\newblock Accelerated gradient methods for nonconvex nonlinear and stochastic
  programming.
\newblock \emph{Mathematical Programming}, 156\penalty0 (1-2):\penalty0 59--99,
  2016.

\bibitem[Ghadimi and Wang(2018)]{ghadimi2018approximation}
Saeed Ghadimi and Mengdi Wang.
\newblock Approximation methods for bilevel programming.
\newblock \emph{arXiv preprint arXiv:1802.02246}, 2018.

\bibitem[Go et~al.(2009)Go, Bhayani, and Huang]{go2009twitter}
Alec Go, Richa Bhayani, and Lei Huang.
\newblock Twitter sentiment classification using distant supervision.
\newblock \emph{CS224N project report, Stanford}, 1\penalty0 (12):\penalty0
  2009, 2009.

\bibitem[Gong et~al.(2024)Gong, Hao, and Liu]{gong2024a}
Xiaochuan Gong, Jie Hao, and Mingrui Liu.
\newblock A nearly optimal single loop algorithm for stochastic bilevel
  optimization under unbounded smoothness.
\newblock In \emph{Forty-first International Conference on Machine Learning},
  2024.

\bibitem[Grazzi et~al.(2022)Grazzi, Pontil, and Salzo]{grazzi2022bilevel}
Riccardo Grazzi, Massimiliano Pontil, and Saverio Salzo.
\newblock Bilevel optimization with a lower-level contraction: Optimal sample
  complexity without warm-start.
\newblock \emph{arXiv preprint arXiv:2202.03397}, 2022.

\bibitem[Guo et~al.(2021)Guo, Hu, Zhang, and Yang]{guo2021randomized}
Zhishuai Guo, Quanqi Hu, Lijun Zhang, and Tianbao Yang.
\newblock Randomized stochastic variance-reduced methods for multi-task
  stochastic bilevel optimization.
\newblock \emph{arXiv preprint arXiv:2105.02266}, 2021.

\bibitem[Hanley and McNeil(1982)]{hanley1982meaning}
James~A Hanley and Barbara~J McNeil.
\newblock The meaning and use of the area under a receiver operating
  characteristic (roc) curve.
\newblock \emph{Radiology}, 143\penalty0 (1):\penalty0 29--36, 1982.

\bibitem[Hanley and McNeil(1983)]{hanley1983method}
James~A Hanley and Barbara~J McNeil.
\newblock A method of comparing the areas under receiver operating
  characteristic curves derived from the same cases.
\newblock \emph{Radiology}, 148\penalty0 (3):\penalty0 839--843, 1983.

\bibitem[Hao et~al.(2023)Hao, Ji, and Liu]{hao2023bilevelcoreset}
Jie Hao, Kaiyi Ji, and Mingrui Liu.
\newblock Bilevel coreset selection in continual learning: A new formulation
  and algorithm.
\newblock \emph{Advances in Neural Information Processing Systems}, 36, 2023.

\bibitem[Hao et~al.(2024)Hao, Gong, and Liu]{hao2024bilevel}
Jie Hao, Xiaochuan Gong, and Mingrui Liu.
\newblock Bilevel optimization under unbounded smoothness: A new algorithm and
  convergence analysis.
\newblock In \emph{The Twelfth International Conference on Learning
  Representations}, 2024.

\bibitem[Hazan and Kale(2014)]{hazan2014beyond}
Elad Hazan and Satyen Kale.
\newblock Beyond the regret minimization barrier: optimal algorithms for
  stochastic strongly-convex optimization.
\newblock \emph{Journal of Machine Learning Research}, 15\penalty0
  (1):\penalty0 2489--2512, 2014.

\bibitem[Hochreiter and Schmidhuber(1997)]{hochreiter1997long}
Sepp Hochreiter and J{\"u}rgen Schmidhuber.
\newblock Long short-term memory.
\newblock \emph{Neural computation}, 9\penalty0 (8):\penalty0 1735--1780, 1997.

\bibitem[Hong et~al.(2023)Hong, Wai, Wang, and Yang]{hong2023two}
Mingyi Hong, Hoi-To Wai, Zhaoran Wang, and Zhuoran Yang.
\newblock A two-timescale stochastic algorithm framework for bilevel
  optimization: Complexity analysis and application to actor-critic.
\newblock \emph{SIAM Journal on Optimization}, 33\penalty0 (1):\penalty0
  147--180, 2023.

\bibitem[Hu and Lessard(2017)]{hu2017dissipativity}
Bin Hu and Laurent Lessard.
\newblock Dissipativity theory for nesterov’s accelerated method.
\newblock In \emph{International Conference on Machine Learning}, pages
  1549--1557. PMLR, 2017.

\bibitem[Hu et~al.(2023)Hu, Qiu, Guo, Zhang, and Yang]{hu2023blockwise}
Quanqi Hu, Zi-Hao Qiu, Zhishuai Guo, Lijun Zhang, and Tianbao Yang.
\newblock Blockwise stochastic variance-reduced methods with parallel speedup
  for multi-block bilevel optimization.
\newblock In \emph{International Conference on Machine Learning}, pages
  13550--13583. PMLR, 2023.

\bibitem[Ji et~al.(2021)Ji, Yang, and Liang]{ji2021bilevel}
Kaiyi Ji, Junjie Yang, and Yingbin Liang.
\newblock Bilevel optimization: Convergence analysis and enhanced design.
\newblock In \emph{International conference on machine learning}, pages
  4882--4892. PMLR, 2021.

\bibitem[Jin et~al.(2021)Jin, Zhang, Wang, and Wang]{jin2021non}
Jikai Jin, Bohang Zhang, Haiyang Wang, and Liwei Wang.
\newblock Non-convex distributionally robust optimization: Non-asymptotic
  analysis.
\newblock \emph{Advances in Neural Information Processing Systems},
  34:\penalty0 2771--2782, 2021.

\bibitem[Khanduri et~al.(2021)Khanduri, Zeng, Hong, Wai, Wang, and
  Yang]{khanduri2021near}
Prashant Khanduri, Siliang Zeng, Mingyi Hong, Hoi-To Wai, Zhaoran Wang, and
  Zhuoran Yang.
\newblock A near-optimal algorithm for stochastic bilevel optimization via
  double-momentum.
\newblock \emph{Advances in Neural Information Processing Systems (NeurIPS)},
  34:\penalty0 30271--30283, 2021.

\bibitem[Konda and Tsitsiklis(1999)]{konda1999actor}
Vijay Konda and John Tsitsiklis.
\newblock Actor-critic algorithms.
\newblock \emph{Advances in neural information processing systems}, 12, 1999.

\bibitem[Kwon et~al.(2023{\natexlab{a}})Kwon, Kwon, Wright, and
  Nowak]{kwon2023fully}
Jeongyeol Kwon, Dohyun Kwon, Stephen Wright, and Robert~D Nowak.
\newblock A fully first-order method for stochastic bilevel optimization.
\newblock In \emph{International Conference on Machine Learning}, pages
  18083--18113. PMLR, 2023{\natexlab{a}}.

\bibitem[Kwon et~al.(2023{\natexlab{b}})Kwon, Kwon, Wright, and
  Nowak]{kwon2023penalty}
Jeongyeol Kwon, Dohyun Kwon, Steve Wright, and Robert Nowak.
\newblock On penalty methods for nonconvex bilevel optimization and first-order
  stochastic approximation.
\newblock \emph{arXiv preprint arXiv:2309.01753}, 2023{\natexlab{b}}.

\bibitem[Lan(2012)]{lan2012optimal}
Guanghui Lan.
\newblock An optimal method for stochastic composite optimization.
\newblock \emph{Mathematical Programming}, 133\penalty0 (1-2):\penalty0
  365--397, 2012.

\bibitem[Li et~al.(2023)Li, Jadbabaie, and Rakhlin]{li2023convergence}
Haochuan Li, Ali Jadbabaie, and Alexander Rakhlin.
\newblock Convergence of adam under relaxed assumptions.
\newblock \emph{arXiv preprint arXiv:2304.13972}, 2023.

\bibitem[Liu et~al.(2022{\natexlab{a}})Liu, Ye, Wright, Stone, and
  Liu]{liu2022bome}
Bo~Liu, Mao Ye, Stephen Wright, Peter Stone, and Qiang Liu.
\newblock Bome! bilevel optimization made easy: A simple first-order approach.
\newblock \emph{Advances in Neural Information Processing Systems},
  35:\penalty0 17248--17262, 2022{\natexlab{a}}.

\bibitem[Liu et~al.(2020{\natexlab{a}})Liu, Yuan, Ying, and
  Yang]{liu2019stochastic}
Mingrui Liu, Zhuoning Yuan, Yiming Ying, and Tianbao Yang.
\newblock Stochastic auc maximization with deep neural networks.
\newblock \emph{ICLR}, 2020{\natexlab{a}}.

\bibitem[Liu et~al.(2022{\natexlab{b}})Liu, Zhuang, Lei, and
  Liao]{liu2022communication}
Mingrui Liu, Zhenxun Zhuang, Yunwen Lei, and Chunyang Liao.
\newblock A communication-efficient distributed gradient clipping algorithm for
  training deep neural networks.
\newblock \emph{Advances in Neural Information Processing Systems},
  35:\penalty0 26204--26217, 2022{\natexlab{b}}.

\bibitem[Liu et~al.(2020{\natexlab{b}})Liu, Mu, Yuan, Zeng, and
  Zhang]{liu2020generic}
Risheng Liu, Pan Mu, Xiaoming Yuan, Shangzhi Zeng, and Jin Zhang.
\newblock A generic first-order algorithmic framework for bi-level programming
  beyond lower-level singleton.
\newblock In \emph{International Conference on Machine Learning}, pages
  6305--6315. PMLR, 2020{\natexlab{b}}.

\bibitem[Liu et~al.(2023)Liu, Jagabathula, and Zhou]{liu2023near}
Zijian Liu, Srikanth Jagabathula, and Zhengyuan Zhou.
\newblock Near-optimal non-convex stochastic optimization under generalized
  smoothness.
\newblock \emph{arXiv preprint arXiv:2302.06032}, 2023.

\bibitem[Madden et~al.(2021)Madden, Becker, and Dall’Anese]{madden2021bounds}
Liam Madden, Stephen Becker, and Emiliano Dall’Anese.
\newblock Bounds for the tracking error of first-order online optimization
  methods.
\newblock \emph{Journal of Optimization Theory and Applications}, 189:\penalty0
  437--457, 2021.

\bibitem[Nesterov(1983)]{nesterov1983method}
Yurii Nesterov.
\newblock A method of solving a convex programming problem with convergence
  rate o (1/k2).
\newblock In \emph{Soviet Mathematics Doklady}, volume~27, pages 372--376,
  1983.

\bibitem[Rajeswaran et~al.(2019)Rajeswaran, Finn, Kakade, and
  Levine]{rajeswaran2019meta}
Aravind Rajeswaran, Chelsea Finn, Sham~M Kakade, and Sergey Levine.
\newblock Meta-learning with implicit gradients.
\newblock In \emph{Advances in Neural Information Processing Systems
  (NeurIPS)}, pages 113--124, 2019.

\bibitem[Reisizadeh et~al.(2023)Reisizadeh, Li, Das, and
  Jadbabaie]{reisizadeh2023variance}
Amirhossein Reisizadeh, Haochuan Li, Subhro Das, and Ali Jadbabaie.
\newblock Variance-reduced clipping for non-convex optimization.
\newblock \emph{arXiv preprint arXiv:2303.00883}, 2023.

\bibitem[Sabach and Shtern(2017)]{sabach2017first}
Shoham Sabach and Shimrit Shtern.
\newblock A first order method for solving convex bilevel optimization
  problems.
\newblock \emph{SIAM Journal on Optimization}, 27\penalty0 (2):\penalty0
  640--660, 2017.

\bibitem[Shen and Chen(2023)]{shen2023penalty}
Han Shen and Tianyi Chen.
\newblock On penalty-based bilevel gradient descent method.
\newblock \emph{arXiv preprint arXiv:2302.05185}, 2023.

\bibitem[Sow et~al.(2022)Sow, Ji, Guan, and Liang]{sow2022constrained}
Daouda Sow, Kaiyi Ji, Ziwei Guan, and Yingbin Liang.
\newblock A constrained optimization approach to bilevel optimization with
  multiple inner minima.
\newblock \emph{arXiv preprint arXiv:2203.01123}, 2022.

\bibitem[Tran-Dinh et~al.(2019)Tran-Dinh, Pham, Phan, and
  Nguyen]{tran2019hybrid}
Quoc Tran-Dinh, Nhan~H Pham, Dzung~T Phan, and Lam~M Nguyen.
\newblock Hybrid stochastic gradient descent algorithms for stochastic
  nonconvex optimization.
\newblock \emph{arXiv preprint arXiv:1905.05920}, 2019.

\bibitem[Vaswani et~al.(2017)Vaswani, Shazeer, Parmar, Uszkoreit, Jones, Gomez,
  Kaiser, and Polosukhin]{vaswani2017attention}
Ashish Vaswani, Noam Shazeer, Niki Parmar, Jakob Uszkoreit, Llion Jones,
  Aidan~N Gomez, {\L}ukasz Kaiser, and Illia Polosukhin.
\newblock Attention is all you need.
\newblock \emph{Advances in neural information processing systems}, 30, 2017.

\bibitem[Vaswani et~al.(2022)Vaswani, Dubois-Taine, and
  Babanezhad]{vaswani2022towards}
Sharan Vaswani, Benjamin Dubois-Taine, and Reza Babanezhad.
\newblock Towards noise-adaptive, problem-adaptive (accelerated) stochastic
  gradient descent.
\newblock In \emph{International conference on machine learning}, pages
  22015--22059. PMLR, 2022.

\bibitem[Vicente et~al.(1994)Vicente, Savard, and
  J{\'u}dice]{vicente1994descent}
Luis Vicente, Gilles Savard, and Joaquim J{\'u}dice.
\newblock Descent approaches for quadratic bilevel programming.
\newblock \emph{Journal of optimization theory and applications}, 81\penalty0
  (2):\penalty0 379--399, 1994.

\bibitem[Wang et~al.(2023)Wang, Zhang, Ma, and Chen]{wang2023convergence}
Bohan Wang, Huishuai Zhang, Zhiming Ma, and Wei Chen.
\newblock Convergence of adagrad for non-convex objectives: Simple proofs and
  relaxed assumptions.
\newblock In \emph{The Thirty Sixth Annual Conference on Learning Theory},
  pages 161--190. PMLR, 2023.

\bibitem[White and Anandalingam(1993)]{white1993penalty}
Douglas~J White and G~Anandalingam.
\newblock A penalty function approach for solving bi-level linear programs.
\newblock \emph{Journal of Global Optimization}, 3:\penalty0 397--419, 1993.

\bibitem[Wilson et~al.(2018)Wilson, Veeravalli, and
  Nedi{\'c}]{wilson2018adaptive}
Craig Wilson, Venugopal~V Veeravalli, and Angelia Nedi{\'c}.
\newblock Adaptive sequential stochastic optimization.
\newblock \emph{IEEE Transactions on Automatic Control}, 64\penalty0
  (2):\penalty0 496--509, 2018.

\bibitem[Yang et~al.(2021)Yang, Ji, and Liang]{yang2021provably}
Junjie Yang, Kaiyi Ji, and Yingbin Liang.
\newblock Provably faster algorithms for bilevel optimization.
\newblock \emph{Advances in Neural Information Processing Systems},
  34:\penalty0 13670--13682, 2021.

\bibitem[Ying et~al.(2016)Ying, Wen, and Lyu]{ying2016stochastic}
Yiming Ying, Longyin Wen, and Siwei Lyu.
\newblock Stochastic online auc maximization.
\newblock In \emph{Advances in Neural Information Processing Systems}, pages
  451--459, 2016.

\bibitem[Yuan et~al.(2021)Yuan, Yan, Sonka, and Yang]{yuan2021large}
Zhuoning Yuan, Yan Yan, Milan Sonka, and Tianbao Yang.
\newblock Large-scale robust deep auc maximization: A new surrogate loss and
  empirical studies on medical image classification.
\newblock In \emph{Proceedings of the IEEE/CVF International Conference on
  Computer Vision}, pages 3040--3049, 2021.

\bibitem[Zhang et~al.(2020{\natexlab{a}})Zhang, Jin, Fang, and
  Wang]{zhang2020improved}
Bohang Zhang, Jikai Jin, Cong Fang, and Liwei Wang.
\newblock Improved analysis of clipping algorithms for non-convex optimization.
\newblock \emph{Advances in Neural Information Processing Systems},
  33:\penalty0 15511--15521, 2020{\natexlab{a}}.

\bibitem[Zhang et~al.(2020{\natexlab{b}})Zhang, He, Sra, and
  Jadbabaie]{zhang2019gradient}
Jingzhao Zhang, Tianxing He, Suvrit Sra, and Ali Jadbabaie.
\newblock Why gradient clipping accelerates training: A theoretical
  justification for adaptivity.
\newblock \emph{International Conference on Learning Representations},
  2020{\natexlab{b}}.

\end{thebibliography}
	\bibliographystyle{plainnat}
	
	%%%%%%%%%%%%%%%%%%%%%%%%%%%%%%%%%%%%%%%%%%%%%%%%%%%%%%%%%%%%
	
	\appendix
	
	% \tableofcontents
\newpage

\section{Technical Lemmas}
\label{app:tech_lemmas}
In this section, we will introduce a few useful lemmas. The following technical lemma on recursive control is crucial for providing high probability guarantee of the lower-level variables $y_t$ and $\hat{y}_t$ in \Cref{alg:bilevel} at \textit{anytime}. We follow a similar argument as in \cite[Proposition 29]{cutler2023stochastic} with a slight generalization.

\begin{lemma}[Recursive control on MGF] \label{lm:recursive-control}
Consider scalar stochastic processes $(V_t)$, $(V'_{t,1})$, $(V'_{t,2})$, $(D_{t,1})$, $(D_{t,2})$ and $(X_t)$ on a probability space with filtration $(\gH_t)$, which are linked by the inequality
\begin{equation*}
    V_{t+1} \leq \alpha_tV_t + D_{t,1}\sqrt{V'_{t,1}} + D_{t,2}\sqrt{V'_{t,2}} + X_t + \kappa_t
\end{equation*}
for some deterministic constants $\alpha_t \in (-\infty, 1]$ and $\kappa_t\in\R$. Suppose the following properties hold.
\begin{itemize}
    \item $V_t, V'_{t,1}$ and $V'_{t,2}$ are non-negative and $\gH_t$-measurable.
    \item $D_{t,i}$ is mean-zero sub-Gaussian conditioned on $\gH_t$ with deterministic parameter $\sigma_i$, and $V'_{t,i}\leq V_t$ for $i=1,2$:
    \begin{equation*}
        \E[\exp(\lambda D_{t,i}) \mid \gH_t] \leq \exp(\lambda^2\sigma_i^2/2)
        \quad\text{for all}\quad
        \lambda \in \R.
    \end{equation*}
    \item $X_t$ is non-negative and sub-exponential conditioned on $\gH_t$ with deterministic parameter $\nu_t$:
    \begin{equation*}
        \E[\exp(\lambda X_t) \mid \gH_t] \leq \exp(\lambda\nu_t)
        \quad\text{for all}\quad
        0\leq \lambda\leq 1/\nu_t.
    \end{equation*}
\end{itemize}
Then the estimate
\begin{equation*}
    \E[\exp(\lambda V_{t+1})] \leq \exp(\lambda(\nu_t+\kappa_t))\E[\exp(\lambda(1+\alpha_t)V_t/2)]
\end{equation*}
holds for any $\lambda$ satisfying $0\leq \lambda\leq \min\left\{\frac{1-\alpha_t}{2(\sigma_1^2+\sigma_2^2)}, \frac{1}{2\nu_t}\right\}$.
\end{lemma}

\begin{proof}[Proof of \Cref{lm:recursive-control}]
For any index $t\geq0$ and any scalar $\lambda\geq0$, the law of total expectation implies
\begin{equation*}
    \begin{aligned}
        \E[\exp(\lambda V_{t+1})]
        &\leq \E\left[\exp(\lambda(\alpha_tV_t + D_{t,1}\sqrt{V'_{t,1}} + D_{t,2}\sqrt{V'_{t,2}} + X_t + \kappa_t))\right] \\
        &= \exp(\lambda\kappa_t)\E\left[\exp(\lambda\alpha_tV_t)\E\left[\exp\left(\lambda\left(D_{t,1}\sqrt{V'_{t,1}} + D_{t,2}\sqrt{V'_{t,2}}\right)\right)\exp(\lambda X_t) \mid \gH_t\right]\right].
    \end{aligned}
\end{equation*}
H\"{o}lder's inequality in turn yields
\begin{equation*}
    \begin{aligned}
        &\E\left[\exp(\lambda\alpha_tV_t)\E\left[\exp\left(\lambda\left(D_{t,1}\sqrt{V'_{t,1}} + D_{t,2}\sqrt{V'_{t,2}}\right)\right)\exp(\lambda X_t) \mid \gH_t\right]\right] \\
        &\quad\quad\leq \sqrt{\E\left[\exp\left(2\lambda\left(D_{t,1}\sqrt{V'_{t,1}} + D_{t,2}\sqrt{V'_{t,2}}\right)\right) \mid \gH_t\right] \cdot \E\left[\exp(2\lambda X_t) \mid \gH_t\right]} \\
        &\quad\quad\leq \sqrt{\exp\left(2\lambda^2(\sigma_i^2V'_{t,1}+\sigma_i^2V'_{t,2})\right)\exp(2\lambda\nu_t)} \\
        &\quad\quad\leq \exp\left(\lambda^2(\sigma_1^2+\sigma_2^2)V_t\right)\exp(\lambda\nu_t)
    \end{aligned}
\end{equation*}
provided $0\leq \lambda\leq 1/2\nu_t$, where we use $V'_{t,i}\leq V_t$ for $i=1,2$ in the last inequality. Therefore, under the condition that
\begin{equation*}
    0\leq \lambda\leq \min\left\{\frac{1-\alpha_t}{2(\sigma_1^2+\sigma_2^2)}, \frac{1}{2\nu_t}\right\},
\end{equation*}
the following estimate hold for all $t\geq0$:
\begin{equation*}
    \begin{aligned}
        \E[\exp(\lambda V_{t+1})]
        &\leq \exp(\lambda\kappa_t)\E\left[\exp(\lambda\alpha_tV_t)\exp\left(\lambda^2(\sigma_1^2+\sigma_2^2)V_t\right)\exp(\lambda\nu_t)\right] \\
        &= \exp(\lambda(\nu_t+\kappa_t))\E\left[\exp\left(\lambda(\alpha_t+\lambda(\sigma_1^2+\sigma_2^2))V_t\right)\right] \\
        &\leq \exp(\lambda(\nu_t+\kappa_t))\E[\exp(\lambda(1+\alpha_t)V_t/2)],
    \end{aligned}
\end{equation*}
where the last inequality follows by the given range of $\lambda$. Thus the proof is completed.
\end{proof}

Next, we introduce the following Young's inequality beyond Euclidean norm cases. This lemma serves as an important role when dealing with distributional drift for high probability SNAG analysis.

\begin{lemma}[Young's inequality] \label{lm:P-norm-young}
For any vectors $v_1,v_2\in\R^d$, positive semidefinite (PSD) matrix $\rmQ\in\R^{d\times d}$, and scalar $c>0$, it holds that \footnote{Here we define $\|v\|_{\rmQ} \coloneqq \sqrt{v^{\top}\rmQ v}$ for any vector $v\in\R^d$ and PSD matrix $\rmQ\in\R^{d\times d}$.}
\begin{equation*}
    \|v_1+v_2\|_{\rmQ}^2 \leq (1+c)\|v_1\|_{\rmQ}^2 + \left(1+\frac{1}{c}\right)\|v_2\|_{\rmQ}^2.
\end{equation*}
\end{lemma}

\begin{proof}[Proof of \Cref{lm:P-norm-young}]
By definition of $\|\cdot\|_{\rmQ}$, we have
\begin{equation} \label{eq:P-norm-1}
    \begin{aligned}
        \|v_1+v_2\|_{\rmQ}^2
        &= (v_1+v_2)^{\top}\rmQ(v_1+v_2) \\
        &= \|v_1\|_{\rmQ}^2 + \|v_2\|_{\rmQ}^2 + 2v_1^{\top}\rmQ v_2.
    \end{aligned}
\end{equation}
Since $\rmQ\in\R^{d\times d}$ is PSD, let $\rmQ = \rmU\rmU^{\top}$ be the Cholesky decomposition, then
\begin{equation} \label{eq:P-norm-2}
    \begin{aligned}
        2v_1^{\top}\rmQ v_2
        &= 2v_1^{\top}\rmU\rmU^{\top}v_2
        = 2(\rmU^{\top} v_1)^{\top}(\rmU^{\top} v_2) \\
        &\leq c\|\rmU^{\top}v_1\|^2 + \frac{1}{c}\|\rmU^{\top}v_2\|^2 \\
        &= c\|v_1\|_{\rmQ}^2 + \frac{1}{c}\|v_2\|_{\rmQ}^2,
    \end{aligned}
\end{equation}
where we use Young's inequality and definition of $\|\cdot\|_{\rmQ}$ for the second and third lines, respectively. Combing \eqref{eq:P-norm-1} and \eqref{eq:P-norm-2} gives the result as claimed.
\end{proof}

\section{Auxiliary Lemmas for Bilevel Optimization} 
\label{sec:auxiliary-bilevel}
In this section, we provide important properties of the objective function $\Phi$ in bilevel optimization problems, as well as characterizations (such as variance and bias) for stochastic hypergradient estimator $\gdf(x,y;\Bar{\xi})$ based on Neumann series. For readers' convenience, we only list the results here and defer the detailed proofs to \Cref{sec:omit-proof}.

\begin{lemma}[Lipschitz property, {\cite[Lemma 8]{hao2024bilevel}}] \label{lm:lip-y*z*}
Under \cref{ass:relax-smooth,ass:f-and-g}, $y^*(x)$ is $(l_{g,1}/\mu)$-Lipschitz continuous.
\end{lemma}

\begin{lemma}[$(L_0,L_1)$-smoothness, {\cite[Lemma 9]{hao2024bilevel}}] \label{lm:technical-rs-phi}
Under \cref{ass:relax-smooth,ass:f-and-g}, for any $x, x'$ we have
\begin{equation*}
    \|\gdphi(x)-\gdphi(x')\| \leq (L_0 + L_1\|\gdphi(x')\|)\|x-x'\| 
    \quad \text{if} \quad
    \|x-x'\| \leq \frac{1}{\sqrt{2(1+l_{g,1}^2/\mu^2)(L_{x,1}^2+L_{y,1}^2)}},
\end{equation*}
where $(L_0,L_1)$-smoothness constant $L_0$ and $L_1$ are defined as
\begin{equation*}
    L_0 = \sqrt{1+\frac{l_{g,1}^2}{\mu^2}}\left(L_{x,0} + L_{x,1}\frac{l_{g,1}l_{f,0}}{\mu} + \frac{l_{g,1}}{\mu}(L_{y,0}+L_{y,1}l_{f,0}) + l_{f,0}\frac{l_{g,1}l_{g,2}+l_{g,2}\mu}{\mu^2}\right) \quad \text{and} \quad
    L_1 = \sqrt{1+\frac{l_{g,1}^2}{\mu^2}}L_{x,1}.
\end{equation*}
% \begin{equation*}
%     \begin{aligned}
%         L_0 &= \sqrt{1+\frac{l_{g,1}^2}{\mu^2}}\left(L_{x,0} + L_{x,1}\frac{l_{g,1}l_{f,0}}{\mu} + \frac{l_{g,1}}{\mu}(L_{y,0}+L_{y,1}l_{f,0}) + l_{f,0}\frac{l_{g,1}l_{g,2}+l_{g,2}\mu}{\mu^2}\right) \\
%         L_1 &= \sqrt{1+\frac{l_{g,1}^2}{\mu^2}}L_{x,1}.
%     \end{aligned}
% \end{equation*}
\end{lemma}

\begin{lemma}[Descent inequality, {\cite[Lemma 10]{hao2024bilevel}}] \label{lm:technical-descent}
Suppose \cref{ass:relax-smooth,ass:f-and-g} and~\ref{ass:f-and-g} hold. 
Then for any $x, x'$ we have
\begin{equation*}
    \Phi(x) \leq \Phi(x') + \langle \gdphi(x'),x-x' \rangle + \frac{L_0+L_1\|\gdphi(x')\|}{2}\|x-x'\|^2
    \quad \text{if} \quad
    \|x-x'\| \leq \frac{1}{\sqrt{2(1+l_{g,1}^2/\mu^2)(L_{x,1}^2+L_{y,1}^2)}}.
\end{equation*}
\end{lemma}

\begin{lemma}[{\cite[Lemma B.1]{khanduri2021near}}] \label{lm:hyper-var}
Under \cref{ass:relax-smooth,ass:f-and-g,ass:noise,ass:individual-noise}, the bias of the stochastic hypergradient estimate of the upper-level objective satisfies
\begin{equation*}
    \|\gdf(x,y) - \E_{\Bar{\xi}}[\gdf(x,y;\Bar{\xi})]\| \leq \frac{l_{g,1}l_{f,0}}{\mu}\left(1-\frac{\mu}{l_{g,1}}\right)^Q,
\end{equation*}
where $Q$ is the number of samples chosen to approximate the Hessian inverse. Moreover, 
we have
\begin{equation*}
    \E_{\Bar{\xi}}[\|\gdf(x,y;\Bar{\xi}) - \E_{\Bar{\xi}}[\gdf(x,y;\Bar{\xi})]\|^2] \leq \sigma_{f,1}^2 + \frac{3}{\mu^2}\left[(\sigma_{f,1}^2+l_{f,0}^2)(\sigma_{g,2}^2+2l_{g,1}^2) + \sigma_{f,1}^2l_{g,1}^2\right] \coloneqq \Bar{\sigma}^2.
\end{equation*}
\end{lemma}

\begin{lemma} \label{lm:gdf-gdphi-bias}
Under \cref{ass:relax-smooth,ass:f-and-g,ass:noise,ass:individual-noise}, we have
\begin{equation*}
    \begin{aligned}
        \|\gdf(x,y) - \gdphi(x)\| \leq (\Bar{L} + L_{x,1}\|\gdphi(x)\|)\|y-y^*(x)\|,
    \end{aligned}
\end{equation*}
where constant $\Bar{L}$ is defined as
\begin{equation*}
    \Bar{L} \coloneqq L_{x,0} + L_{x,1}\frac{l_{g,1}l_{f,0}}{\mu} + \frac{l_{g,1}}{\mu}(L_{y,0}+L_{y,1}l_{f,0}) + l_{f,0}\frac{\mu l_{g,2}+l_{g,1}l_{g,2}}{\mu^2} \leq L_0.
\end{equation*}
\end{lemma}

\begin{lemma} \label{lm:hyper-estimator}
Under \cref{ass:relax-smooth,ass:f-and-g,ass:noise,ass:individual-noise}, we have
\begin{enumerate}[(i)]
    \item For any fixed $y\in\R^{d_y}$ and any $x,x'\in\R^{d_x}$,
    \begin{equation*}
        \E_{\Bar{\xi}}\|\gdf(x,y;\Bar{\xi}) - \gdf(x',y;\Bar{\xi})\|^2 \leq (\Bar{L}_0^2+\Bar{L}_1^2\|\gdphi(x)\|^2)\|x-x'\|^2.
    \end{equation*}
    \item For any fixed $x\in\R^{d_x}$ and any $y,y'\in\R^{d_y}$,
    \begin{equation*}
        \E_{\Bar{\xi}}\|\gdf(x,y;\Bar{\xi}) - \gdf(x,y';\Bar{\xi})\|^2 \leq (\Bar{L}_0^2+\Bar{L}_1^2\|\gdphi(x)\|^2)\|y-y'\|^2.
    \end{equation*}
\end{enumerate}
In the above expressions, we define $\Bar{L}_0$ and $\Bar{L}_1$ as 
% \begin{small}
\begin{equation*}
    \begin{aligned}
        \Bar{L}_0 &= \left\{4\left(L_{x,0} + L_{x,1}\left(\frac{l_{g,1}l_{f,0}}{\mu} + \left(L_{x,0} + \frac{L_{x,1}l_{g,1}l_{f,0}}{\mu}\right)\|y-y^*(x)\|\right)\right)^2 \right.\\
        &\left.\quad\quad\quad\quad\quad\quad\quad\quad+ \frac{6Q}{2\mu l_{g,1}-\mu^2}\left(l_{g,1}^2(L_{y,0}+L_{y,1}l_{f,0})^2 + l_{f,0}^2l_{g,2}^2 + \frac{l_{f,0}^2l_{g,1}^2l_{g,2}^2Q^2}{(l_{g,1}-\mu)^2}\right)\right\}^{1/2}, \\
        \Bar{L}_1 &= 2L_{x,1}(1+L_{x,1}\|y-y^*(x)\|).
    \end{aligned} 
\end{equation*}
% \end{small}
\end{lemma}

Note that in \Cref{lm:hyper-estimator}, constant $\Bar{L}_0$ depends on the value of $\|y-y^*(x)\|$. When we consider this term in \Cref{alg:bilevel}, it turns into $\|y_t-y_t^*\|$ or $\|\hat{y}_t-y_t^*\|$, which are both as small as $O(\epsilon)$ (and thus bounded) with high probability by \cref{lem:optionI,lem:optionII,lem:averaging}. In other words, we can treat this term as another constant for our algorithm and analysis.

\section{Proofs of Results in \Cref{sec:nesterov}}
\label{sec:proof-single-level}
For convenience, we will restate a few concepts included in \Cref{sec:nesterov} here. We consider the sequences of stochastic optimization problems 
\begin{equation} \label{eq:obj-time-drift}
    \min_{w\in\R^d} \phi_t(w) 
\end{equation}
indexed by time $t\in\mathbb{N}$. We denote the minimizer and the minimal value of $\phi_t$ as $w_t^*$ and $\phi_t^*$, and we define the \textit{minimizer drift} at time $t$ to be $\Delta_t\coloneqq \|w_t^*-w_{t+1}^*\|$. With a slight abuse of notation, we consider the SNAG algorithm applied to the sequence $\{\phi_t\}_{t=1}^{T}$, where $T$ is the total number of iterations:
\begin{equation}
\label{eq:NAG-single}
\begin{aligned}
    z_t &= w_t + \gamma(w_t-w_{t-1})  \\
    w_{t+1} &= w_t + \gamma(w_t-w_{t-1}) - \alpha g_t, 
\end{aligned}
\end{equation}
where $g_t=\nabla \phi_t(z_t;\xi_t)$ is the stochastic gradient evaluated at $z_t$ with random sample $\xi_t$. Define $\varepsilon_t=g_t-\nabla \phi_t(z_t)$ as the stochastic gradient noise at $t$-th iteration. Define $\mathcal{H}_t=\sigma(\xi_1,\ldots,\xi_{t-1})$ as the filtration, which is the $\sigma$-algebra generated by all random variables until $t$-th iteration. We will make the following standard assumption, as illustrated below~\footnote{Note that \cref{ass:sc-smooth-single-level,ass:noise-drift-single-level} are more concrete than that in \Cref{sec:nesterov}.}.

\begin{assumption} \label{ass:sc-smooth-single-level}
The sequences of time-varying functions satisfy that, each function $\phi_t : \R^d \rightarrow \R$ is $\mu$-strongly convex and $L$-smooth for some constants $\mu,L > 0$.
\end{assumption}

\begin{assumption}[Sub-Gaussian drift and noise] \label{ass:noise-drift-single-level}
There exists constants $\Delta,\sigma>0$ such that the following holds for all $t\geq0$:
\begin{enumerate}[(i)]
    \item (\textbf{Drift}) The drift $\Delta_t^2$ is sub-exponential conditioned on $\gH_t$ with parameter $\Delta^2$:
    \begin{equation*}
        \E\left[\exp(\lambda\Delta_t^2) \mid \gH_t\right] \leq \exp(\lambda\Delta^2)
        \quad\text{for all}\quad
        0\leq \lambda\leq \Delta^{-2}.
    \end{equation*}
    \item (\textbf{Noise}) The noise $\varepsilon_t$ is norm sub-Gaussian conditioned on $\gH_t$ with parameter $\sigma/2$:
    \begin{equation*}
        \pr\left\{\|\varepsilon_t\|\geq \varrho \mid \gH_t \right\} \leq 2\exp(-2\varrho^2/\sigma^2)
        \quad\text{for all}\quad
        \varrho > 0.
    \end{equation*}
\end{enumerate}
\end{assumption}

The following lemma characterize the one-step improvement for stochastic Nesterov accelerated gradient method. Although part of our analysis is similar to \cite{chen2023convergence,aybat2020robust}, our final goal is quite different: we aim to derive a careful formulation (see \eqref{eq:one-step-single-level}) such that we can apply \Cref{lm:recursive-control} to recursively control the moment generating function of $V_t$ with distributional drift, thus leading to a high probability bound for $V_t$ at \textit{anytime} (see \Cref{lm:dist-track-single-level}), while \cite{chen2023convergence,aybat2020robust} only show the convergence in expectation without distributional drift. 

\begin{lemma}[Distance recursion, with drift] \label{lm:distance-recursion-single-level}
Suppose that \cref{ass:sc-smooth-single-level,ass:noise-drift-single-level} hold. Let $\{w_t\}$ be the iterates produced by update rule \eqref{eq:NAG-single} with constant learning rate $\alpha\leq 1/2L$ and set constants $\gamma, \rho>0$, and matrix $\rmP\in\R^{2d\times 2d}$ as
\begin{equation} \label{eq:gamma-rhp-P}
    \gamma = \frac{1-\sqrt{\mu\alpha}}{1+\sqrt{\mu\alpha}},
    \quad
    \rho^2 = 1-\sqrt{\mu\alpha},
    \quad
    \rmP = \frac{1}{2\alpha}
    \begin{bmatrix}
        1 & \sqrt{\mu\alpha}-1 \\
        \sqrt{\mu\alpha}-1 & (1-\sqrt{\mu\alpha})^2
    \end{bmatrix} \otimes \rmI_d.
\end{equation}
Define $\theta_t = [(w_t-w_t^*)^{\top}, (w_{t-1}-w_t^*)^{\top}]^{\top} \in \R^{2d}$, also define the potential function and $u_{t,1}, u_{t,2}$ as 
\begin{equation} \label{eq:potential-V}
    V_t = \theta_t^{\top}\rmP\theta_t + \phi_t(w_t) - \phi_t(w_t^*), 
    \quad\quad
    u_{t,1} = \frac{w_t-w_t^*}{\|w_t-w_t^*\|},
    \quad\quad
    u_{t,2} = \frac{z_t-w_t}{\|z_t-w_t\|}.
\end{equation}
Then for all $t\geq0$, it holds that
\begin{equation} \label{eq:all-d}
    \begin{aligned}
        \begin{bmatrix}
            w_{t+1} - w_t^* \\
            w_t - w_t^* 
        \end{bmatrix}^{\top}
        \rmP
        \begin{bmatrix}
            w_{t+1} - w_t^* \\
            w_t - w_t^* 
        \end{bmatrix}
        &+ \phi_t(w_{t+1}) - \phi_t(w_t^*) \\
        &\leq \rho^2V_t - \alpha(1-L\alpha)\langle \nabla\phi_t(z_t), \varepsilon_t \rangle + \frac{L\alpha^2}{2}\|\varepsilon_t\|^2.
    \end{aligned}
\end{equation}
Specifically, if $\phi_t(w) \coloneqq \frac{\mu}{2}\|w-w_t^*\|^2$, then we have
\begin{equation} \label{eq:one-step-single-level}
    \begin{aligned}
        V_{t+1}
        &\leq \left(1-\frac{\sqrt{\mu\alpha}}{2}\right)V_t + \left(1+\frac{\sqrt{\mu\alpha}}{4}\right) 
        \left[- \sqrt{2\mu}\alpha(1-L\alpha)\langle u_{t,1}, \varepsilon_t \rangle\sqrt{\frac{\mu}{2}}\|w_t-w_t^*\| \right.\\
        &\left.- \frac{2\mu\sqrt{2\alpha}}{1+\sqrt{\mu\alpha}}\alpha(1-L\alpha)\langle u_{t,2}, \varepsilon_t \rangle\frac{1+\sqrt{\mu\alpha}}{2\sqrt{2\alpha}}\|z_t-w_t\| + \frac{\alpha(1+L\alpha)}{2}\|\varepsilon_t\|^2\right] + \frac{20\mu\Delta_t^2}{\sqrt{\mu\alpha}}.
    \end{aligned}
\end{equation}
\end{lemma}

\begin{proof}[Proof of \Cref{lm:distance-recursion-single-level}]
We first apply \Cref{lm:P-norm-young} with 
\begin{equation*}
    v_1+v_2 = \theta_{t+1} = 
    \begin{bmatrix}
        w_{t+1}-w_{t+1}^* \\
        w_t - w_{t+1}^*
    \end{bmatrix},
    \quad
    v_1 = 
    \begin{bmatrix}
        w_{t+1} - w_t^* \\
        w_t - w_t^* 
    \end{bmatrix},
    \quad
    v_2 = 
    \begin{bmatrix}
        w_t^* - w_{t+1}^* \\
        w_t^* - w_{t+1}^*
    \end{bmatrix},
    \quad
    \rmQ = \rmP
\end{equation*}
to obtain
\begin{small}
\begin{equation*}
    \begin{aligned}
        &V_{t+1} 
        = \theta_{t+1}^{\top}\rmP\theta_{t+1} + \phi_{t+1}(w_{t+1}) - \phi_{t+1}(w_{t+1}^*) \\
        &\leq \left(1+\frac{\sqrt{\mu\alpha}}{4}\right)
        \begin{bmatrix}
            w_{t+1} - w_t^* \\
            w_t - w_t^* 
        \end{bmatrix}^{\top}
        \rmP
        \begin{bmatrix}
            w_{t+1} - w_t^* \\
            w_t - w_t^* 
        \end{bmatrix}
        + 
        \left(1+\frac{4}{\sqrt{\mu\alpha}}\right)
        \begin{bmatrix}
            w_t^* - w_{t+1}^* \\
            w_t^* - w_{t+1}^*
        \end{bmatrix}^{\top}
        \rmP
        \begin{bmatrix}
            w_t^* - w_{t+1}^* \\
            w_t^* - w_{t+1}^*
        \end{bmatrix}
        + \phi_{t+1}(w_{t+1}) - \phi_{t+1}(w_{t+1}^*) \\
        &= 
        \underbrace{\left(1+\frac{\sqrt{\mu\alpha}}{4}\right)\left\{
        \begin{bmatrix}
            w_{t+1} - w_t^* \\
            w_t - w_t^* 
        \end{bmatrix}^{\top}
        \rmP
        \begin{bmatrix}
            w_{t+1} - w_t^* \\
            w_t - w_t^* 
        \end{bmatrix} 
        + \phi_t(w_{t+1}) - \phi_t(w_t^*)
        \right\}}_{(A)} \\
        &\quad+ \underbrace{\phi_{t+1}(w_{t+1}) - \phi_{t+1}(w_{t+1}^*) - \left(1+\frac{\sqrt{\mu\alpha}}{4}\right)(\phi_t(w_{t+1}) - \phi_t(w_t^*))}_{(B)} 
        + \underbrace{\left(1+\frac{4}{\sqrt{\mu\alpha}}\right)
        \begin{bmatrix}
            w_t^* - w_{t+1}^* \\
            w_t^* - w_{t+1}^*
        \end{bmatrix}^{\top}
        \rmP
        \begin{bmatrix}
            w_t^* - w_{t+1}^* \\
            w_t^* - w_{t+1}^*
        \end{bmatrix}}_{(C)}.
    \end{aligned}
\end{equation*}
\end{small}%
Now we bound terms $(A)$, $(B)$ and $(C)$ individually. 

\paragraph{Bounding $(A)$.}
Let us define vector $\omega_t\in\R^{d}$ and matrices $\rmA\in\R^{2d\times 2d}, \rmB\in\R^{2d\times d}$ as 
% $\rmA,\rmB\in\R^{2d\times 2d}$ as 
\begin{equation*}
    % \omega_t = 
    % \begin{bmatrix}
    %     \nabla \phi_t(z_t) \\
    %     \varepsilon_t
    % \end{bmatrix},
    \omega_t = \nabla \phi_t(z_t),
    \quad
    \rmA = 
    \begin{bmatrix}
        1+\gamma & -\gamma \\
        1 & 0
    \end{bmatrix} \otimes \rmI_d,
    \quad
    % \rmB = 
    % \begin{bmatrix}
    %     -\alpha & -\alpha \\
    %     0 & 0
    % \end{bmatrix} \otimes \rmI_d.
    \rmB = 
    \begin{bmatrix}
        -\alpha \\
        0 
    \end{bmatrix} \otimes \rmI_d.
\end{equation*}
By \eqref{eq:NAG-single} we have 
\begin{equation} \label{eq:dynamic}
    \begin{aligned}
        [(w_{t+1} - w_t^*)^{\top}, (w_t - w_t^*)^{\top}]^{\top} = \rmA\theta_t + \rmB(\omega_t + \varepsilon_t)
    \end{aligned}
\end{equation}
% $[(w_{t+1} - w_t^*)^{\top}, (w_t - w_t^*)^{\top}]^{\top} = \rmA\theta_t + \rmB(\omega_t + \varepsilon_t)$. 
% $[(w_{t+1} - w_t^*)^{\top}, (w_t - w_t^*)^{\top}]^{\top} = \rmA\theta_t + \rmB\omega_t$. 
% \begin{equation*}
%     \begin{aligned}
%         \begin{bmatrix}
%             w_{t+1} - w_t^* \\
%             w_t - w_t^* 
%         \end{bmatrix}
%         = \rmA\theta_t + \rmB\omega_t.
%     \end{aligned}
% \end{equation*}
Since $\phi_t$ is $\mu$-strongly convex, then
\begin{equation} \label{eq:phi-t-SC}
    \phi_t(w_t) - \phi_t(z_t) \geq \langle \nabla\phi_t(z_t), w_t-z_t \rangle + \frac{\mu}{2}\|w_t-z_t\|^2.
\end{equation}
By $L$-smoothness of $\phi_t$ and the fact that $w_{t+1} = z_t - \alpha g_t$, we have
\begin{equation} \label{eq:phi-t-smooth}
    \begin{aligned}
        \phi_t(z_t) &- \phi_t(w_{t+1})
        \geq \langle \nabla\phi_t(z_t), z_t-w_{t+1} \rangle - \frac{L}{2}\|w_{t+1}-z_t\|^2 \\
        &= \alpha\langle \nabla\phi_t(z_t), g_t \rangle - \frac{L\alpha^2}{2}\|g_t\|^2 \\
        &= \alpha\|\nabla\phi_t(z_t)\|^2 + \alpha\langle \nabla\phi_t(z_t), \varepsilon_t \rangle - \frac{L\alpha^2}{2}(\|\nabla\phi_t(z_t)\|^2 + 2\langle \nabla\phi_t(z_t), \varepsilon_t \rangle + \|\varepsilon_t\|^2) \\
        &= \frac{\alpha}{2}(2-L\alpha)\|\nabla\phi_t(z_t)\|^2 - \frac{L\alpha^2}{2}\|\varepsilon_t\|^2 + \alpha(1-L\alpha)\langle \nabla\phi_t(z_t), \varepsilon_t \rangle,
    \end{aligned}
\end{equation}
where we use $g_t = \nabla\phi_t(z_t) + \varepsilon_t$ in the second equality. Noting that by \eqref{eq:NAG-single} we have
\begin{equation*}
    w_t-z_t = -\gamma(w_t-w_t^*) + \gamma(w_{t-1}-w_t^*),
\end{equation*}
and combining \eqref{eq:phi-t-SC} and \eqref{eq:phi-t-smooth} we obtain 
\begin{equation} \label{eq:X_1}
    \begin{aligned}
        \phi_t(w_t) - \phi_t(w_{t+1}) 
        &\geq \langle \nabla\phi_t(z_t), w_t-z_t \rangle + \frac{\mu}{2}\|w_t-z_t\|^2 \\
        &\quad+ \frac{\alpha}{2}(2-L\alpha)\|\nabla\phi_t(z_t)\|^2 - \frac{L\alpha^2}{2}\|\varepsilon_t\|^2 + \alpha(1-L\alpha)\langle \nabla\phi_t(z_t), \varepsilon_t \rangle \\
        &= 
        \begin{bmatrix}
            \theta_t \\
            \omega_t
        \end{bmatrix}^{\top}
        \rmX_1
        \begin{bmatrix}
            \theta_t \\
            \omega_t
        \end{bmatrix}
        - \frac{L\alpha^2}{2}\|\varepsilon_t\|^2 + \alpha(1-L\alpha)\langle \nabla\phi_t(z_t), \varepsilon_t \rangle,
    \end{aligned}
\end{equation}
where matrix $\rmX_1\in\R^{3d\times3d}$ is defined as
% where matrix $\rmX_1\in\R^{4d\times4d}$ is defined as
% \begin{equation*}
%     \rmX_1 = \frac{1}{2}
%     \begin{bmatrix}
%         \mu\gamma^2 & -\mu\gamma^2 & -\gamma & 0 \\
%         -\mu\gamma^2 & \mu\gamma^2 & \gamma & 0 \\
%         -\gamma & \gamma & \alpha(2-L\alpha) & 0 \\
%         0 & 0 & 0 & 0
%     \end{bmatrix} \otimes \rmI_d.
% \end{equation*}
\begin{equation*}
    \rmX_1 = \frac{1}{2}
    \begin{bmatrix}
        \mu\gamma^2 & -\mu\gamma^2 & -\gamma \\
        -\mu\gamma^2 & \mu\gamma^2 & \gamma \\
        -\gamma & \gamma & \alpha(2-L\alpha) \\
    \end{bmatrix} \otimes \rmI_d.
\end{equation*}
Then applying the strong convexity of $\phi_t$ again gives
\begin{equation} \label{eq:phi-t-SC-1}
    \phi_t(w_t^*) - \phi_t(z_t) \geq \langle \nabla\phi_t(z_t), w_t^*-z_t \rangle + \frac{\mu}{2}\|w_t^*-z_t\|^2.
\end{equation}
Noting that
\begin{equation*}
    \begin{aligned}
        w_t^* - z_t
        &= (w_t^*-w_t) - \gamma(w_t-w_t^*) + \gamma(w_{t-1}-w_t^*) \\
        &= -(1+\gamma)(w_t-w_t^*) + \gamma(w_{t-1}-w_t^*),
    \end{aligned}
\end{equation*}
and combining \eqref{eq:phi-t-smooth} and \eqref{eq:phi-t-SC-1} we obtain
\begin{equation} \label{eq:X_2}
    \begin{aligned}
        \phi_t(w_t^*) - \phi_t(w_{t+1}) 
        &\geq \langle \nabla\phi_t(z_t), w_t^*-z_t \rangle + \frac{\mu}{2}\|w_t^*-z_t\|^2 \\
        &\quad+ \frac{\alpha}{2}(2-L\alpha)\|\nabla\phi_t(z_t)\|^2 - \frac{L\alpha^2}{2}\|\varepsilon_t\|^2 + \alpha(1-L\alpha)\langle \nabla\phi_t(z_t), \varepsilon_t \rangle \\
        &= 
        \begin{bmatrix}
            \theta_t \\
            \omega_t
        \end{bmatrix}^{\top}
        \rmX_2
        \begin{bmatrix}
            \theta_t \\
            \omega_t
        \end{bmatrix}
        - \frac{L\alpha^2}{2}\|\varepsilon_t\|^2 + \alpha(1-L\alpha)\langle \nabla\phi_t(z_t), \varepsilon_t \rangle,
    \end{aligned}
\end{equation}
where matrix $\rmX_2\in\R^{3d\times3d}$ is defined as
% where matrix $\rmX_2\in\R^{4d\times4d}$ is defined as
% \begin{equation*}
%     \rmX_2 = \frac{1}{2}
%     \begin{bmatrix}
%         \mu(1+\gamma)^2 & -\mu\gamma(1+\gamma) & -(1+\gamma) & 0 \\
%         -\mu\gamma(1+\gamma) & \mu\gamma^2 & \gamma & 0 \\
%         -(1+\gamma) & \gamma & \alpha(2-L\alpha) & 0 \\
%         0 & 0 & 0 & 0
%     \end{bmatrix} \otimes \rmI_d.
% \end{equation*}
\begin{equation*}
    \rmX_2 = \frac{1}{2}
    \begin{bmatrix}
        \mu(1+\gamma)^2 & -\mu\gamma(1+\gamma) & -(1+\gamma) \\
        -\mu\gamma(1+\gamma) & \mu\gamma^2 & \gamma \\
        -(1+\gamma) & \gamma & \alpha(2-L\alpha) \\
    \end{bmatrix} \otimes \rmI_d.
\end{equation*}
Next, we multiply \eqref{eq:X_1} by $\rho^2$ and \eqref{eq:X_2} by $1-\rho^2$, then sum them up to get
\begin{equation} \label{eq:multiply-step}
    \begin{aligned}
        \rho^2(\phi_t(w_t)&-\phi_t(w_t^*)) - (\phi_t(w_{t+1})-\phi_t(w_t^*)) \\
        &\geq 
        \begin{bmatrix}
            \theta_t \\
            \omega_t
        \end{bmatrix}^{\top}
        (\rho^2\rmX_1 + (1-\rho^2)\rmX_2)
        \begin{bmatrix}
            \theta_t \\
            \omega_t
        \end{bmatrix}
        - \frac{L\alpha^2}{2}\|\varepsilon_t\|^2 + \alpha(1-L\alpha)\langle \nabla\phi_t(z_t), \varepsilon_t \rangle.
    \end{aligned}
\end{equation}
% With definition of $\gamma,\rho^2$ and $\rmP$ given in the statement of the theorem, by direct calculation we have
% \begin{equation*}
%     \begin{aligned}
%         &
%         \begin{bmatrix}
%             \rmA^{\top}\rmP\rmA-\rho^2\rmP & \rmA^{\top}\rmP\rmB \\
%             \rmB^{\top}\rmP\rmA & \rmB^{\top}\rmP\rmB
%         \end{bmatrix}
%         - (\rho^2\rmX_1 + (1-\rho^2)\rmX_2) \\
%         &\quad\quad\quad\quad\quad= 
%         \begin{bmatrix}
%             -\frac{\mu(1-\sqrt{\mu\alpha})^3}{2\sqrt{\mu\alpha}(1+\sqrt{\mu\alpha})} & \frac{\mu(1-\sqrt{\mu\alpha})^3}{2\sqrt{\mu\alpha}(1+\sqrt{\mu\alpha})} & 0 & -\frac{1+\mu\alpha}{2(1+\sqrt{\mu\alpha})} \\
%             \frac{\mu(1-\sqrt{\mu\alpha})^3}{2\sqrt{\mu\alpha}(1+\sqrt{\mu\alpha})} & -\frac{\mu(1-\sqrt{\mu\alpha})^3}{2\sqrt{\mu\alpha}(1+\sqrt{\mu\alpha})} & 0 & \frac{1-\sqrt{\mu\alpha}}{2(1+\sqrt{\mu\alpha})} \\
%             0 & 0 & -\frac{\alpha(1-L\alpha)}{2} & \frac{\alpha}{2} \\
%             -\frac{1+\mu\alpha}{2(1+\sqrt{\mu\alpha})} & \frac{1-\sqrt{\mu\alpha}}{2(1+\sqrt{\mu\alpha})} & \frac{\alpha}{2} & \frac{\alpha}{2}
%         \end{bmatrix} \otimes \rmI_d
%         := \rmC,
%     \end{aligned}
% \end{equation*}
% where matrix $\rmC$ is negative semidefinite \cite{chen2023convergence}, i.e., $\rmC\preceq\mathbf{0}$. Then we use this fact to get
By \cite[Section 3.1]{hu2017dissipativity} (see also \cite[Corollary 4.9]{aybat2020robust}, \cite[Theorem 2.3]{aybat2019universally}), we have the following fact:
\begin{equation} \label{eq:nsd-fact}
    \begin{aligned}
        \begin{bmatrix}
            \rmA^{\top}\rmP\rmA-\rho^2\rmP & \rmA^{\top}\rmP\rmB \\
            \rmB^{\top}\rmP\rmA & \rmB^{\top}\rmP\rmB
        \end{bmatrix}
        - (\rho^2\rmX_1 + (1-\rho^2)\rmX_2)
        \preceq 0,
    \end{aligned}
\end{equation}
which, combined with \eqref{eq:multiply-step} yields
\begin{equation*}
    \begin{aligned}
        &
        \begin{bmatrix}
            w_{t+1} - w_t^* \\
            w_t - w_t^* 
        \end{bmatrix}^{\top}
        \rmP
        \begin{bmatrix}
            w_{t+1} - w_t^* \\
            w_t - w_t^* 
        \end{bmatrix}
        - \rho^2\theta_t^{\top}\rmP\theta_t \\
        &\quad= 
        % \begin{bmatrix}
        %     \theta_t \\
        %     \omega_t
        % \end{bmatrix}^{\top}
        % (\rho^2\rmX_1 + (1-\rho^2)\rmX_2 + \rmC)
        % \begin{bmatrix}
        %     \theta_t \\
        %     \omega_t
        % \end{bmatrix} \\
        \begin{bmatrix}
            \theta_t \\
            \omega_t
        \end{bmatrix}^{\top}
        \begin{bmatrix}
            \rmA^{\top}\rmP\rmA-\rho^2\rmP & \rmA^{\top}\rmP\rmB \\
            \rmB^{\top}\rmP\rmA & \rmB^{\top}\rmP\rmB
        \end{bmatrix}
        \begin{bmatrix}
            \theta_t \\
            \omega_t
        \end{bmatrix}
        + \varepsilon_t^{\top}\rmB^{\top}\rmP\rmB\varepsilon_t \\
        &\quad\leq 
        \begin{bmatrix}
            \theta_t \\
            \omega_t
        \end{bmatrix}^{\top}
        (\rho^2\rmX_1 + (1-\rho^2)\rmX_2)
        \begin{bmatrix}
            \theta_t \\
            \omega_t
        \end{bmatrix}
        + \frac{\alpha}{2}\|\varepsilon_t\|^2 \\
        &\quad\leq- (\phi_t(w_{t+1})-\phi_t(w_t^*)) + \rho^2(\phi_t(w_t)-\phi_t(w_t^*)) - \alpha(1-L\alpha)\langle \nabla\phi_t(z_t), \varepsilon_t \rangle + \frac{\alpha+L\alpha^2}{2}\|\varepsilon_t\|^2,
    \end{aligned}
\end{equation*}
where the first inequality uses \eqref{eq:dynamic} and \eqref{eq:nsd-fact}, along with the fact that $\rmB^{\top}\rmP\rmB = (\alpha/2)\otimes\rmI_d$ and hence $\lambda_{\max}(\rmB^{\top}\rmP\rmB) = \alpha/2$; and the last inequality follows by \eqref{eq:multiply-step}. Rearrange the above inequality and by definition of the potential function $V_t$, we obtain
\begin{equation} \label{eq:nag-one-step}
    \begin{aligned}
        \begin{bmatrix}
            w_{t+1} - w_t^* \\
            w_t - w_t^* 
        \end{bmatrix}^{\top}
        \rmP
        \begin{bmatrix}
            w_{t+1} - w_t^* \\
            w_t - w_t^* 
        \end{bmatrix}
        &+ \phi_t(w_{t+1}) - \phi_t(w_t^*) \\
        &\leq \rho^2V_t - \alpha(1-L\alpha)\langle \nabla\phi_t(z_t), \varepsilon_t \rangle + \frac{\alpha(1+L\alpha)}{2}\|\varepsilon_t\|^2.
    \end{aligned}
\end{equation}
Now recall that in \eqref{eq:obj-time-drift} our objective function has the form of $\phi_t(w) = \frac{\mu}{2}\|w-w_t^*\|^2$, hence $\nabla\phi_t(z_t) = \mu(z_t-w_t^*)$. Plugging this into the above inequality gives
\begin{equation}
    \begin{aligned}
        &\begin{bmatrix}
            w_{t+1} - w_t^* \\
            w_t - w_t^* 
        \end{bmatrix}^{\top}
        \rmP
        \begin{bmatrix}
            w_{t+1} - w_t^* \\
            w_t - w_t^* 
        \end{bmatrix}
        + \phi_t(w_{t+1}) - \phi_t(w_t^*) \\
        &\quad\quad\leq \rho^2V_t - \mu\alpha(1-L\alpha)\langle z_t-w_t^*, \varepsilon_t \rangle + \frac{\alpha+L\alpha^2}{2}\|\varepsilon_t\|^2 \\
        &\quad\quad= \rho^2V_t - \mu\alpha(1-L\alpha)\langle w_t-w_t^*, \varepsilon_t \rangle - \mu\alpha(1-L\alpha)\langle z_t-w_t, \varepsilon_t \rangle + \frac{\alpha(1+L\alpha)}{2}\|\varepsilon_t\|^2 \\
        &\quad\quad= \rho^2V_t - \sqrt{2\mu}\alpha(1-L\alpha)\langle u_{t,1}, \varepsilon_t \rangle\sqrt{\frac{\mu}{2}}\|w_t-w_t^*\| \\
        &\quad\quad\quad\quad\quad\quad- \frac{2\mu\sqrt{2\alpha}}{1+\sqrt{\mu\alpha}}\alpha(1-L\alpha)\langle u_{t,2}, \varepsilon_t \rangle\frac{1+\sqrt{\mu\alpha}}{2\sqrt{2\alpha}}\|z_t-w_t\| + \frac{\alpha(1+L\alpha)}{2}\|\varepsilon_t\|^2,
    \end{aligned}
\end{equation}
where $u_{t,1}$ and $u_{t,2}$ are defined as
\begin{equation*}
    u_{t,1} = \frac{w_t-w_t^*}{\|w_t-w_t^*\|},
    \quad\quad
    u_{t,2} = \frac{z_t-w_t}{\|z_t-w_t\|}.
\end{equation*}
Therefore, we conclude that
\begin{equation} \label{eq:bound-A}
    \begin{aligned}
        (A) 
        &\leq \left(1+\frac{\sqrt{\mu\alpha}}{4}\right)
        \left[\rho^2V_t - \sqrt{2\mu}\alpha(1-L\alpha)\langle u_{t,1}, \varepsilon_t \rangle\sqrt{\frac{\mu}{2}}\|w_t-w_t^*\| \right.\\
        &\left.\quad\quad\quad\quad\quad- \frac{2\mu\sqrt{2\alpha}}{1+\sqrt{\mu\alpha}}\alpha(1-L\alpha)\langle u_{t,2}, \varepsilon_t \rangle\frac{1+\sqrt{\mu\alpha}}{2\sqrt{2\alpha}}\|z_t-w_t\| + \frac{\alpha(1+L\alpha)}{2}\|\varepsilon_t\|^2\right] \\
        &\leq \left(1-\frac{3\sqrt{\mu\alpha}}{4}\right)V_t + 
        \left[- \sqrt{2\mu}\alpha(1-L\alpha)\langle u_{t,1}, \varepsilon_t \rangle\sqrt{\frac{\mu}{2}}\|w_t-w_t^*\| \right.\\
        &\left.\quad\quad\quad\quad\quad- \frac{2\mu\sqrt{2\alpha}}{1+\sqrt{\mu\alpha}}\alpha(1-L\alpha)\langle u_{t,2}, \varepsilon_t \rangle\frac{1+\sqrt{\mu\alpha}}{2\sqrt{2\alpha}}\|z_t-w_t\| + \frac{\alpha(1+L\alpha)}{2}\|\varepsilon_t\|^2\right], \\
    \end{aligned}
\end{equation}
where the last inequality follows from the definition of $\rho$ and simple calculation
\begin{equation*}
    \left(1+\frac{\sqrt{\mu\alpha}}{4}\right)\rho^2
    = \left(1+\frac{\sqrt{\mu\alpha}}{4}\right)(1-\sqrt{\mu\alpha})
    \leq 1-\frac{3\sqrt{\mu\alpha}}{4}.
\end{equation*}

\paragraph{Bounding $(B)$.}
Recall that under distributional drift, the objective function in \eqref{eq:obj-time-drift} has the form of $\phi_t(w) = \frac{\mu}{2}\|w-w_t^*\|^2$, then we have
\begin{equation*}
    \begin{aligned}
        (B) 
        &= \phi_{t+1}(w_{t+1}) - \phi_{t+1}(w_{t+1}^*) - \left(1+\frac{\sqrt{\mu\alpha}}{4}\right)(\phi_t(w_{t+1}) - \phi_t(w_t^*)) \\
        &\leq \left(1+\frac{\sqrt{\mu\alpha}}{4}\right)(\phi_{t+1}(w_{t+1}) - \phi_{t+1}(w_{t+1}^*) - \phi_t(w_{t+1}) + \phi_t(w_t^*)) \\
        &= \frac{\mu}{2}\left(1+\frac{\sqrt{\mu\alpha}}{4}\right)(\|w_{t+1}-w_{t+1}^*\|^2 - \|w_{t+1}-w_t^*\|^2) \\
        &\leq \frac{\mu}{2}\left(1+\frac{\sqrt{\mu\alpha}}{4}\right)\|w_t^*-w_{t+1}^*\|\|w_{t+1}-w_t^*+w_{t+1}-w_{t+1}^*\| \\
        &\leq \frac{\mu}{2}\left(1+\frac{\sqrt{\mu\alpha}}{4}\right)\Delta_t(2\|w_{t+1}-w_{t+1}^*\| + \|w_{t+1}^*-w_t^*\|) \\
    \end{aligned}
\end{equation*}
Since $\phi_{t+1}$ is $\mu$-strongly convex and matrix $\rmP$ is PSD, then 
\begin{equation*}
    \begin{aligned}
        V_{t+1} 
        &= \theta_{t+1}^{\top}\rmP\theta_{t+1} + \phi_{t+1}(w_{t+1}) - \phi_t(w_{t+1}^*)
        \geq \phi_{t+1}(w_{t+1}) - \phi_t(w_{t+1}^*) \\
        &\geq \frac{\mu}{2}\|w_{t+1}-w_{t+1}^*\|^2
        \quad\quad\quad\quad\Longrightarrow\quad\quad\quad\quad
        \|w_{t+1}-w_{t+1}^*\| \leq \sqrt{\frac{2}{\mu}}\sqrt{V_{t+1}}.
    \end{aligned}
\end{equation*}
Plugging the above fact back into the upper bound for $(B)$ gives
\begin{equation} \label{eq:bound-B}
    \begin{aligned}
        (B)
        &\leq \frac{\mu}{2}\left(1+\frac{\sqrt{\mu\alpha}}{4}\right)\Delta_t\left(2\sqrt{\frac{2}{\mu}}\sqrt{V_{t+1}} + \Delta_t\right) \\
        &= \sqrt{2\mu}\Delta_t\left(1+\frac{\sqrt{\mu\alpha}}{4}\right)\sqrt{V_{t+1}} + \frac{\mu}{2}\left(1+\frac{\sqrt{\mu\alpha}}{4}\right)\Delta_t^2.
    \end{aligned}
\end{equation}

\paragraph{Bounding $(C)$.}
For this part, we handle the distributional drift. By definition of $\rmP$ in \eqref{eq:gamma-rhp-P}, we have
% Recall that under this condition we consider one-dimensional case, that is, we let $d=1$ in \eqref{eq:obj-time-drift} such that $w\in\R$. Then we use $\rmP\in\R^{2\times2}$ to obtain
\begin{equation} \label{eq:bound-C}
    \begin{aligned}
        (C)
        &= \left(1+\frac{4}{\sqrt{\mu\alpha}}\right)
        \begin{bmatrix}
            w_t^* - w_{t+1}^* \\
            w_t^* - w_{t+1}^*
        \end{bmatrix}^{\top}
        \frac{1}{2\alpha}
        \begin{bmatrix}
        \rmI_{d} & (\sqrt{\mu\alpha}-1)\rmI_{d} \\
        (\sqrt{\mu\alpha}-1)\rmI_{d} & (1-\sqrt{\mu\alpha})^2\rmI_{d}
        \end{bmatrix}
        \begin{bmatrix}
            w_t^* - w_{t+1}^* \\
            w_t^* - w_{t+1}^*
        \end{bmatrix} \\
        &= \frac{\mu}{2}\left(1+\frac{4}{\sqrt{\mu\alpha}}\right)\Delta_t^2, 
    \end{aligned}
\end{equation}
where in the last equality we use the basic algebra of block matrix multiplication and the definition of $\Delta_t= \|w_t^*-w_{t+1}^*\|$.
% the fact that 
% \begin{equation}
% \| w_t^* - w_{t+1}^*\|^2=\Delta_t^2.
% \end{equation}

\paragraph{Final Bound for $V_{t+1}$.}
Now we are ready to derive the upper bound for $V_{t+1}$. Combining \eqref{eq:bound-A}, \eqref{eq:bound-B} and \eqref{eq:bound-C} together yields
\begin{equation} \label{eq:bound-V(t+1)-1}
    \begin{aligned}
        V_{t+1}
        &\leq (A) + (B) + (C) \\
        &\leq \left(1-\frac{3\sqrt{\mu\alpha}}{4}\right)V_t + 
        \left[- \sqrt{2\mu}\alpha(1-L\alpha)\langle u_{t,1}, \varepsilon_t \rangle\sqrt{\frac{\mu}{2}}\|w_t-w_t^*\| \right.\\
        &\left.\quad\quad\quad\quad\quad- \frac{2\mu\sqrt{2\alpha}}{1+\sqrt{\mu\alpha}}\alpha(1-L\alpha)\langle u_{t,2}, \varepsilon_t \rangle\frac{1+\sqrt{\mu\alpha}}{2\sqrt{2\alpha}}\|z_t-w_t\| + \frac{\alpha(1+L\alpha)}{2}\|\varepsilon_t\|^2\right] \\
        &\quad+ \sqrt{2\mu}\Delta_t\left(1+\frac{\sqrt{\mu\alpha}}{4}\right)\sqrt{V_{t+1}} + \mu\left(1+\frac{\sqrt{\mu\alpha}}{8}+\frac{2}{\sqrt{\mu\alpha}}\right)\Delta_t^2.
    \end{aligned}
\end{equation}
For simplicity, we define $D$ as the following 
\begin{equation*}
    \begin{aligned}
        D 
        &= \left(1-\frac{3\sqrt{\mu\alpha}}{4}\right)V_t + 
        \left[- \sqrt{2\mu}\alpha(1-L\alpha)\langle u_{t,1}, \varepsilon_t \rangle\sqrt{\frac{\mu}{2}}\|w_t-w_t^*\| \right.\\
        &\left.\quad- \frac{2\mu\sqrt{2\alpha}}{1+\sqrt{\mu\alpha}}\alpha(1-L\alpha)\langle u_{t,2}, \varepsilon_t \rangle\frac{1+\sqrt{\mu\alpha}}{2\sqrt{2\alpha}}\|z_t-w_t\| + \frac{\alpha(1+L\alpha)}{2}\|\varepsilon_t\|^2\right] + \mu\left(1+\frac{\sqrt{\mu\alpha}}{8}+\frac{2}{\sqrt{\mu\alpha}}\right)\Delta_t^2.
    \end{aligned}
\end{equation*}
Hence \eqref{eq:bound-V(t+1)-1} turns into
\begin{equation*}
    \begin{aligned}
        V_{t+1} - \sqrt{2\mu}\Delta_t\left(1+\frac{\sqrt{\mu\alpha}}{4}\right)\sqrt{V_{t+1}} - D \leq 0.
    \end{aligned}
\end{equation*}
Solving the above inequality we get
\begin{equation*}
    \begin{aligned}
        \sqrt{V_{t+1}}
        &\leq \frac{1}{2}\left[\sqrt{2\mu}\Delta_t\left(1+\frac{\sqrt{\mu\alpha}}{4}\right) + \sqrt{\left(\sqrt{2\mu}\Delta_t\left(1+\frac{\sqrt{\mu\alpha}}{4}\right)\right)^2 + 4D}\right] \\
        &\leq \sqrt{2\mu}\Delta_t\left(1+\frac{\sqrt{\mu\alpha}}{4}\right) + \sqrt{D} 
    \end{aligned}
\end{equation*}
Then an application of Young's inequality reveals
\begin{equation*}
    \begin{aligned}
        V_{t+1}
        &\leq \left(1+\frac{\sqrt{\mu\alpha}}{4}\right)D + \left(1+\frac{4}{\sqrt{\mu\alpha}}\right)\left(\sqrt{2\mu}\Delta_t\left(1+\frac{\sqrt{\mu\alpha}}{4}\right)\right)^2 \\
        &= \left(1+\frac{\sqrt{\mu\alpha}}{4}\right)\left\{
        \left(1-\frac{3\sqrt{\mu\alpha}}{4}\right)V_t + 
        \left[- \sqrt{2\mu}\alpha(1-L\alpha)\langle u_{t,1}, \varepsilon_t \rangle\sqrt{\frac{\mu}{2}}\|w_t-w_t^*\| \right.\right.\\
        &\left.\left.\quad- \frac{2\mu\sqrt{2\alpha}}{1+\sqrt{\mu\alpha}}\alpha(1-L\alpha)\langle u_{t,2}, \varepsilon_t \rangle\frac{1+\sqrt{\mu\alpha}}{2\sqrt{2\alpha}}\|z_t-w_t\| + \frac{\alpha(1+L\alpha)}{2}\|\varepsilon_t\|^2\right] + \mu\left(1+\frac{\sqrt{\mu\alpha}}{8}+\frac{2}{\sqrt{\mu\alpha}}\right)\Delta_t^2
        \right\} \\
        &\quad+ \left(1+\frac{4}{\sqrt{\mu\alpha}}\right)2\mu\Delta_t^2\left(1+\frac{\sqrt{\mu\alpha}}{4}\right)^2 \\
        &\leq \left(1-\frac{\sqrt{\mu\alpha}}{2}\right)V_t + \left(1+\frac{\sqrt{\mu\alpha}}{4}\right) 
        \left[- \sqrt{2\mu}\alpha(1-L\alpha)\langle u_{t,1}, \varepsilon_t \rangle\sqrt{\frac{\mu}{2}}\|w_t-w_t^*\| \right.\\
        &\left.\quad- \frac{2\mu\sqrt{2\alpha}}{1+\sqrt{\mu\alpha}}\alpha(1-L\alpha)\langle u_{t,2}, \varepsilon_t \rangle\frac{1+\sqrt{\mu\alpha}}{2\sqrt{2\alpha}}\|z_t-w_t\| + \frac{\alpha(1+L\alpha)}{2}\|\varepsilon_t\|^2\right] \\
        &\quad+ \mu\left(1+\frac{\sqrt{\mu\alpha}}{4}\right)\left(1+\frac{\sqrt{\mu\alpha}}{8}+\frac{2}{\sqrt{\mu\alpha}}\right)\Delta_t^2 + \left(1+\frac{4}{\sqrt{\mu\alpha}}\right)2\mu\Delta_t^2\left(1+\frac{\sqrt{\mu\alpha}}{4}\right)^2,
    \end{aligned}
\end{equation*}
where we plug in the definition of $D$ for the first equality. 
Since the learning rate $\alpha\leq 1/L$ and thus $\mu\alpha\leq 1$, then we have
\begin{equation*}
    \begin{aligned}
        \left(1+\frac{\sqrt{\mu\alpha}}{4}\right)\left(1+\frac{\sqrt{\mu\alpha}}{8}+\frac{2}{\sqrt{\mu\alpha}}\right) \leq \frac{125}{32\sqrt{\mu\alpha}},
        \quad\quad
        2\left(1+\frac{4}{\sqrt{\mu\alpha}}\right)\left(1+\frac{\sqrt{\mu\alpha}}{4}\right)^2 
        \leq \frac{125}{8\sqrt{\mu\alpha}}.
    \end{aligned}
\end{equation*}
Therefore, we finally conclude that
\begin{equation*}
    \begin{aligned}
        V_{t+1}
        &\leq \left(1-\frac{\sqrt{\mu\alpha}}{2}\right)V_t + \left(1+\frac{\sqrt{\mu\alpha}}{4}\right) 
        \left[- \sqrt{2\mu}\alpha(1-L\alpha)\langle u_{t,1}, \varepsilon_t \rangle\sqrt{\frac{\mu}{2}}\|w_t-w_t^*\| \right.\\
        &\left.\quad- \frac{2\mu\sqrt{2\alpha}}{1+\sqrt{\mu\alpha}}\alpha(1-L\alpha)\langle u_{t,2}, \varepsilon_t \rangle\frac{1+\sqrt{\mu\alpha}}{2\sqrt{2\alpha}}\|z_t-w_t\| + \frac{\alpha(1+L\alpha)}{2}\|\varepsilon_t\|^2\right] + \frac{20\mu\Delta_t^2}{\sqrt{\mu\alpha}},
    \end{aligned}
\end{equation*}
which is as claimed in \eqref{eq:one-step-single-level}.
\end{proof}

When there is no drift, the following lemma holds for any general strongly convex functions $\phi$ in  $\R^d$.

\begin{lemma}[Distance recursion, without drift] \label{lm:dist-without-drift}
Under the same settings as in \Cref{lm:distance-recursion-single-level} with $\phi_t(w)\equiv \phi(w)$, 
% $w\in\R^d$, 
and $w_t^*\equiv w^*$, where $\phi(w)$ can be any general strongly functions in $\R^d$. We redefine $u_{t,1}, u_{t,2}$ as 
\begin{equation} 
    u_{t,1} = \frac{\nabla\phi(w_t)-\nabla\phi(w^*)}{\|\nabla\phi(w_t)-\nabla\phi(w^*)\|},
    \quad\quad
    u_{t,2} = \frac{\nabla\phi(z_t)-\nabla\phi(w_t)}{\|\nabla\phi(z_t)-\nabla\phi(w_t)\|}.
\end{equation}
Then for all $t\geq0$, it holds that
\begin{equation} 
    \begin{aligned}
        V_{t+1}
        &\leq (1-\sqrt{\mu\alpha})V_t - \sqrt{\frac{2}{\mu}}L\alpha(1-L\alpha)\langle u_{t,1}, \varepsilon_t \rangle\frac{1}{L}\sqrt{\frac{\mu}{2}}\|\nabla\phi_t(w_t)-\nabla\phi_t(w^*)\| \\
        &- \frac{2L\sqrt{2\alpha}}{1+\sqrt{\mu\alpha}}\alpha(1-L\alpha)\langle u_{t,2}, \varepsilon_t \rangle\frac{1+\sqrt{\mu\alpha}}{2L\sqrt{2\alpha}}\|\nabla\phi_t(z_t)-\nabla\phi_t(w_t)\| + \frac{\alpha(1+L\alpha)}{2}\|\varepsilon_t\|^2.
    \end{aligned}
\end{equation}
\end{lemma}

\begin{proof}[Proof of \Cref{lm:dist-without-drift}]
By \eqref{eq:all-d} in \Cref{lm:distance-recursion-single-level} with $\phi_t(w)\equiv \phi(w)$ and $w_t^*\equiv w^*$, we have
\begin{equation*}
    \begin{aligned}
        V_{t+1} 
        &\leq \rho^2V_t - \alpha(1-L\alpha)\langle \nabla\phi_t(z_t), \varepsilon_t \rangle + \frac{L\alpha^2}{2}\|\varepsilon_t\|^2 \\
        &= \rho^2V_t - \alpha(1-L\alpha)\langle \nabla\phi_t(w_t)-\nabla\phi_t(w^*), \varepsilon_t \rangle - \alpha(1-L\alpha)\langle \nabla\phi_t(z_t)-\nabla\phi_t(w_t), \varepsilon_t \rangle + \frac{\alpha(1+L\alpha)}{2}\|\varepsilon_t\|^2 \\
        &= (1-\sqrt{\mu\alpha})V_t - \sqrt{\frac{2}{\mu}}L\alpha(1-L\alpha)\langle u_{t,1}, \varepsilon_t \rangle\frac{1}{L}\sqrt{\frac{\mu}{2}}\|\nabla\phi_t(w_t)-\nabla\phi_t(w^*)\| \\
        &\quad\quad- \frac{2L\sqrt{2\alpha}}{1+\sqrt{\mu\alpha}}\alpha(1-L\alpha)\langle u_{t,2}, \varepsilon_t \rangle\frac{1+\sqrt{\mu\alpha}}{2L\sqrt{2\alpha}}\|\nabla\phi_t(z_t)-\nabla\phi_t(w_t)\| + \frac{\alpha(1+L\alpha)}{2}\|\varepsilon_t\|^2.
    \end{aligned}
\end{equation*}
Hence the proof is completed.
\end{proof}

The following result shows the first part of \Cref{lm:option-1-single}. To the best of our knowledge, this is the first high probability guarantee with improved rate for SNAG under distributional drift.

\begin{lemma}[High-probability distance tracking, with drift] \label{lm:dist-track-single-level} 
Under the same setting as in \Cref{lm:distance-recursion-single-level} with 
% $\alpha\leq \min\{1/2L, \mu/2L^2, 1/10^3\mu\}$, 
$\alpha\leq 1/25L$, 
for any given $\delta\in(0,1)$ and all $t\in[T]$, the following holds with probability at least $1-\delta$ over the randomness in $\gH_t$:
\begin{equation}
    % V_t \leq \left(1-\frac{\sqrt{\mu\alpha}}{4}\right)^tV_0 + \left(2\alpha\sigma^2 + \frac{80\Delta^2}{\alpha}\right)\ln\frac{eT}{\delta}.
    V_t \leq \left(1-\frac{\sqrt{\mu\alpha}}{4}\right)^tV_0 + \left(\frac{5\sqrt{\alpha}\sigma^2}{\sqrt{\mu}} + \frac{80\Delta^2}{\alpha}\right)\ln\frac{eT}{\delta}.
\end{equation}
\end{lemma}

\begin{proof}[Proof of \Cref{lm:dist-track-single-level}]
We will invoke \Cref{lm:recursive-control} to show the results. 
To apply \Cref{lm:recursive-control}, we first need to show the following two facts:
\begin{equation} \label{eq:less-Vt}
    \begin{aligned}
        \text{Fact (I)}:\quad \sqrt{\frac{\mu}{2}}\|w_t-w_t^*\| \leq \sqrt{V_t}
        \quad\quad\&\quad\quad
        \text{Fact (II)}:\quad \frac{1+\sqrt{\mu\alpha}}{2\sqrt{2\alpha}}\|z_t-w_t\| \leq \sqrt{V_t}.
    \end{aligned}
\end{equation}

\paragraph{Fact (I) verification.}
Since $\phi_t$ is $\mu$-strongly convex and matrix $\rmP$ is PSD, then 
\begin{equation*}
    \begin{aligned}
        V_t
        &= \theta_t^{\top}\rmP\theta_t + \phi_t(w_t) - \phi_t(w_t^*)
        \geq \phi_t(w_t) - \phi_t(w_t^*) \\
        &\geq \frac{\mu}{2}\|w_t-w_t^*\|^2
        \quad\quad\quad\quad\Longrightarrow\quad\quad\quad\quad
        \sqrt{\frac{\mu}{2}}\|w_t-w_t^*\| \leq \sqrt{V_t}.
    \end{aligned}
\end{equation*}

\paragraph{Fact (II) verification.}
By definition of matrix $\rmP$ and simple calculation we have
\begin{equation*}
    \begin{aligned}
        \sqrt{V_t}
        &\geq \sqrt{\theta_t^{\top}\rmP\theta_t} 
        = \sqrt{\frac{1}{2\alpha}}\|(w_t-w_t^*) + (\sqrt{\mu\alpha}-1)(w_{t-1}-w_t^*)\| \\
        &= \sqrt{\frac{1}{2\alpha}}\|(1-\sqrt{\mu\alpha})(w_t-w_{t-1}) + \sqrt{\mu\alpha}(w_t-w_t^*)\| \\
        &\geq \sqrt{\frac{1}{2\alpha}}(1-\sqrt{\mu\alpha})\|w_t-w_{t-1}\| - \sqrt{\frac{\mu}{2}}\|w_t-w_t^*\| \\
        &\geq \sqrt{\frac{1}{2\alpha}}(1-\sqrt{\mu\alpha})\|w_t-w_{t-1}\| - \sqrt{V_t}.
    \end{aligned}
\end{equation*}
Rearrange the above inequality, and recall the update rule of stochastic Nesterov accelerated gradient method, we have
\begin{equation*}
    \begin{aligned}
        \|z_t-w_t\|
        &= \gamma\|w_t-w_{t-1}\|
        \leq \frac{2\gamma\sqrt{2\alpha}}{1-\sqrt{\mu\alpha}}\sqrt{V_t} 
        = \frac{2\sqrt{2\alpha}}{1+\sqrt{\mu\alpha}}\sqrt{V_t},
    \end{aligned}
\end{equation*}
where for the last equality we use the definition of $\gamma$ as in \eqref{eq:gamma-rhp-P}. Rearrange it gives \eqref{eq:less-Vt}.

By \Cref{lm:distance-recursion-single-level} and the choice of $\alpha\leq 1/L$, we have
\begin{equation} \label{eq:apply-control}
    \begin{aligned}
        V_{t+1}
        &\leq \left(1-\frac{\sqrt{\mu\alpha}}{2}\right)V_t + \left(1+\frac{\sqrt{\mu\alpha}}{4}\right) 
        \left[\sqrt{2\mu}\alpha(1-L\alpha)\langle u_{t,1}, -\varepsilon_t \rangle\sqrt{\frac{\mu}{2}}\|w_t-w_t^*\| \right.\\
        &\left.\quad\quad+ \frac{2\mu\sqrt{2\alpha}}{1+\sqrt{\mu\alpha}}\alpha(1-L\alpha)\langle u_{t,2}, -\varepsilon_t \rangle\frac{1+\sqrt{\mu\alpha}}{2\sqrt{2\alpha}}\|z_t-w_t\| + \alpha\|\varepsilon_t\|^2\right] + \frac{20\mu\Delta_t^2}{\sqrt{\mu\alpha}}.
    \end{aligned}
\end{equation}
Note that under \Cref{ass:noise-drift-single-level}, there exists an absolute constant $c\geq1$ such that for all $t\geq0$, $\|\varepsilon_t\|^2$ is sub-exponential conditioned on $\gH_t$ with parameter $c\sigma^2$, and $\varepsilon_t$ is mean-zero sub-Gaussian conditioned on $\gH_t$ with parameter $c\sigma$ \cite[Theorem 30]{cutler2023stochastic}. For convenience we simply let $c=1$ here. Thus $\langle u_{t,1}, -\varepsilon_t \rangle$ is mean-zero sub-Gaussian conditioned on $\gH_t$ with parameter $\sigma$, and $\Delta_t^2$ is sub-exponential conditioned on $\gH_t$ with parameter $\Delta^2$ by assumption. Hence, in light of \eqref{eq:apply-control}, we apply \Cref{lm:recursive-control} with
\begin{equation*}
    \gH_t = \gH_t,
    \quad
    V_t = V_t,
    \quad
    V'_{t,1} = \frac{\mu}{2}\|w_t-w_t^*\|^2,
    \quad
    V'_{t,2} = \frac{(1+\sqrt{\mu\alpha})^2}{8\alpha}\|z_t-w_t\|^2,
\end{equation*}
\begin{equation*}
    D_{t,1} = \left(1+\frac{\sqrt{\mu\alpha}}{4}\right)\sqrt{2\mu}\alpha(1-L\alpha)\langle u_{t,1}, -\varepsilon_t \rangle,
    \quad
    D_{t,2} = \left(1+\frac{\sqrt{\mu\alpha}}{4}\right)\frac{2\mu\sqrt{2\alpha}}{1+\sqrt{\mu\alpha}}\alpha(1-L\alpha)\langle u_{t,2}, -\varepsilon_t \rangle,
\end{equation*}
\begin{equation*}
    % X_t = \left(1+\frac{\sqrt{\mu\alpha}}{4}\right)\frac{\alpha(1+L\alpha)}{2}\|\varepsilon_t\|^2 + \frac{20\mu\Delta_t^2}{\sqrt{\mu\alpha}}, 
    X_t = \left(1+\frac{\sqrt{\mu\alpha}}{4}\right)\alpha\|\varepsilon_t\|^2 + \frac{20\mu\Delta_t^2}{\sqrt{\mu\alpha}}, 
    \quad
    \alpha_t = 1-\frac{\sqrt{\mu\alpha}}{2},
    \quad
    \kappa_t = 0,
\end{equation*}
\begin{equation*}
    \sigma_1 = \left(1+\frac{\sqrt{\mu\alpha}}{4}\right)\sqrt{2\mu}\alpha(1-L\alpha)\sigma,
    \quad
    \sigma_2 = \left(1+\frac{\sqrt{\mu\alpha}}{4}\right)\frac{2\mu\sqrt{2\alpha}}{1+\sqrt{\mu\alpha}}\alpha(1-L\alpha)\sigma,
\end{equation*}
\begin{equation*}
    % \nu_t = \left(1+\frac{\sqrt{\mu\alpha}}{4}\right)\frac{L\alpha^2}{2}\sigma^2 + \frac{20\mu\Delta^2}{\sqrt{\mu\alpha}},
    \nu_t = \left(1+\frac{\sqrt{\mu\alpha}}{4}\right)\alpha\sigma^2 + \frac{20\mu\Delta^2}{\sqrt{\mu\alpha}},
\end{equation*}
yielding the following recursion 
\begin{equation} \label{eq:recursion-dist}
    \begin{aligned}
        \E[\exp(\lambda V_{t+1})]
        &\leq \exp\left(\lambda\left[\left(1+\frac{\sqrt{\mu\alpha}}{4}\right)\alpha\sigma^2 + \frac{20\mu\Delta^2}{\sqrt{\mu\alpha}}\right]\right) \E\left[\exp\left(\lambda\left(1-\frac{\sqrt{\mu\alpha}}{4}\right)V_t\right)\right]
    \end{aligned}
\end{equation}
for all $\lambda$ satisfying
\begin{equation*}
    0 \leq \lambda \leq
    \min\left\{\frac{2}{125\alpha\sqrt{\mu\alpha}\sigma^2}, \frac{1}{5\alpha\sigma^2/2 + 40\mu\Delta^2/\sqrt{\mu\alpha}}\right\}.
\end{equation*}
We then apply \eqref{eq:recursion-dist} recursively to deduce
\begin{equation*}
    \begin{aligned}
        \E[\exp(\lambda V_{t+1})]
        &\leq \exp\left[\lambda\left(1-\frac{\sqrt{\mu\alpha}}{4}\right)^tV_0 + \lambda\left(\left(1+\frac{\sqrt{\mu\alpha}}{4}\right)\alpha\sigma^2 + \frac{20\mu\Delta^2}{\sqrt{\mu\alpha}}\right)\sum_{i=0}^{t-1}\left(1-\frac{\sqrt{\mu\alpha}}{4}\right)^{i}\right] \\
        &\leq \exp\left\{\lambda\left[\left(1-\frac{\sqrt{\mu\alpha}}{4}\right)^tV_0 + 4\left(1+\frac{\sqrt{\mu\alpha}}{4}\right)\frac{\alpha}{\sqrt{\mu\alpha}}\sigma^2 + \frac{80\Delta^2}{\alpha}\right]\right\}
    \end{aligned}
\end{equation*}
for all $\lambda$ satisfying
\begin{equation*}
    0 \leq \lambda \leq
    \min\left\{\frac{2}{125\alpha\sqrt{\mu\alpha}\sigma^2}, \frac{1}{5\alpha\sigma^2/2 + 40\mu\Delta^2/\sqrt{\mu\alpha}}\right\}.
\end{equation*}
Moreover, setting
\begin{equation*}
    % \nu \coloneqq 2\alpha\sigma^2 + \frac{80\Delta^2}{\alpha}
    \nu \coloneqq \frac{5\sqrt{\alpha}\sigma^2}{\sqrt{\mu}} + \frac{80\Delta^2}{\alpha}
\end{equation*}
and taking into account 
% $\alpha\leq \min\{1/2L, \mu/2L^2, 1/10^3\mu\}$, 
$\alpha\leq 1/25L$, 
then we have
\begin{equation*}
    % 2\left(1+\frac{\sqrt{\mu\alpha}}{4}\right)\frac{L\alpha^2}{\sqrt{\mu\alpha}}\sigma^2 + \frac{80\Delta^2}{\alpha} \leq \nu
    4\left(1+\frac{\sqrt{\mu\alpha}}{4}\right)\frac{\alpha}{\sqrt{\mu\alpha}}\sigma^2 + \frac{80\Delta^2}{\alpha} \leq \nu
\end{equation*}
and
\begin{equation*}
    % \frac{1}{\nu} = \frac{1}{2\alpha\sigma^2 + 80\Delta^2/\alpha} 
    % \leq \min\left\{\frac{2}{125\alpha\sqrt{\mu\alpha}\sigma^2}, \frac{1}{5\alpha\sigma^2/4 + 40\mu\Delta^2/\sqrt{\mu\alpha}}\right\}.
    \frac{1}{\nu} = \frac{1}{5\sqrt{\alpha}\sigma^2/\sqrt{\mu} + 80\Delta^2/\alpha} 
    \leq \min\left\{\frac{2}{125\alpha\sqrt{\mu\alpha}\sigma^2}, \frac{1}{5\alpha\sigma^2/2 + 40\mu\Delta^2/\sqrt{\mu\alpha}}\right\}.
\end{equation*}
Thus we obtain
\begin{equation*}
    \begin{aligned}
        \E\left[\exp\left(\lambda\left(V_t - \left(1-\frac{\sqrt{\mu\alpha}}{4}\right)^tV_0\right)\right)\right] \leq \exp(\lambda\nu)
        \quad\text{for all}\quad 0\leq \lambda\leq 1/\nu.
    \end{aligned}
\end{equation*}
Taking $\lambda=1/\nu$ and applying Markov's inequality and union bound completes the proof.
\end{proof}

The following result shows the second part of \Cref{lm:option-1-single}.

\begin{lemma}[High-probability distance tracking, without drift] \label{lm:track-without-drift}
Under the same setting as in \Cref{lm:dist-without-drift} with $\alpha\leq 1/25L$, for any given $\delta\in(0,1)$ and all $t\in[T]$, the following holds with probability at least $1-\delta$ over the randomness in $\gH_t$:
\begin{equation}
    % V_t \leq \left(1-\frac{\sqrt{\mu\alpha}}{4}\right)^tV_0 + 2\alpha\sigma^2 \ln\frac{eT}{\delta}.
    V_t \leq \left(1-\frac{\sqrt{\mu\alpha}}{4}\right)^tV_0 + \frac{5\sqrt{\alpha}\sigma^2}{\sqrt{\mu}} \ln\frac{eT}{\delta}.
\end{equation}
\end{lemma}

\begin{proof}[Proof of \Cref{lm:track-without-drift}]
First it is easy to verify that 
\begin{equation*}
    \frac{1}{L}\sqrt{\frac{\mu}{2}}\|\nabla\phi_t(w_t)-\nabla\phi_t(w^*)\| \leq \sqrt{V_t}
    \quad\quad\text{and}\quad\quad
    \frac{1+\sqrt{\mu\alpha}}{2L\sqrt{2\alpha}}\|\nabla\phi_t(z_t)-\nabla\phi_t(w_t)\| \leq \sqrt{V_t}.
\end{equation*}
Then we apply \Cref{lm:recursive-control} to obtain the final result. We omit the detailed proof here since it follows the same procedure as in proof of \Cref{lm:dist-track-single-level}.
\end{proof}

\section{Proofs of Results in \Cref{sec:nagbilevel}}
\label{sec:proof-sketch}
We first present the following algebraic fact under suitable choice of parameters.

\begin{lemma}[Parameter choice, informal] \label{lm:parameter-choice}
For any given $\delta\in(0,1)$ and any small $\epsilon$ satisfying 
\begin{equation} \label{eq:para-eps}
    \begin{aligned}
        \epsilon \leq \left(\frac{170\cdot32ed_0L_0^2\tilde{\sigma}_{g,1}^2}{\delta\mu^2}\max\left\{\frac{l_{g,1}}{\tilde{\sigma}_{g,1}}, \frac{\Bar{\sigma}}{d_0}\right\}\right)^{1/3},
    \end{aligned}
\end{equation}
if we set parameters $\alpha,\beta,\eta, T$ as 
\begin{equation} \label{eq:para}
    1-\beta = \frac{\mu^2\epsilon^2}{170\cdot64L_0^2\tilde{\sigma}_{g,1}^2\ln(P)},
    \quad
    \eta = \min\left\{\frac{\tilde{\sigma}_{g,1}}{l_{g,1}}, \frac{d_0}{\Bar{\sigma}}\right\}(1-\beta), 
    \quad
    \alpha = \frac{1}{\mu}(1-\beta),
    \quad
    \sigma_{g,1} = \sqrt{\mu\alpha}\tilde{\sigma}_{g,1},
    \quad
    T = \frac{4d_0}{\eta\epsilon},
\end{equation}
where $\Bar{\sigma}$ is defined in \Cref{lm:hyper-var}, and $d_0$ and $P$ are defined as
\begin{equation} \label{eq:d0-P}
    d_0=\Phi(x_0)-\inf_{x\in\R^{d_x}}\Phi(x),
    \quad
    P = \left(\frac{170\cdot64ed_0L_0^2\tilde{\sigma}_{g,1}^2}{\delta\mu^2\epsilon^3}\max\left\{\frac{l_{g,1}}{\tilde{\sigma}_{g,1}}, \frac{\Bar{\sigma}}{d_0}\right\}\right)^2.
\end{equation}
Then the following holds for all $t\in[T]$:
% \begin{equation*}
%     \left(\frac{4\alpha\tilde{\sigma}_{g,1}^2}{\mu} + \frac{160\eta^2l_{g,1}^2}{\mu^3\alpha}\right)\ln\frac{eT}{\delta} \leq \frac{\epsilon^2}{16L_0^2}.
% \end{equation*}
\begin{equation*}
    \left(\frac{4\alpha\tilde{\sigma}_{g,1}^2}{\mu} + \frac{160\eta^2l_{g,1}^2}{\mu^3\alpha}\right)\ln\frac{eT}{\delta} \leq \frac{\epsilon^2}{64L_0^2}.
\end{equation*}
\end{lemma}

\begin{proof}[Proof of \Cref{lm:parameter-choice}]
By \Cref{lm:lip-y*z*}, we have $\Delta_t = \|y_t^*-y_{t+1}^*\| \leq \frac{l_{g,1}}{\mu}\|x_t-x_{t+1}\| = \eta l_{g,1}/\mu$. Thus in our bilevel setting, we choose $\Delta = \eta l_{g,1}/\mu$, where $\Delta$ is defined in \Cref{sec:nesterov}.
By choice of $\alpha,\eta,T$ as in \eqref{eq:para}, we have
\begin{equation*}
    \begin{aligned}
        \left(\frac{10\alpha\tilde{\sigma}_{g,1}^2}{\mu} + \frac{160\eta^2l_{g,1}^2}{\mu^3\alpha}\right)\ln\frac{eT}{\delta} 
        &= \left(\frac{10(1-\beta)\tilde{\sigma}_{g,1}^2}{\mu^2} + \frac{160\eta^2l_{g,1}^2}{\mu^2(1-\beta)}\right)\ln\left(\frac{4ed_0}{\delta\eta\epsilon}\right) \\
        &\leq \frac{170(1-\beta)\tilde{\sigma}_{g,1}^2}{\mu^2}\ln\left(\frac{4ed_0}{\delta\epsilon(1-\beta)}\max\left\{\frac{l_{g,1}}{\tilde{\sigma}_{g,1}}, \frac{\Bar{\sigma}}{4d_0}\right\}\right).
    \end{aligned}
\end{equation*}
Now we choose $\beta$ to be 
\begin{equation*}
    \begin{aligned}
        1-\beta = \frac{\mu^2\epsilon^2}{170\cdot64L_0^2\tilde{\sigma}_{g,1}^2\ln(P)}, 
        \quad\text{where}\quad
        P = \left(\frac{170\cdot64ed_0L_0^2\tilde{\sigma}_{g,1}^2}{\delta\mu^2\epsilon^3}\max\left\{\frac{l_{g,1}}{\tilde{\sigma}_{g,1}}, \frac{\Bar{\sigma}}{d_0}\right\}\right)^2.
    \end{aligned}
\end{equation*}
Then we have
% \begin{equation*}
%     \begin{aligned}
%         \frac{170(1-\beta)\tilde{\sigma}_{g,1}^2}{\mu^2}\ln\left(\frac{4ed_0l_{g,1}}{\delta\epsilon\tilde{\sigma}_{g,1}(1-\beta)}\right)
%         = \frac{\epsilon^2}{16L_0^2\ln(P)}\ln\left(\sqrt{P}\ln(P)\right)
%         \leq \frac{\epsilon^2}{16L_0^2},
%     \end{aligned}
% \end{equation*}
\begin{equation*}
    \begin{aligned}
        \frac{170(1-\beta)\tilde{\sigma}_{g,1}^2}{\mu^2}\ln\left(\frac{4ed_0l_{g,1}}{\delta\epsilon\tilde{\sigma}_{g,1}(1-\beta)}\right)
        = \frac{\epsilon^2}{64L_0^2\ln(P)}\ln\left(\sqrt{P}\ln(P)\right)
        \leq \frac{\epsilon^2}{64L_0^2},
    \end{aligned}
\end{equation*}
where we use the fact that $\ln(\sqrt{P}\ln(P)) \leq \ln(P) \leq \ln^2(P)$ for any $P\geq 4$ by choice of $\epsilon$ as in \eqref{eq:para-eps}.
\end{proof}

In the rest of this section, we assume \cref{ass:relax-smooth,ass:f-and-g,ass:noise,ass:individual-noise} hold. In addition, the failure probability $\delta\in(0,1)$ and $\epsilon>0$ are chosen in the same way as in \Cref{thm:main}. 

\subsection{Proof of \Cref{lem:warmstart}}

% \begin{lemma}[Warm-start, Restatement of \Cref{lem:warmstart}] \label{lm:warm-start}
% Let $\{y_t^{\init}\}$ be the iterates produced by line 2 of \Cref{alg:bilevel}. Set $\alpha^{\init}=\mu\alpha/(\mu+l_{g,1})=\widetilde{\Theta}(\epsilon^2)$ with $\alpha$ defined in \eqref{eq:para} and $\phi_t(y)\equiv g(x_0,y)$. Then $\|y_{T_0}^{\init}-y_0^*\| \leq \sqrt{\frac{\mu}{8(\mu+l_{g,1})}}\frac{\epsilon}{L_0}$ holds with probability at least $1-\delta$ over the randomness in $\widetilde{\mathcal{F}}^{\init}$ (we denote this event as $\gE_{\init}$) in $T_0=\widetilde{O}(\epsilon^{-1})$ iterations, where
% \begin{equation}
%     T_0 = \ln\left(\frac{\mu^2\epsilon^2}{8(\mu+l_{g,1})^2L_0^2\|y_0^{\init}-y_0^*\|^2}\right) 
%     \Big/
%     \ln\left(1-\frac{\sqrt{\mu\alpha}}{4}\right)
%     = \widetilde{O}(\epsilon^{-1}).
% \end{equation}
% \end{lemma}
\begin{lemma}[Warm-start, Restatement of \Cref{lem:warmstart}] \label{lm:warm-start}
Let $\{y_t^{\init}\}$ be the iterates produced by line 2 of \Cref{alg:bilevel}. Set $\alpha^{\init}=\mu\alpha^2=\widetilde{\Theta}(\epsilon^4)$ with $\alpha$ defined in \eqref{eq:para}, $\sigma_{g,1}=(\mu\alpha)^{1/4}\tilde{\sigma}_{g,1}$, and $\phi_t(y)\equiv g(x_0,y)$. Then $\|y_{T_0}^{\init}-y_0^*\| \leq \sqrt{\frac{\mu\alpha}{32}}\frac{\epsilon}{L_0}$ holds with probability at least $1-\delta$ over the randomness in $\widetilde{\mathcal{F}}^{\init}$ (we denote this event as $\gE_{\init}$) in $T_0=\widetilde{O}(\epsilon^{-2})$ iterations, where
\begin{equation}
    T_0 = \ln\left(\frac{\mu^3\alpha^3\epsilon^2}{256L_0^2\|y_0^{\init}-y_0^*\|^2}\right) 
    \Big/
    \ln\left(1-\frac{\mu\alpha}{4}\right)
    = \widetilde{O}(\epsilon^{-2}).
\end{equation}
\end{lemma}

\begin{proof}[Proof of \Cref{lm:warm-start}]
By \Cref{lm:track-without-drift,lm:parameter-choice} and $\mu$-strong convexity of $g$ in $y$, we have with probability at least $1-\delta$ over the randomness in $\gF^{\init}$ that
% \begin{equation*}
%     \begin{aligned}
%         \|y_{T_0}^{\init}-y_0^*\|^2 
%         &\leq \frac{2}{\mu}\left(1-\frac{\sqrt{\mu\alpha}}{4}\right)^{T_0}U_0^{\init} + \frac{\mu}{\mu+l_{g,1}}\frac{4\alpha\tilde{\sigma}_{g,1}^2}{\mu}\ln\frac{eT_0}{\delta} \\
%         &\leq \frac{2}{\mu}\left(1-\frac{\sqrt{\mu\alpha}}{4}\right)^{T_0}U_0^{\init} + \frac{\mu}{\mu+l_{g,1}}\frac{\epsilon^2}{16L_0^2}.
%     \end{aligned}
% \end{equation*}
\begin{equation*}
    \begin{aligned}
        \|y_{T_0}^{\init}-y_0^*\|^2 
        &\leq \frac{2}{\mu}\left(1-\frac{\sqrt{\mu^2\alpha^2}}{4}\right)^{T_0}U_0^{\init} + \frac{10\mu\alpha^2\tilde{\sigma}_{g,1}^2}{\mu}\ln\frac{eT_0}{\delta} \\
        &\leq \frac{2}{\mu}\left(1-\frac{\mu\alpha}{4}\right)^{T_0}U_0^{\init} + \frac{\mu\alpha\epsilon^2}{64L_0^2},
    \end{aligned}
\end{equation*}
where the first inequality uses the choice of $\alpha^{\init}=\mu\alpha^2$. 
By $l_{g,1}$-smoothness of $g$ we have
% \begin{equation*}
%     \begin{aligned}
%         U_0^{\init} 
%         &= \frac{\mu}{2}\|y_0^{\init}-y_0^*\|^2 + g(x_0,y_0^{\init}) - g(x_0,y_0^*) \leq \frac{\mu+l_{g,1}}{2}\|y_0^{\init}-y_0^*\|^2.
%     \end{aligned}
% \end{equation*}
\begin{equation*}
    \begin{aligned}
        U_0^{\init} 
        &\leq \frac{2-\sqrt{\mu^2\alpha^2}+\mu^2\alpha^2}{2\mu\alpha^2}\|y_0^{\init}-y_0^*\|^2 + g(x_0,y_0^{\init}) - g(x_0,y_0^*) \\
        &\leq \frac{3}{2\mu\alpha^2}\|y_0^{\init}-y_0^*\|^2 + \frac{l_{g,1}}{2}\|y_0^{\init}-y_0^*\|^2
        \leq \frac{2}{\mu\alpha^2}\|y_0^{\init}-y_0^*\|^2,
    \end{aligned}
\end{equation*}
where the last inequality uses $\alpha\leq 1/l_{g,1}$.
Now we set
% \begin{equation*}
%     \frac{2}{\mu}\left(1-\frac{\sqrt{\mu\alpha}}{4}\right)^{T_0}\frac{\mu+l_{g,1}}{2}\|y_0^{\init}-y_0^*\|^2 + \frac{\mu}{\mu+l_{g,1}}\frac{\epsilon^2}{16L_0^2}
%     \leq \frac{\mu}{\mu+l_{g,1}}\frac{\epsilon^2}{8L_0^2},
% \end{equation*}
\begin{equation*}
    \frac{2}{\mu}\left(1-\frac{\mu\alpha}{4}\right)^{T_0}\frac{2}{\mu\alpha^2}\|y_0^{\init}-y_0^*\|^2 + \frac{\mu\alpha\epsilon^2}{64L_0^2}
    \leq \frac{\mu\alpha\epsilon^2}{32L_0^2},
\end{equation*}
which gives
% \begin{equation*}
%     T_0 \geq \ln\left(\frac{\mu^2\epsilon^2}{8(\mu+l_{g,1})^2L_0^2\|y_0^{\init}-y_0^*\|^2}\right) 
%     \Big/
%     \ln\left(1-\frac{\sqrt{\mu\alpha}}{4}\right)
% \end{equation*}
\begin{equation*}
    T_0 \geq \ln\left(\frac{\mu^3\alpha^3\epsilon^2}{256L_0^2\|y_0^{\init}-y_0^*\|^2}\right) 
    \Big/
    \ln\left(1-\frac{\mu\alpha}{4}\right).
\end{equation*}
By choice of $\alpha$ as in \eqref{eq:para} and simple calculation we obtain $T_0 = \widetilde{O}(\epsilon^{-2})$ when $\epsilon$ is small. 
\end{proof}

\subsection{Proof of \Cref{lem:optionI}}

\begin{lemma}[Option I, Restatement of \Cref{lem:optionI}] \label{lm:anytime-yt}
% Under \cref{ass:relax-smooth,ass:f-and-g,ass:noise,ass:individual-noise} and event $\gE_{\init}$, 
Under event $\gE_{\init}$, let $\{y_t\}$ be the iterates produced by Option I. Set $\alpha=\widetilde{\Theta}(\epsilon^2)$ as in \eqref{eq:para}, $\sigma_{g,1}=(\mu\alpha)^{1/4}\tilde{\sigma}_{g,1}$, and $\phi_t(y) = g(x_t,y) = \frac{\mu}{2}\|y-y_t^*\|^2$. Then for any $t\in[T]$, \Cref{alg:bilevel} guarantees with probability at least $1-\delta$ over the randomness in $\widetilde{\mathcal{F}}_T^1$ (we denote this event as $\gE_y^1$) that $\|y_t-y_t^*\| \leq \epsilon/2L_0$.
\end{lemma}

\begin{proof}[Proof of \Cref{lm:anytime-yt}]
By \Cref{lm:dist-track-single-level,lm:warm-start} we have
% \begin{equation*}
%     \begin{aligned}
%         \|y_t-y_t^*\|^2 
%         &\leq \frac{2}{\mu}\left(1-\frac{\sqrt{\mu\alpha}}{4}\right)^tU_0 + \left(\frac{4\alpha\tilde{\sigma}_{g,1}^2}{\mu} + \frac{160\eta^2l_{g,1}^2}{\mu^3\alpha}\right)\ln\frac{eT}{\delta} \\
%         &= \frac{2}{\mu}\left(1-\frac{\sqrt{\mu\alpha}}{4}\right)^t\left(\frac{\mu}{2}\|y_{T_0}^{\init}-y_0^*\|^2 + g(x_0,y_{T_0}^{\init}) - g(x_0,y_0^*)\right) + \left(\frac{4\alpha\tilde{\sigma}_{g,1}^2}{\mu} + \frac{160\eta^2l_{g,1}^2}{\mu^3\alpha}\right)\ln\frac{eT}{\delta} \\
%         &\leq \frac{\mu+l_{g,1}}{\mu}\|y_{T_0}^{\init}-y_0^*\|^2 + \frac{\epsilon^2}{16L_0^2} 
%         \leq \frac{3\epsilon^2}{16L_0^2}.
%     \end{aligned}
% \end{equation*}
\begin{equation*}
    \begin{aligned}
        \|y_t-y_t^*\|^2 
        &\leq \frac{2}{\mu}\left(1-\frac{\sqrt{\mu\alpha}}{4}\right)^tU_0 + \left(\frac{10\alpha\tilde{\sigma}_{g,1}^2}{\mu} + \frac{160\eta^2l_{g,1}^2}{\mu^3\alpha}\right)\ln\frac{eT}{\delta} \\
        &\leq \frac{2}{\mu}\left(1-\frac{\sqrt{\mu\alpha}}{4}\right)^t\frac{2}{\alpha}\|y_{T_0}^{\init}-y_0^*\|^2 + \left(\frac{10\alpha\tilde{\sigma}_{g,1}^2}{\mu} + \frac{160\eta^2l_{g,1}^2}{\mu^3\alpha}\right)\ln\frac{eT}{\delta} \\
        &\leq \frac{4}{\mu\alpha}\|y_{T_0}^{\init}-y_0^*\|^2 + \frac{\epsilon^2}{64L_0^2} 
        \leq \frac{\epsilon^2}{4L_0^2}.
    \end{aligned}
\end{equation*}
Thus we conclude that for all $t\in[T]$, we have $\|y_t-y_t^*\| \leq \epsilon/2L_0$.
\end{proof}

\subsection{Proof of \Cref{lem:optionII}}

% \begin{lemma}[Option II, Restatement of \Cref{lem:optionII}] \label{lm:periodic} 
% Under event $\gE_{\init}$, let $\{y_t\}$ be the iterates produced by Option II. set $\alpha'=\mu\alpha/(\mu+l_{g,1})=\widetilde{\Theta}(\epsilon^2)$ with $\alpha$ defined in \eqref{eq:para} and run SNAG in each update round for
% \begin{equation}
%     N = \ln\left(\frac{1}{8}\right) \Big/ \ln\left(1-\frac{\sqrt{\mu\alpha}}{4}\right) = \widetilde{O}(\epsilon^{-1})
% \end{equation}
% steps in every $I=\frac{\mu\epsilon}{2(1-\beta)L_0\sigma_{g,1}}=\widetilde{O}(\epsilon^{-1})$ iterations, set $\phi_t(y) = g(x_t,y)$ when $t$ is a multiple of $I$ (i.e., $x_t$ is fixed for each update round of Option II so $g$ can be general functions). Then for any $t\in[T]$, \Cref{alg:bilevel} guarantees with probability at least $1-\delta$ over the randomness in $\sigma(\cup_{t\leq T}\widetilde{\gF}_t^2)$ (we denote this event as $\gE_y^2$) that $\|y_t-y_t^*\| \leq \epsilon/L_0$. 
% \end{lemma}
\begin{lemma}[Option II, Restatement of \Cref{lem:optionII}] \label{lm:periodic} 
Under event $\gE_{\init}$, let $\{y_t\}$ be the iterates produced by Option II. Set $\alpha=\widetilde{\Theta}(\epsilon^2)$ as in \eqref{eq:para}, $\sigma_{g,1}=(\mu\alpha)^{1/4}\tilde{\sigma}_{g,1}$, and run SNAG in each update round for
\begin{equation*}
    N = \ln\left(\frac{\mu\alpha}{128}\right) \Big/ \ln\left(1-\frac{\sqrt{\mu\alpha}}{4}\right) = \widetilde{O}(\epsilon^{-1})
\end{equation*}
steps in every $I=\frac{\mu\epsilon}{2(1-\beta)L_0\tilde{\sigma}_{g,1}}=\widetilde{O}(\epsilon^{-1})$ iterations, set $\phi_t(y) = g(x_t,y)$ when $t$ is a multiple of $I$ (i.e., $x_t$ is fixed for each update round of Option II so $g$ can be general functions). Then for any $t\in[T]$, \Cref{alg:bilevel} guarantees with probability at least $1-\delta$ over the randomness in $\sigma(\cup_{t\leq T}\widetilde{\gF}_t^2)$ (we denote this event as $\gE_y^2$) that $\|y_t-y_t^*\| \leq \epsilon/L_0$. 
\end{lemma}

\begin{proof}[Proof of \Cref{lm:periodic}]
At the beginning of the first round, by \Cref{lm:warm-start,lm:anytime-yt} we have $\|y_0-y_0^*\|\leq \epsilon/2L_0$, then we do not update the lower-level variable until $t=I$-th iteration, then for $t=I$, we have
\begin{equation*}
    \begin{aligned}
        \|y_I-y_I^{*}\|
        &= \|y_0-y_I^*\|
        \leq \|y_0-y_0^*\| + \sum_{i=1}^{I}\|y_{i}^*-y_{i-1}^*\| \\
        &\leq \frac{\epsilon}{2L_0} + \frac{\eta l_{g,1}}{\mu}I
        = \frac{\epsilon}{L_0},
    \end{aligned}
\end{equation*}
where in the last equality we plug in the definition of $\eta$ and $I$. 
By $l_{g,1}$-smoothness of $g$ we have
% \begin{equation*}
%     \begin{aligned}
%         U_I
%         &= \frac{\mu}{2}\|y_I-y_0^*\|^2 + g(x_I,y_I) - g(x_I,y_I^*) \leq \frac{\mu+l_{g,1}}{2}\|y_I-y_I^*\|^2 \leq \frac{\mu+l_{g,1}}{2}\frac{\epsilon^2}{L_0^2}
%     \end{aligned}
% \end{equation*}
\begin{equation*}
    \begin{aligned}
        U_I
        &\leq \frac{2-2\sqrt{\mu\alpha}+\mu\alpha}{2\alpha}\|y_I-y_0^*\|^2 + g(x_I,y_I) - g(x_I,y_I^*) \\
        &\leq \frac{3}{2\alpha}\|y_I-y_I^*\|^2 + \frac{l_{g,1}}{2}\|y_I-y_I^*\|^2
        \leq \frac{2\epsilon^2}{\alpha L_0^2}
    \end{aligned}
\end{equation*}
Then for $N$ steps update in the inner loops of $t=I$-th iteration, we set
% \begin{equation*}
%     \frac{2}{\mu}\left(1-\frac{\sqrt{\mu\alpha}}{4}\right)^{N}\frac{\mu+l_{g,1}}{2}\frac{\epsilon^2}{L_0^2} + \frac{\mu}{\mu+l_{g,1}}\frac{\epsilon^2}{16L_0^2}
%     \leq \frac{\mu}{\mu+l_{g,1}}\frac{\epsilon^2}{8L_0^2},
% \end{equation*}
\begin{equation*}
    \frac{2}{\mu}\left(1-\frac{\sqrt{\mu\alpha}}{4}\right)^{N}\frac{2\epsilon^2}{\alpha L_0^2} + \frac{\epsilon^2}{64L_0^2}
    \leq \frac{\epsilon^2}{16L_0^2},
\end{equation*}
which gives
% \begin{equation*}
%     N \geq \ln\left(\frac{1}{8}\right) \Big/ \ln\left(1-\frac{\sqrt{\mu\alpha}}{4}\right)
% \end{equation*}
\begin{equation*}
    N \geq \ln\left(\frac{\mu\alpha}{128}\right) \Big/ \ln\left(1-\frac{\sqrt{\mu\alpha}}{4}\right)
\end{equation*}
By choice of $\alpha$ as in \eqref{eq:para} and simple calculation we obtain $N = \widetilde{O}(\epsilon^{-1})$ when $\epsilon$ is small. Now we have
% \begin{equation*}
%     \|y_{I+1}-y_I^*\|^2 = \|y_I^{N}-y_I^*\|^2 \leq \frac{\mu+l_{g,1}}{\mu}\|y_I-y_I^*\|^2 + \frac{\epsilon^2}{16L_0^2} \leq \frac{3\epsilon^2}{16L_0^2},
% \end{equation*}
\begin{equation*}
    \|y_{I+1}-y_I^*\|^2 = \|y_I^{N}-y_I^*\|^2 \leq \frac{2}{\mu}\left(1-\frac{\sqrt{\mu\alpha}}{4}\right)^{N}\frac{2\epsilon^2}{\alpha L_0^2} + \frac{\epsilon^2}{64L_0^2} \leq \frac{\epsilon^2}{16L_0^2},
\end{equation*}
which yields 
\begin{equation*}
    \|y_{I+1}-y_{I+1}^*\| 
    \leq \|y_{I+1}-y_I^*\| + \|y_I^*-y_{I+1}^*\|
    \leq \frac{\epsilon}{4L_0} + \frac{\eta l_{g,1}}{\mu} \leq \frac{\epsilon}{2L_0},
\end{equation*}
where we choose $1-\beta$ to be small (see \eqref{eq:thm-beta} for details) such that $\eta$ is small enough to make above inequality holds.
% $\|y_{I+1}-y_{I+1}^*\|\leq \epsilon/2L_0$. 
Repeating the same process yields the result.
\end{proof}

\subsection{Proof of \Cref{lem:averaging}}
\begin{lemma}[Averaging, Restatement of \Cref{lem:averaging}] \label{lm:averaging}
Under \cref{ass:relax-smooth,ass:f-and-g,ass:noise,ass:individual-noise} and event $\gE_{\init}\cap\gE_y^1$ (Option I) or $\gE_{\init}\cap\gE_y^2$ (Option II), we further set $\tau=\sqrt{\mu\alpha}$ in the averaging step (line 21 of \Cref{alg:bilevel}). Then for any $t\geq 0$ we have 
\begin{equation*}
    \|\hat{y}_t-y_t^*\| \leq \frac{2\epsilon}{L_0}
    \quad\text{and}\quad
    \|\hat{y}_{t+1}-\hat{y}_t\| \leq \frac{\mu\epsilon^2}{24L_0^2\tilde{\sigma}_{g,1}} =: \vartheta.
\end{equation*}
\end{lemma}

\begin{proof}[Proof of \Cref{lm:averaging}]
We will first show the following result by induction, i.e., for any $t\geq0$, the averaged sequence $\{\hat{y}_t\}$ satisfies
\begin{equation} \label{eq:hat-y-induction}
    \begin{aligned}
        \|\hat{y}_t-y_t^*\| \leq \frac{(1-\tau)\eta l_{g,1}}{\tau\mu} + \frac{\epsilon}{L_0}.
    \end{aligned}
\end{equation}
For $t=0$, by \Cref{lm:warm-start} we have
% \begin{equation*}
%     \|\hat{y}_0-y_0^*\| = \|y_{T_0}^{\init}-y_0^*\| \leq \sqrt{\frac{\mu}{8(\mu+l_{g,1})}}\frac{\epsilon}{L_0} \leq \frac{\epsilon}{L_0},
% \end{equation*}
\begin{equation*}
    \|\hat{y}_0-y_0^*\| = \|y_{T_0}^{\init}-y_0^*\| \leq \sqrt{\frac{\mu\alpha}{32}}\frac{\epsilon}{L_0} = \sqrt{\frac{1-\beta}{32}}\frac{\epsilon}{L_0} \leq \frac{\epsilon}{L_0},
\end{equation*}
thus the base case holds. Now suppose \eqref{eq:hat-y-induction} holds for some $t\geq 0$, then for time step $t+1$ we have
\begin{equation*}
    \begin{aligned}
        \|\hat{y}_{t+1} - y_{t+1}^*\|
        &= \|(1-\tau)(\hat{y}_t-y_{t+1}^*) + \tau(y_{t+1}-y_{t+1}^*)\| \\
        &= \|(1-\tau)(\hat{y}_t-y_t^*) + (1-\tau)(y_t^*-y_{t+1}^*) + \tau(y_{t+1}-y_{t+1}^*)\| \\
        &\leq (1-\tau)\|\hat{y}_t-y_t^*\| + (1-\tau)\|y_t^*-y_{t+1}^*\| + \tau\|y_{t+1}-y_{t+1}^*\| \\
        &\leq (1-\tau)\left(\frac{(1-\tau)\eta l_{g,1}}{\tau\mu} + \frac{\epsilon}{L_0}\right) + \frac{(1-\tau)\eta l_{g,1}}{\mu} + \frac{\tau\epsilon}{L_0} \\
        &\leq \frac{(1-\tau)\eta l_{g,1}}{\tau\mu} + \frac{\epsilon}{L_0},
    \end{aligned}
\end{equation*}
where we use induction hypothesis in the second inequality. Therefore, we have that \eqref{eq:hat-y-induction} holds for any $t\geq0$.
Also, as a consequence, for any $t\geq0$ we have
\begin{equation*}
    \begin{aligned}
        \|\hat{y}_{t+1} - \hat{y}_t\|
        &= \|\tau(y_{t+1}-\hat{y}_t)\| \\
        &\leq \tau\|y_{t+1}-y_{t+1}^*\| + \tau\|y_{t+1}^*-y_t^*\| + \tau\|y_t^*-\hat{y}_t\| \\
        &\leq \tau\left(\frac{\epsilon}{L_0} + \frac{\eta l_{g,1}}{\mu} + \frac{(1-\tau)\eta l_{g,1}}{\tau\mu} + \frac{\epsilon}{L_0}\right) \\
        &= \tau\left(\frac{\eta l_{g,1}}{\tau\mu} + \frac{2\epsilon}{L_0}\right).
    \end{aligned}
\end{equation*}
Now we plug in the definition of $\alpha,\beta,\tau,\eta$ as in \eqref{eq:para} to obtain 
\begin{equation*}
    \begin{aligned}
        \|\hat{y}_t-y_t^*\| 
        \leq \frac{(1-\tau)\eta l_{g,1}}{\tau\mu} + \frac{\epsilon}{L_0} 
        = \frac{\tilde{\sigma}_{g,1}}{\mu}\sqrt{1-\beta} + \frac{\epsilon}{L_0} \leq \frac{2\epsilon}{L_0}
    \end{aligned}
\end{equation*}
and 
\begin{equation*}
    \begin{aligned}
        \|\hat{y}_{t+1}-\hat{y}_t\| 
        \leq \tau\left(\frac{\eta l_{g,1}}{\tau\mu} + \frac{2\epsilon}{L_0}\right) \leq \frac{\mu\epsilon^2}{24L_0^2\tilde{\sigma}_{g,1}}.
    \end{aligned}
\end{equation*}
\end{proof}

\subsection{Proof of \Cref{lem:uppererror}}
\begin{lemma}[Restatement of \Cref{lem:uppererror}] \label{lm:x-momentum-recursion}
Under \cref{ass:relax-smooth,ass:f-and-g,ass:noise,ass:individual-noise} and event $\gE_{\init}\cap\gE_y^1$ (Option I) or $\gE_{\init}\cap\gE_y^2$ (Option II), define $\epsilon_t = m_t-\E_t[\gdf(x_t,\hat{y}_t;\Bar{\xi}_t)]$, then we have the following averaged cumulative error bound:
\begin{equation*}
    \begin{aligned}
        \frac{1}{T}\sum_{t=0}^{T-1}\E\|\epsilon_t\|
        &\leq \frac{\Bar{\sigma}}{T(1-\beta)} + \sqrt{1-\beta}\Bar{\sigma} + \frac{\Bar{L}_0}{\sqrt{1-\beta}}\sqrt{\frac{2(\eta^2+\vartheta^2)}{S}} + \Bar{L}_1\sqrt{\frac{2(\eta^2+\vartheta^2)}{S(1-\beta)}}\frac{1}{T}\sum_{t=0}^{T-1}\E\|\gdphi(x_t)\|.
    \end{aligned}
\end{equation*}
\end{lemma}

\begin{proof}[Proof of \Cref{lm:x-momentum-recursion}]
Define $\epsilon_t = m_t-\E_t[\gdf(x_t,\hat{y}_t;\Bar{\xi}_t)]$, also define $\tilde{\epsilon}_t$ and $\hat{\epsilon}_t$ as
\begin{equation*}
    \begin{aligned}
        \Tilde{\epsilon}_t &= \gdf(x_t,\hat{y}_t;\Bar{\xi}_t) - \E_t[\gdf(x_t,\hat{y}_t;\Bar{\xi}_t)], \\
        \hat{\epsilon}_t &= \gdf(x_t,\hat{y}_t;\Bar{\xi}_t) - \gdf(x_{t-1},\hat{y}_{t-1};\Bar{\xi}_t) - \E_t[\gdf(x_t,\hat{y}_t;\Bar{\xi}_t)] + \E_t[\gdf(x_{t-1},\hat{y}_{t-1};\Bar{\xi}_t)]. \\
    \end{aligned}
\end{equation*}
By definition of $\epsilon_t$, $\tilde{\epsilon}_t$ and $\hat{\epsilon}_t$, we have the following recursion for any $t\geq 0$:
\begin{equation} \label{eq:momentum-recursion}
    \epsilon_{t+1} = \beta\epsilon_t + (1-\beta)\hat{\epsilon}_{t+1} + \beta\Tilde{\epsilon}_{t+1}.
\end{equation}
Then we apply \eqref{eq:momentum-recursion} recursively to obtain
\begin{equation*}
    \epsilon_t = \beta^t\epsilon_0 + \beta\sum_{i=1}^{t}\beta^{t-i}\Tilde{\epsilon}_i + (1-\beta)\sum_{i=1}^{t}\beta^{t-i}\hat{\epsilon}_i,
\end{equation*}
which by triangle inequality and total expectation gives
\begin{equation} \label{eq:momentum-recursion-1}
    \E\|\epsilon_t\| = \beta^t\underbrace{\E\|\epsilon_0\|}_{Err_1} + (1-\beta)\underbrace{\E\left\|\sum_{i=1}^{t}\beta^{t-i}\hat{\epsilon}_i\right\|}_{Err_2} + \beta\underbrace{\E\left\|\sum_{i=1}^{t}\beta^{t-i}\Tilde{\epsilon}_i\right\|}_{Err_3}.
\end{equation}
\paragraph{Bounding $Err_1$.}
By definition of $\epsilon_0$ and \Cref{lm:hyper-var}, along with Jensen's inequality, we have
\begin{equation} \label{eq:Err-1}
    \E\|\epsilon_0\| \leq \sqrt{\E\|\epsilon_0\|^2} \leq \Bar{\sigma}.
\end{equation}
\paragraph{Bounding $Err_2$.}
We apply \Cref{lm:hyper-var} and follow the similar procedure as in \cite[Lemma D.9]{hao2024bilevel} to obtain
\begin{equation} \label{eq:Err-2}
    \E\left\|\sum_{i=1}^{t}\beta^{t-i}\tilde{\epsilon}_i\right\| 
    \leq \sqrt{\E\left\|\sum_{i=1}^{t}\beta^{t-i}\tilde{\epsilon}_i\right\|^2}
    \leq \sqrt{\sum_{i=1}^{t}\beta^{2(t-i)}\E\|\tilde{\epsilon}_i\|^2}
    \leq \frac{\Bar{\sigma}}{\sqrt{1-\beta}},
\end{equation}
where we use Jensen's inequality for the first step.
\paragraph{Bounding $Err_3$.}
We will first use induction to show that for $0\leq i\leq t+1$, the following inequality holds:
\begin{small}
\begin{equation} \label{eq:induction-error}
    \begin{aligned}
        \E\left[\left\|\sum_{j=1}^{t}\beta^{t-j}\hat{\epsilon}_j\right\|\right] 
        \leq \sqrt{\frac{2(\eta^2+\vartheta^2)}{S}}\Bar{L}_1\sum_{j=t+1-i}^{t}\beta^{t-j}\E\|\gdphi(x_j)\| + \E\left[\sqrt{\frac{2(\eta^2+\vartheta^2)\Bar{L}_0^2}{S}\sum_{j=1}^{i}\beta^{2j-2} + \left\|\sum_{j=1}^{t-i}\beta^{t-j}\hat{\epsilon}_j\right\|^2}\right].
    \end{aligned}
\end{equation}
\end{small}%
Then it's easy to check that by setting $i=t+1$ we can obtain the bound. When $i=0$, \eqref{eq:induction-error} holds obviously since
\begin{small}
\begin{equation*}
    \E\left[\left\|\sum_{j=1}^{t}\beta^{t-j}\hat{\epsilon}_j\right\|\right] 
    \leq \E\left[\sqrt{\left\|\sum_{j=1}^{t}\beta^{t-j}\hat{\epsilon}_j\right\|^2}\right]
    = \E\left[\left\|\sum_{j=1}^{t}\beta^{t-j}\hat{\epsilon}_j\right\|\right].
\end{equation*}
\end{small}%
Hence the base case stands. Now suppose \eqref{eq:induction-error} holds for some $i\geq 0$, and we aim to show that \eqref{eq:induction-error} holds for $i+1$. In fact, we have
\begin{small}
% \begin{equation*}
    \begin{align}
        &\E\left[\sqrt{\frac{2(\eta^2+\vartheta^2)\Bar{L}_0^2}{S}\sum_{j=1}^{i}\beta^{2j-2} + \left\|\sum_{j=1}^{t-i}\beta^{t-j}\hat{\epsilon}_j\right\|^2}\right] \\
        &= \E_{\gF_{t-i-1}}\left[\E_{t-i-1}\left[\sqrt{\frac{2(\eta^2+\vartheta^2)\Bar{L}_0^2}{S}\sum_{j=1}^{i}\beta^{2j-2} + \left\|\sum_{j=1}^{t-i}\beta^{t-j}\hat{\epsilon}_j\right\|^2}\right]\right] \label{eq:line-1} \\
        &\leq \E_{\gF_{t-i-1}}\left[\E_{t-i-1}\sqrt{\left[\frac{2(\eta^2+\vartheta^2)\Bar{L}_0^2}{S}\sum_{j=1}^{i}\beta^{2j-2} + \left\|\sum_{j=1}^{t-i}\beta^{t-j}\hat{\epsilon}_j\right\|^2\right]}\right] \label{eq:line-2} \\
        &= \E_{\gF_{t-i-1}}\left[\E_{t-i-1}\sqrt{\left[\frac{2(\eta^2+\vartheta^2)\Bar{L}_0^2}{S}\sum_{j=1}^{i}\beta^{2j-2} + \beta^{2i}\|\hat{\epsilon}_{t-i}\|^2 + \left\|\sum_{j=1}^{t-i-1}\beta^{t-j}\hat{\epsilon}_j\right\|^2\right]}\right] \label{eq:line-3} \\
        &\leq \E_{\gF_{t-i-1}}\left[\E_{t-i-1}\sqrt{\left[\frac{2(\eta^2+\vartheta^2)\Bar{L}_0^2}{S}\sum_{j=1}^{i}\beta^{2j-2} + \frac{\beta^{2i}}{S}2(\Bar{L}_0^2+\Bar{L}_1^2\|\gdphi(x_{t-i})\|^2)(\eta^2+\vartheta^2) + \left\|\sum_{j=1}^{t-i-1}\beta^{t-j}\hat{\epsilon}_j\right\|^2\right]}\right] \label{eq:line-4} \\
        &= \E_{\gF_{t-i-1}}\left[\sqrt{\frac{2\beta^{2i}}{S}\Bar{L}_1^2(\eta^2+\vartheta^2)\|\gdphi(x_{t-i})\|^2+\frac{2(\eta^2+\vartheta^2)\Bar{L}_0^2}{S}\sum_{j=1}^{i+1}\beta^{2j-2}+\left\|\sum_{j=1}^{t-i-1}\beta^{t-j}\hat{\epsilon}_j\right\|^2}\right] \label{eq:line-5} \\
        &\leq \E_{\gF_{t-i-1}}\left[\sqrt{\frac{2(\eta^2+\vartheta^2)}{S}}\beta^{i}\Bar{L}_1\|\gdphi(x_{t-i})\|+\sqrt{\frac{2(\eta^2+\vartheta^2)\Bar{L}_0^2}{S}\sum_{j=1}^{i+1}\beta^{2j-2}+\left\|\sum_{j=1}^{t-i-1}\beta^{t-j}\hat{\epsilon}_j\right\|^2}\right], \label{eq:line-6}
    \end{align}
% \end{equation*}
\end{small}%
where \eqref{eq:line-1} follows by law of total expectation, \eqref{eq:line-2} follows by Jensen's inequality, \eqref{eq:line-3} uses the fact that $\hat{\epsilon}_j$ for $j<t-i$ are $\gF_{t-i-1}$-measurable, and are uncorrelated with $\hat{\epsilon}_{t-i}$; for \eqref{eq:line-4} we use \Cref{lm:hyper-estimator,lm:averaging} to derive
\begin{equation*}
    \begin{aligned}
        \E_{t-i-1}[\|\hat{\epsilon}_{t-i}\|^2]
        &= \E_{t-i-1}\left[\|\gdf(x_{t-i},\hat{y}_{t-i};\Bar{\xi}_{t-i}) - \gdf(x_{t-i-1},\hat{y}_{t-i-1};\Bar{\xi}_{t-i}) \right.\\
        &\quad\left.- \E_{t-i}[\gdf(x_{t-i},\hat{y}_{t-i};\Bar{\xi}_{t-i})] + \E_{t-i}[\gdf(x_{t-i-1},\hat{y}_{t-i-1};\Bar{\xi}_{t-i})]\|^2\right] \\
        &\leq \E_{t-i-1}\left[\|\gdf(x_{t-i},\hat{y}_{t-i};\Bar{\xi}_{t-i}) - \gdf(x_{t-i-1},\hat{y}_{t-i-1};\Bar{\xi}_{t-i})\|^2\right] \\
        &\leq 2\E_{t-i-1}\left[\|\gdf(x_{t-i},\hat{y}_{t-i};\Bar{\xi}_{t-i}) - \gdf(x_{t-i},\hat{y}_{t-i-1};\Bar{\xi}_{t-i})\|^2\right] \\
        &\quad+ 2\E_{t-i-1}\left[\|\gdf(x_{t-i},\hat{y}_{t-i-1};\Bar{\xi}_{t-i}) - \gdf(x_{t-i-1},\hat{y}_{t-i-1};\Bar{\xi}_{t-i})\|^2\right] \\
        &\leq \frac{2}{S}(\Bar{L}_0^2+\Bar{L}_1^2\|\gdphi(x_{t-i})\|^2)(\|\hat{y}_{t-i}-\hat{y}_{t-i-1}\|^2 + \|x_{t-i}-x_{t-i-1}\|^2) \\
        &= \frac{2}{S}(\Bar{L}_0^2+\Bar{L}_1^2\|\gdphi(x_{t-i})\|^2)(\eta^2+\vartheta^2).
    \end{aligned}
\end{equation*}
And \eqref{eq:line-5} follows from the fact that $x_{t-i}$ is $\gF_{t-i-1}$-measurable, \eqref{eq:line-6} uses $\sqrt{a+b}\leq \sqrt{a}+\sqrt{b}$ for $a,b\geq0$. Hence the induction proof is completed. We set $i=t+1$ to obtain 
\begin{equation} \label{eq:Err-3}
    \begin{aligned}
        \E\left[\left\|\sum_{i=1}^{t}\beta^{t-i}\hat{\epsilon}_i\right\|\right] 
        &\leq \sqrt{\frac{2(\eta^2+\vartheta^2)}{S}}\Bar{L}_1\sum_{i=0}^{t}\beta^{t-i}\E\|\gdphi(x_i)\| + \sqrt{\frac{2(\eta^2+\vartheta^2)\Bar{L}_0^2}{S}\sum_{i=0}^{t}\beta^{2i}} \\
        &\leq \sqrt{\frac{2(\eta^2+\vartheta^2)}{S}}\Bar{L}_1\sum_{i=0}^{t}\beta^{t-i}\E\|\gdphi(x_i)\| + \frac{\Bar{L}_0}{\sqrt{1-\beta}}\sqrt{\frac{2(\eta^2+\vartheta^2)}{S}}.
    \end{aligned}
\end{equation}
\paragraph{Final Bound.}
Combining \eqref{eq:Err-1}, \eqref{eq:Err-2} and \eqref{eq:Err-3} yields
\begin{equation*}
    \begin{aligned}
        \E\|\epsilon_t\|
        &\leq \beta^t\Bar{\sigma} + \sqrt{1-\beta}\Bar{\sigma} + \frac{\Bar{L}_0}{\sqrt{1-\beta}}\sqrt{\frac{2(\eta^2+\vartheta^2)}{S}} + \sqrt{\frac{2(\eta^2+\vartheta^2)}{S}}\Bar{L}_1\sum_{i=0}^{t}\beta^{t-i}\E\|\gdphi(x_i)\|.
    \end{aligned}
\end{equation*}
Taking summation and dividing $1/T$ on both sides gives the final result
\begin{equation*}
    \begin{aligned}
        \frac{1}{T}\sum_{t=0}^{T-1}\E\|\epsilon_t\|
        &\leq \frac{\Bar{\sigma}}{T(1-\beta)} + \sqrt{1-\beta}\Bar{\sigma} + \frac{\Bar{L}_0}{\sqrt{1-\beta}}\sqrt{\frac{2(\eta^2+\vartheta^2)}{S}} + \Bar{L}_1\sqrt{\frac{2(\eta^2+\vartheta^2)}{S(1-\beta)}}\frac{1}{T}\sum_{t=0}^{T-1}\E\|\gdphi(x_t)\|.
    \end{aligned}
\end{equation*}
\end{proof}

\section{Proof of \Cref{thm:main}}
\label{sec:proof-main}
Before starting the proof of main results, i.e., \Cref{thm:main}, we first need the following lemma.

\begin{lemma} \label{lm:descent-lemma}
Suppose that \cref{ass:relax-smooth,ass:f-and-g,ass:noise,ass:individual-noise} hold. For any $\eta$ satisfying
\begin{equation*}
    \eta \leq \frac{1}{\sqrt{2(1+l_{g,1}^2/\mu^2)(L_{x,1}^2+L_{y,1}^2)}},
\end{equation*}
it holds that
\begin{equation*}
    \begin{aligned}
        &\left(1-\frac{1}{2}\eta L_1-2L_1\|\hat{y}_t-y_t^*\|\right)\frac{1}{T}\sum_{t=0}^{T-1}\E\|\gdphi(x_t)\| \\
        &\leq \frac{\Phi(x_0)-\Phi(x_T)}{T\eta} + \frac{2}{T}\sum_{t=0}^{T-1}\E\|\epsilon_t\| + \frac{2l_{g,1}l_{f,0}}{\mu}\left(1-\frac{\mu}{l_{g,1}}\right)^Q + \frac{2L_0}{T}\sum_{t=0}^{T-1}\|\hat{y}_t-y_t^*\| + \frac{1}{2}\eta L_0.
    \end{aligned}
\end{equation*}
\end{lemma}

\begin{proof}[Proof of \Cref{lm:descent-lemma}]
Define $h_t = m_t-\gdphi(x_t)$. Then we apply \Cref{lm:technical-descent} to obtain
\begin{equation} \label{eq:one-step-descent}
    \begin{aligned}
        \Phi(x_{t+1})
        &\leq \Phi(x_t) + \langle \gdphi(x_t), x_{t+1}-x_t \rangle + \frac{L_0+L_1\|\gdphi(x_t)\|}{2}\|x_{t+1}-x_t\|^2 \\
        &= \Phi(x_t) - \eta\langle \gdphi(x_t), \frac{m_t}{\|m_t\|} \rangle + \frac{1}{2}\eta^2(L_0 + L_1\|\gdphi(x_t)\|) \\
        &= \Phi(x_t) - \eta\langle m_t-h_t, \frac{m_t}{\|m_t\|} \rangle + \frac{1}{2}\eta^2(L_0 + L_1\|\gdphi(x_t)\|) \\
        &= \Phi(x_t) - \eta\|m_t\| + \eta\langle h_t, \frac{m_t}{\|m_t\|} \rangle + \frac{1}{2}\eta^2(L_0 + L_1\|\gdphi(x_t)\|) \\
        &\leq \Phi(x_t) - \eta\|m_t\| + \eta\|h_t\| + \frac{1}{2}\eta^2(L_0 + L_1\|\gdphi(x_t)\|) \\
        &\leq \Phi(x_t) - \eta\|\gdphi(x_t)\| + 2\eta\|h_t\| + \frac{1}{2}\eta^2(L_0 + L_1\|\gdphi(x_t)\|),
    \end{aligned}
\end{equation}
where for the last two lines we use Cauchy-Schwarz inequality and $\|h_t\| = \|\gdphi(x_t)+h_t\| \geq \|\gdphi(x_t)\|-\|h_t\|$. Now expanding $h_t$ by triangle inequality, we have
\begin{equation*}
    \begin{aligned}
        \|h_t\|
        &= \|m_t-\gdphi(x_t)\| \\
        &\leq \|m_t-\E_t[\gdf(x_t,\hat{y}_t;\Bar{\xi}_t)]\| + \|\E_t[\gdf(x_t,\hat{y}_t;\Bar{\xi}_t)]-\gdf(x_t,\hat{y}_t)\| + \|\gdf(x_t,\hat{y}_t)-\gdphi(x_t)\| \\
        &\leq \|\epsilon_t\| + \frac{l_{g,1}l_{f,0}}{\mu}\left(1-\frac{\mu}{l_{g,1}}\right)^Q + (L_0+L_1\|\gdphi(x_t)\|)\|\hat{y}_t-y_t^*\|,
    \end{aligned}
\end{equation*}
where we use definition of $\epsilon_t$, \cref{lm:hyper-var,lm:gdf-gdphi-bias} in the last inequality.
Plugging the above inequality back into \eqref{eq:one-step-descent} we obtain
\begin{equation*}
    \begin{aligned}
        \Phi(x_{t+1})
        &\leq \Phi(x_t) - \eta\|\gdphi(x_t)\| + 2\eta\|\epsilon_t\| + 2\eta\frac{l_{g,1}l_{f,0}}{\mu}\left(1-\frac{\mu}{l_{g,1}}\right)^Q \\
        &\quad+ 2\eta(L_0+L_1\|\gdphi(x_t)\|)\|\hat{y}_t-y_t^*\| + \frac{1}{2}\eta^2(L_0 + L_1\|\gdphi(x_t)\|) \\
        &= \Phi(x_t) - \left(\eta-\frac{1}{2}\eta^2L_1-2\eta L_1\|\hat{y}_t-y_t^*\|\right)\|\gdphi(x_t)\| + 2\eta\|\epsilon_t\| \\
        &\quad+ 2\eta\frac{l_{g,1}l_{f,0}}{\mu}\left(1-\frac{\mu}{l_{g,1}}\right)^Q + 2\eta L_0\|\hat{y}_t-y_t^*\| + \frac{1}{2}\eta^2L_0.
    \end{aligned}
\end{equation*}
Dividing $1/T\eta$ on both sides, then taking telescope sum and total expectation, and rearranging it finally yields
\begin{equation*}
    \begin{aligned}
        &\left(1-\frac{1}{2}\eta L_1-2L_1\|\hat{y}_t-y_t^*\|\right)\frac{1}{T}\sum_{t=0}^{T-1}\E\|\gdphi(x_t)\| \\
        &\leq \frac{\Phi(x_0)-\Phi(x_T)}{T\eta} + \frac{2}{T}\sum_{t=0}^{T-1}\E\|\epsilon_t\| + \frac{2l_{g,1}l_{f,0}}{\mu}\left(1-\frac{\mu}{l_{g,1}}\right)^Q + \frac{2L_0}{T}\sum_{t=0}^{T-1}\|\hat{y}_t-y_t^*\| + \frac{1}{2}\eta L_0.
    \end{aligned}
\end{equation*}
\end{proof}

\begin{theorem}[Restatement of \Cref{thm:main}] \label{thm:main-appendix}
Suppose \cref{ass:relax-smooth,ass:f-and-g,ass:noise,ass:individual-noise} hold. Let $\{x_t\}$ be the iterates produced by \Cref{alg:bilevel}. For any given $\delta\in(0,1)$ and any small $\epsilon>0$ satisfying
\begin{equation} \label{eq:thm-eps}
    \begin{aligned}
        \epsilon \leq \min\left\{\frac{L_0}{32L_1}, \frac{l_{g,1}L_0}{\mu\Bar{L}_1}, \frac{L_0}{8\Bar{L}_1}, \frac{L_0l_{g,1}\tilde{\sigma}_{g,1}}{\mu^2}, \frac{L_0}{\mu}\sqrt{\frac{l_{g,1}\tilde{\sigma}_{g,1}}{L_1}}, \left(\frac{164\cdot32ed_0L_0^2\tilde{\sigma}_{g,1}^2}{\delta\mu^2}\max\left\{\frac{l_{g,1}}{\tilde{\sigma}_{g,1}}, \frac{\Bar{\sigma}}{d_0}\right\}\right)^{1/3}\right\},
    \end{aligned}
\end{equation}
if $\sigma_{g,1}$ satisfies
\begin{equation} \label{eq:sigma-g1}
    \sigma_{g,1} = \left(\min\left\{\frac{\mu^2\epsilon^2}{164\cdot16L_0^2\tilde{\sigma}_{g,1}^2\ln(P)}, \frac{l_{g,1}}{4\tilde{\sigma}_{g,1}L_1}, \frac{\epsilon^2}{4\Bar{\sigma}^2}\right\}\right)^{1/4}\tilde{\sigma}_{g,1}
\end{equation}
with $\tilde{\sigma}_{g,1}=O(1)$, and we set parameters $\alpha, \alpha^{\init}, \beta, \gamma, \eta, \tau, I, N, S, Q, T_0$ as 
\begin{equation} \label{eq:thm-beta}
    1-\beta = \min\left\{\frac{\mu^2\epsilon^2}{164\cdot16L_0^2\tilde{\sigma}_{g,1}^2\ln(P)}, \frac{l_{g,1}}{4\tilde{\sigma}_{g,1}L_1}, \frac{\epsilon^2}{4\Bar{\sigma}^2}\right\}, 
    \quad
    \eta = \min\left\{\frac{\tilde{\sigma}_{g,1}}{l_{g,1}}, \frac{d_0}{\Bar{\sigma}}\right\}(1-\beta),
\end{equation}
\begin{equation} \label{eq:thm-alpha} 
    \alpha^{\init} = \frac{1-\beta}{\mu+l_{g,1}}
    \quad
    \alpha = \frac{1-\beta}{\mu},
    \quad
    % \alpha' = \frac{1-\beta}{\mu+l_{g,1}},
    % \quad
    \gamma = \frac{1-\sqrt{\mu\alpha}}{1+\sqrt{\mu\alpha}},
    \quad
    \tau = 1-\sqrt{\mu\alpha},
    % \quad
    % \sigma_{g,1} = (\mu\alpha)^{1/4}\tilde{\sigma}_{g,1},
\end{equation}
\begin{equation} \label{eq:thm-T0}
    T_0 = \ln\left(\frac{\mu^3\alpha^3\epsilon^2}{256L_0^2\|y_0^{\init}-y_0^*\|^2}\right) 
    \Big/
    \ln\left(1-\frac{\mu\alpha}{4}\right),
\end{equation}
\begin{equation} \label{eq:thm-I}
    I=\frac{\mu\epsilon}{2(1-\beta)L_0\tilde{\sigma}_{g,1}},
    \quad
    N = \ln\left(\frac{\mu\alpha}{128}\right) \Big/ \ln\left(1-\frac{\sqrt{\mu\alpha}}{4}\right),
\end{equation}
\begin{equation} \label{eq:thm-S}
    S = \max\left\{128\ln(P), \frac{128\Bar{L}_0^2}{L_0^2}\ln(P), \frac{\mu^2\Bar{L}_0^2}{l_{g,1}^2L_0^2}\right\}, 
    \quad
    Q = \ln\left(1-\frac{\mu}{l_{g,1}}\right) \Big/ \ln\left(\frac{\mu\epsilon}{l_{g,1}l_{f,0}}\right),
\end{equation}
where $d_0$ and $P$ are defined in \eqref{eq:d0-P}. Then with probability at least $1-2\delta$ over the randomness in $\sigma(\gF^{\init}\cup\widetilde{\gF}_T^1)$ (for Option I) or $\sigma(\gF^{\init}\cup(\cup_{t\leq T}\widetilde{\gF}_t^2))$ (for Option II), \Cref{alg:bilevel} guarantees $\frac{1}{T}\sum_{t=1}^{T}\E\|\gdphi(x_t)\|\leq 20\epsilon$ within $T = \frac{4d_0}{\eta\epsilon} =\widetilde{O}(1/\epsilon^3)$ iterations, where the expectation is taken over the randomness in $\gF_T$. For Option I, it requires $T_0+SQT= \widetilde{O}(1/\epsilon^3)$ oracle calls of stochastic gradient or Hessian/Jacobian vector product. For Option II, it requires $T_0+\frac{NT}{I}+SQT= \widetilde{O}(1/\epsilon^3)$ oracle calls of stochastic gradient or Hessian/Jacobian vector product.
\end{theorem}

\begin{proof}[Proof of \Cref{thm:main-appendix}]
By \Cref{lm:x-momentum-recursion,lm:descent-lemma}, we have
\begin{equation*}
    \begin{aligned}
        &\left(1-\frac{1}{2}\eta L_1-2L_1\|\hat{y}_t-y_t^*\|\right)\frac{1}{T}\sum_{t=0}^{T-1}\E\|\gdphi(x_t)\| \\
        &\leq \frac{\Phi(x_0)-\Phi(x_T)}{T\eta} + \frac{2}{T}\sum_{t=0}^{T-1}\E\|\epsilon_t\| + \frac{2l_{g,1}l_{f,0}}{\mu}\left(1-\frac{\mu}{l_{g,1}}\right)^Q + \frac{2L_0}{T}\sum_{t=0}^{T-1}\|\hat{y}_t-y_t^*\| + \frac{1}{2}\eta L_0 \\
        &\leq \frac{d_0}{T\eta} + \frac{2l_{g,1}l_{f,0}}{\mu}\left(1-\frac{\mu}{l_{g,1}}\right)^Q + \frac{2L_0}{T}\sum_{t=0}^{T-1}\|\hat{y}_t-y_t^*\| + \frac{1}{2}\eta L_0 \\
        &\quad+ \frac{2\Bar{\sigma}}{T(1-\beta)} + 2\sqrt{1-\beta}\Bar{\sigma} + \frac{2\Bar{L}_0}{\sqrt{1-\beta}}\sqrt{\frac{2(\eta^2+\vartheta^2)}{S}} + 2\Bar{L}_1\sqrt{\frac{2(\eta^2+\vartheta^2)}{S(1-\beta)}}\frac{1}{T}\sum_{t=0}^{T-1}\E\|\gdphi(x_t)\|.
    \end{aligned}
\end{equation*}
Rearranging the above inequality gives
\begin{equation*}
    \begin{aligned}
        &\underbrace{\left(1-\frac{1}{2}\eta L_1-2L_1\|\hat{y}_t-y_t^*\|-2\Bar{L}_1\sqrt{\frac{2(\eta^2+\vartheta^2)}{S(1-\beta)}}\right)}_{(\mathrm{LHS})}\frac{1}{T}\sum_{t=0}^{T-1}\E\|\gdphi(x_t)\| \\
        &\leq \underbrace{\frac{d_0}{T\eta} + \frac{2l_{g,1}l_{f,0}}{\mu}\left(1-\frac{\mu}{l_{g,1}}\right)^Q + \frac{2L_0}{T}\sum_{t=0}^{T-1}\|\hat{y}_t-y_t^*\| + \frac{1}{2}\eta L_0 + \frac{2\Bar{\sigma}}{T(1-\beta)} + 2\sqrt{1-\beta}\Bar{\sigma} + \frac{2\Bar{L}_0}{\sqrt{1-\beta}}\sqrt{\frac{2(\eta^2+\vartheta^2)}{S}}}_{(\mathrm{RHS})}.
    \end{aligned}
\end{equation*}
\paragraph{Bounding (LHS).}
By \Cref{lm:averaging}, we have
\begin{equation} \label{eq:lhs}
    \begin{aligned}
        (\mathrm{LHS})
        &\geq 1 - \frac{\tilde{\sigma}_{g,1}L_1}{2l_{g,1}}(1-\beta) - 2L_1\frac{2\epsilon}{L_0} - 2\Bar{L}_1\sqrt{\frac{2(\eta^2+\vartheta^2)}{S(1-\beta)}} \\
        &\geq 1 - \frac{1}{8} - \frac{1}{8} - \frac{1}{4}
        = \frac{1}{2}
    \end{aligned}
\end{equation}
\paragraph{Bounding (RHS).}
By choice of parameters, we have
\begin{equation} \label{eq:rhs}
    \begin{aligned}
        (\mathrm{RHS})
        &\leq \frac{1}{4}\epsilon + 2\epsilon + 4\epsilon + \epsilon + \frac{1}{2}\epsilon + \epsilon + \epsilon 
        \leq 10\epsilon.
    \end{aligned}
\end{equation}
Combining \eqref{eq:lhs} and \eqref{eq:rhs} finally yields
\begin{equation*}
    \frac{1}{T}\sum_{t=0}^{T-1}\E\|\gdphi(x_t)\| \leq 20\epsilon.
\end{equation*}
\end{proof}

\section{Omitted Proofs in \Cref{sec:auxiliary-bilevel}}
\label{sec:omit-proof}
\subsection{Proof of \Cref{lm:gdf-gdphi-bias}}

\begin{lemma}[Restatement of \Cref{lm:gdf-gdphi-bias}]
Under \cref{ass:relax-smooth,ass:f-and-g,ass:noise,ass:individual-noise}, we have
\begin{equation*}
    \begin{aligned}
        \|\gdf(x,y) - \gdphi(x)\| \leq (\Bar{L} + L_{x,1}\|\gdphi(x)\|)\|y-y^*(x)\|,
    \end{aligned}
\end{equation*}
where constant $\Bar{L}$ is defined as
\begin{equation*}
    \Bar{L} \coloneqq L_{x,0} + L_{x,1}\frac{l_{g,1}l_{f,0}}{\mu} + \frac{l_{g,1}}{\mu}(L_{y,0}+L_{y,1}l_{f,0}) + l_{f,0}\frac{\mu l_{g,2}+l_{g,1}l_{g,2}}{\mu^2} \leq L_0.
\end{equation*}
\end{lemma}

\begin{proof}[Proof of \Cref{lm:gdf-gdphi-bias}]
Recall that the exact expressions of $\gdf(x,y)$ and $\gdphi(x)$ are
\begin{equation*}
    \gdf(x,y) = \gdx f(x,y) - \gdxy g(x,y)[\gdyy g(x,y)]^{-1}\gdy f(x,y)
\end{equation*}
and
\begin{equation*}
    \gdphi(x) = \gdx f(x,y^*(x)) - \gdxy g(x,y^*(x))[\gdyy g(x,y^*(x))]^{-1}\gdy f(x,y^*(x)).
\end{equation*}
Then by \Cref{ass:f-and-g} we have
{\allowdisplaybreaks
% \begin{equation*}
    \begin{align*}
        &\|\gdf(x,y) - \gdphi(x)\|
        \leq \|\gdx f(x,y) - \gdx f(x,y^*(x))\| \\
        &\quad+ \|\gdxy g(x,y)[\gdyy g(x,y)]^{-1}\gdy f(x,y) - \gdxy g(x,y^*(x))[\gdyy g(x,y^*(x))]^{-1}\gdy f(x,y^*(x))\| \\
        &\leq (L_{x,0}+L_{x,1}\|\gdx f(x,y^*(x))\|)\|y-y^*(x)\| \\
        &\quad+ \|\gdxy g(x,y)[\gdyy g(x,y)]^{-1}\gdy f(x,y) - \gdxy g(x,y^*(x))[\gdyy g(x,y)]^{-1}\gdy f(x,y)\| \\
        &\quad+ \|\gdxy g(x,y^*(x))[\gdyy g(x,y)]^{-1}\gdy f(x,y) - \gdxy g(x,y^*(x))[\gdyy g(x,y^*(x))]^{-1}\gdy f(x,y)\| \\
        &\quad+ \|\gdxy g(x,y^*(x))[\gdyy g(x,y^*(x))]^{-1}\gdy f(x,y) - \gdxy g(x,y^*(x))[\gdyy g(x,y^*(x))]^{-1}\gdy f(x,y^*(x))\| \\
        &\leq \left(L_{x,0}+L_{x,1}\left(\frac{l_{g,1}l_{f,0}}{\mu} + \|\gdphi(x)\|\right)\right)\|y-y^*(x)\| \\
        &\quad+ \frac{l_{f,0}}{\mu}l_{g,2}\|y-y^*(x)\| + \frac{l_{f,0}l_{g,1}}{\mu^2}l_{g,2}\|y-y^*(x)\| + \frac{l_{g,1}}{\mu}(L_{y,0}+L_{y,1}\|\gdy f(x,y^*(x))\|)\|y-y^*(x)\| \\
        % &\leq \left(L_{x,0} + L_{x,1}\frac{l_{g,1}l_{f,0}}{\mu}+L_{x,1}\|\gdphi(x)\|\right)\|y-y^*(x)\| \\
        % &\quad+ \left(\frac{l_{f,0}l_{g,2}}{\mu} + \frac{l_{f,0}l_{g,1}l_{g,2}}{\mu^2} + \frac{l_{g,1}}{\mu}(L_{y,0}+L_{y,1}l_{f,0})\right)\|y-y^*(x)\| \\
        &= \left(L_{x,0} + L_{x,1}\frac{l_{g,1}l_{f,0}}{\mu} + \frac{l_{g,1}}{\mu}(L_{y,0}+L_{y,1}l_{f,0}) + l_{f,0}\frac{\mu l_{g,2}+l_{g,1}l_{g,2}}{\mu^2} + L_{x,1}\|\gdphi(x)\|\right)\|y-y^*(x)\|.
    \end{align*}
% \end{equation*}
}%
By definition of $\Bar{L}$ we conclude the proof.
\end{proof}

\subsection{Proof of \Cref{lm:hyper-estimator}}

\begin{lemma}[Restatement of \Cref{lm:hyper-estimator}]
Under \cref{ass:relax-smooth,ass:f-and-g,ass:noise,ass:individual-noise}, we have
\begin{enumerate}[(i)]
    \item For any fixed $y\in\R^{d_y}$ and any $x_1,x_2\in\R^{d_x}$,
    \begin{equation*}
        \E_{\Bar{\xi}}\|\gdf(x_1,y;\Bar{\xi}) - \gdf(x_2,y;\Bar{\xi})\|^2 \leq (\Bar{L}_0^2+\Bar{L}_1^2\|\gdphi(x_1)\|^2)\|x_1-x_2\|^2.
    \end{equation*}
    \item For any fixed $x\in\R^{d_x}$ and any $y_1,y_2\in\R^{d_y}$,
    \begin{equation*}
        \E_{\Bar{\xi}}\|\gdf(x,y_1;\Bar{\xi}) - \gdf(x,y_2;\Bar{\xi})\|^2 \leq (\Bar{L}_0^2+\Bar{L}_1^2\|\gdphi(x_1)\|^2)\|x_1-x_2\|^2.
    \end{equation*}
\end{enumerate}
In the above expressions, we define $\Bar{L}_0$ and $\Bar{L}_1$ as 
% \begin{small}
\begin{equation*}
    \begin{aligned}
        \Bar{L}_0 &= \left\{4\left(L_{x,0} + L_{x,1}\left(\frac{l_{g,1}l_{f,0}}{\mu} + \left(L_{x,0} + \frac{L_{x,1}l_{g,1}l_{f,0}}{\mu}\right)\|y_1-y_1^*\|\right)\right)^2 \right.\\
        &\left.\quad\quad\quad\quad\quad\quad\quad\quad+ \frac{6Q}{2\mu l_{g,1}-\mu^2}\left(l_{g,1}^2(L_{y,0}+L_{y,1}l_{f,0})^2 + l_{f,0}^2l_{g,2}^2 + \frac{l_{f,0}^2l_{g,1}^2l_{g,2}^2K^2}{(l_{g,1}-\mu)^2}\right)\right\}^{1/2}, \\
        \Bar{L}_1 &= 2L_{x,1}(1+L_{x,1}\|y_1-y_1^*\|).
    \end{aligned} 
\end{equation*}
% \end{small}
\end{lemma}

\begin{proof}[Proof of \Cref{lm:hyper-estimator}]
We show statement $(i)$ of the lemma, and $(ii)$ follows by similar arguments. For any fixed $y\in\R^{d_y}$ and any $x_1,x_2\in\R^{d_x}$, by definition of $\gdf(x,y;\Bar{\xi})$ we have
\begin{small}
\begin{equation*}
    \begin{aligned}
        &\|\gdf(x_1,y;\Bar{\xi}) - \gdf(x_2,y;\Bar{\xi})\|^2 \\
        &\leq 2\|\gdx F(x_1,y;\xi) - \gdx F(x_2,y;\xi)\|^2 + 2\left\|\gdxy G(x_1,y;\zeta^{(0)})\left[\frac{Q}{l_{g,1}}\prod_{i=1}^{\mathsf{q}}\left(I - \frac{1}{l_{g,1}}\gdyy G(x_1,y;\zeta^{(i)})\right)\right]\gdy F(x_1,y;\xi) \right.\\
        &\left.\quad\quad\quad\quad\quad\quad\quad - \gdxy G(x_2,y;\zeta^{(0)})\left[\frac{Q}{l_{g,1}}\prod_{i=1}^{\mathsf{q}}\left(I - \frac{1}{l_{g,1}}\gdyy G(x_2,y;\zeta^{(i)})\right)\right]\gdy F(x_2,y;\xi)\right\|^2 \\
        &\leq 2(L_{x,0}+L_{x,1}\|\gdx f(x_1,y)\|)^2\|x_1-x_2\|^2 + 2\left\|\gdxy G(x_1,y;\zeta^{(0)})\left[\frac{Q}{l_{g,1}}\prod_{i=1}^{\mathsf{q}}\left(I - \frac{1}{l_{g,1}}\gdyy G(x_1,y;\zeta^{(i)})\right)\right]\gdy F(x_1,y;\xi) \right.\\
        &\left.\quad\quad\quad\quad\quad\quad\quad - \gdxy G(x_2,y;\zeta^{(0)})\left[\frac{Q}{l_{g,1}}\prod_{i=1}^{\mathsf{q}}\left(I - \frac{1}{l_{g,1}}\gdyy G(x_2,y;\zeta^{(i)})\right)\right]\gdy F(x_2,y;\xi)\right\|^2. \\
    \end{aligned}
\end{equation*}
\end{small}%
For the second term above, we have
{\allowdisplaybreaks
\begin{small}
%\begin{equation*}
    \begin{align*}
        &\left\|\gdxy G(x_1,y;\zeta^{(0)})\left[\frac{Q}{l_{g,1}}\prod_{i=1}^{\mathsf{q}}\left(I - \frac{1}{l_{g,1}}\gdyy G(x_1,y;\zeta^{(i)})\right)\right]\gdy F(x_1,y;\xi) \right.\\
        &\left.\quad\quad\quad\quad\quad\quad\quad\quad\quad - \gdxy G(x_2,y;\zeta^{(0)})\left[\frac{Q}{l_{g,1}}\prod_{i=1}^{\mathsf{q}}\left(I - \frac{1}{l_{g,1}}\gdyy G(x_2,y;\zeta^{(i)})\right)\right]\gdy F(x_2,y;\xi)\right\|^2 \\
        &\leq 3l_{g,1}^2\frac{Q^2}{l_{g,1}^2}\left(1-\frac{\mu}{l_{g,1}}\right)^{2\mathsf{q}}\|\gdy F(x_1,y;\xi) - \gdy F(x_2,y;\xi)\|^2 \\
        &\quad\quad\quad\quad+ 3l_{f,0}^2\frac{Q^2}{l_{g,1}^2}\left(1-\frac{\mu}{l_{g,1}}\right)^{2\mathsf{q}}\|\gdxy G(x_1,y;\zeta^{(0)}) - \gdxy G(x_2,y;\zeta^{(0)})\|^2 \\
        &\quad\quad\quad\quad+ 3l_{g,1}^2l_{f,0}^2\left\|\frac{Q}{l_{g,1}}\prod_{i=1}^{\mathsf{q}}\left(I - \frac{1}{l_{g,1}}\gdyy G(x_1,y;\zeta^{(i)})\right) - \frac{Q}{l_{g,1}}\prod_{i=1}^{\mathsf{q}}\left(I - \frac{1}{l_{g,1}}\gdyy G(x_2,y;\zeta^{(i)})\right)\right\|^2 \\
        &\leq 3Q^2\left(1-\frac{\mu}{l_{g,1}}\right)^{2\mathsf{q}}(L_{y,0}+L_{y,1}\|\gdy f(x_1,y)\|)^2\|x_1-x_2\|^2 + 3\frac{l_{f,0}^2Q^2}{l_{g,1}^2}\left(1-\frac{\mu}{l_{g,1}}\right)^{2\mathsf{q}}l_{g,2}^2\|x_1-x_2\|^2 \\
        &\quad\quad\quad\quad+ 3l_{f,0}^2Q^2\left\|\prod_{i=1}^{\mathsf{q}}\left(I - \frac{1}{l_{g,1}}\gdyy G(x_1,y;\zeta^{(i)})\right) - \prod_{i=1}^{\mathsf{q}}\left(I - \frac{1}{l_{g,1}}\gdyy G(x_2,y;\zeta^{(i)})\right)\right\|^2. \\
    \end{align*}
%\end{equation*}
\end{small}%
}
Then we take expectation with respect to $\mathsf{q}$ and obtain
\begin{small}
\begin{equation*}
    \begin{aligned}
        &\E_\mathsf{q}\left\|\gdxy G(x_1,y;\zeta^{(0)})\left[\frac{Q}{l_{g,1}}\prod_{i=1}^{\mathsf{q}}\left(I - \frac{1}{l_{g,1}}\gdyy G(x_1,y;\zeta^{(i)})\right)\right]\gdy F(x_1,y;\xi) \right.\\
        &\left.\quad\quad\quad\quad\quad\quad\quad\quad\quad - \gdxy G(x_2,y;\zeta^{(0)})\left[\frac{Q}{l_{g,1}}\prod_{i=1}^{\mathsf{q}}\left(I - \frac{1}{l_{g,1}}\gdyy G(x_2,y;\zeta^{(i)})\right)\right]\gdy F(x_2,y;\xi)\right\|^2 \\
        &\leq \left(3Q^2(L_{y,0}+L_{y,1}\|\gdy f(x_1,y)\|)^2\|x_1-x_2\|^2 + 3\frac{l_{f,0}^2Q^2}{l_{g,1}^2}l_{g,2}^2\|x_1-x_2\|^2\right)\E_\mathsf{q}\left[\left(1-\frac{\mu}{l_{g,1}}\right)^{2\mathsf{q}}\right] \\
        &\quad\quad\quad\quad+ 3l_{f,0}^2Q^2\E_\mathsf{q}\left\|\prod_{i=1}^{\mathsf{q}}\left(I - \frac{1}{l_{g,1}}\gdyy G(x_1,y;\zeta^{(i)})\right) - \prod_{i=1}^{\mathsf{q}}\left(I - \frac{1}{l_{g,1}}\gdyy G(x_2,y;\zeta^{(i)})\right)\right\|^2 \\
        % &\leq \left(3Q^2(L_{y,0}+L_{y,1}\|\gdy f(x_1,y)\|)^2\|x_1-x_2\|^2 + 3\frac{l_{f,0}^2Q^2}{l_{g,1}^2}l_{g,2}^2\|x_1-x_2\|^2\right)\frac{l_{g,1}^2}{Q(2\mu l_{g,1}-\mu^2)} \\
        % &\quad\quad\quad\quad+ 3l_{f,0}^2Q^2\left\|\prod_{i=1}^{\mathsf{q}}\left(I - \frac{1}{l_{g,1}}\gdyy G(x_1,y;\zeta^{(i)})\right) - \prod_{i=1}^{\mathsf{q}}\left(I - \frac{1}{l_{g,1}}\gdyy G(x_2,y;\zeta^{(i)})\right)\right\|^2 \\
        &\leq \left(3Q^2(L_{y,0}+L_{y,1}\|\gdy f(x_1,y)\|)^2\|x_1-x_2\|^2 + 3\frac{l_{f,0}^2Q^2}{l_{g,1}^2}l_{g,2}^2\|x_1-x_2\|^2\right) \cdot \frac{l_{g,1}^2}{Q(2\mu l_{g,1}-\mu^2)} \\
        &\quad\quad\quad\quad+ 3l_{f,0}^2Q^2 \cdot \frac{l_{g,1}^2l_{g,2}^2Q}{(l_{g,1}-\mu)^2(2\mu l_{g,1}-\mu^2)}\|x_1-x_2\|^2 \\
        &\leq \frac{3Q}{2\mu l_{g,1}-\mu^2}\left(l_{g,1}^2(L_{y,0}+L_{y,1}l_{f,0})^2\|x_1-x_2\|^2 + l_{f,0}^2l_{g,2}^2\|x_1-x_2\|^2\right) + \frac{3l_{f,0}^2l_{g,1}^2l_{g,2}^2Q^3}{(l_{g,1}-\mu)^2(2\mu l_{g,1}-\mu^2)}\|x_1-x_2\|^2 \\
        &= \frac{3Q}{2\mu l_{g,1}-\mu^2}\left(l_{g,1}^2(L_{y,0}+L_{y,1}l_{f,0})^2 + l_{f,0}^2l_{g,2}^2 + \frac{l_{f,0}^2l_{g,1}^2l_{g,2}^2Q^2}{(l_{g,1}-\mu)^2}\right)\|x_1-x_2\|^2
    \end{aligned}
\end{equation*}
\end{small}%
Finally, taking expectation on both sides yields
\begin{equation*}
    \begin{aligned}
        \E_{\Bar{\xi}}&\|\gdf(x_1,y;\Bar{\xi}) - \gdf(x_2,y;\Bar{\xi})\|^2 
        \leq 2(L_{x,0}+L_{x,1}\|\gdx f(x_1,y)\|)^2\|x_1-x_2\|^2 \\
        &\quad\quad\quad\quad\quad\quad+ \frac{6Q}{2\mu l_{g,1}-\mu^2}\left(l_{g,1}^2(L_{y,0}+L_{y,1}l_{f,0})^2 + l_{f,0}^2l_{g,2}^2 + \frac{l_{f,0}^2l_{g,1}^2l_{g,2}^2Q^2}{(l_{g,1}-\mu)^2}\right)\|x_1-x_2\|^2. \\
    \end{aligned}
\end{equation*}
Since for any $y\in\R^{d_y}$, we have 
\begin{equation*}
    \begin{aligned}
        \|\gdx f(x_1,y) - \gdx f(x_1,y_1^*)\| 
        &\leq (L_{x,0}+L_{x,1}\|\gdx f(x_1,y_1^*)\|)\|y_1-y_1^*\| \\
        &\leq \left(L_{x,0}+L_{x,1}\left(\frac{l_{g,1}l_{f,0}}{\mu} + \|\gdphi(x_1)\|\right)\right)\|y_1-y_1^*\| \\
        &= \left(L_{x,0} + \frac{L_{x,1}l_{g,1}l_{f,0}}{\mu} + L_{x,1}\|\gdphi(x_1)\|\right)\|y_1-y_1^*\|, \\
    \end{aligned}
\end{equation*}
which yields
\begin{equation*}
    \begin{aligned}
        \|\gdx f(x_1,y)\| 
        &\leq \|\gdx f(x_1,y_1^*)\| + \left(L_{x,0} + \frac{L_{x,1}l_{g,1}l_{f,0}}{\mu} + L_{x,1}\|\gdphi(x_1)\|\right)\|y_1-y_1^*\| \\
        &\leq \frac{l_{g,1}l_{f,0}}{\mu} + \|\gdphi(x_1)\| + \left(L_{x,0} + \frac{L_{x,1}l_{g,1}l_{f,0}}{\mu} + L_{x,1}\|\gdphi(x_1)\|\right)\|y_1-y_1^*\| \\
        &= \left(\frac{l_{g,1}l_{f,0}}{\mu} + \left(L_{x,0} + \frac{L_{x,1}l_{g,1}l_{f,0}}{\mu}\right)\|y_1-y_1^*\|\right) + (1+L_{x,1}\|y_1-y_1^*\|)\|\gdphi(x_1)\|.
    \end{aligned}
\end{equation*}
Therefore, we conclude that
\begin{equation*}
    \begin{aligned}
        &\E_{\Bar{\xi}}\|\gdf(x_1,y;\Bar{\xi}) - \gdf(x_2,y;\Bar{\xi})\|^2 
        \leq 2(L_{x,0}+L_{x,1}\|\gdx f(x_1,y)\|)^2\|x_1-x_2\|^2 \\
        &\quad\quad\quad\quad\quad\quad\quad+ \frac{6Q}{2\mu l_{g,1}-\mu^2}\left(l_{g,1}^2(L_{y,0}+L_{y,1}l_{f,0})^2 + l_{f,0}^2l_{g,2}^2 + \frac{l_{f,0}^2l_{g,1}^2l_{g,2}^2Q^2}{(l_{g,1}-\mu)^2}\right)\|x_1-x_2\|^2 \\
        &\leq 2\left(L_{x,0} + L_{x,1}\left(\frac{l_{g,1}l_{f,0}}{\mu} + \left(L_{x,0} + \frac{L_{x,1}l_{g,1}l_{f,0}}{\mu}\right)\|y_1-y_1^*\|\right) + L_{x,1}(1+L_{x,1}\|y_1-y_1^*\|)\|\gdphi(x_1)\|\right)^2\|x_1-x_2\|^2 \\
        &\quad\quad\quad\quad\quad\quad\quad+ \frac{6Q}{2\mu l_{g,1}-\mu^2}\left(l_{g,1}^2(L_{y,0}+L_{y,1}l_{f,0})^2 + l_{f,0}^2l_{g,2}^2 + \frac{l_{f,0}^2l_{g,1}^2l_{g,2}^2Q^2}{(l_{g,1}-\mu)^2}\right)\|x_1-x_2\|^2 \\
        &\leq 4\left(L_{x,0} + L_{x,1}\left(\frac{l_{g,1}l_{f,0}}{\mu} + \left(L_{x,0} + \frac{L_{x,1}l_{g,1}l_{f,0}}{\mu}\right)\|y_1-y_1^*\|\right)\right)^2\|x_1-x_2\|^2 \\
        &\quad\quad\quad\quad\quad\quad\quad+ 4L_{x,1}^2(1+L_{x,1}\|y_1-y_1^*\|)^2\|\gdphi(x_1)\|^2\|x_1-x_2\|^2 \\
        &\quad\quad\quad\quad\quad\quad\quad+ \frac{6Q}{2\mu l_{g,1}-\mu^2}\left(l_{g,1}^2(L_{y,0}+L_{y,1}l_{f,0})^2 + l_{f,0}^2l_{g,2}^2 + \frac{l_{f,0}^2l_{g,1}^2l_{g,2}^2Q^2}{(l_{g,1}-\mu)^2}\right)\|x_1-x_2\|^2 \\
        &= (\Bar{L}_0^2+\Bar{L}_1^2\|\gdphi(x_1)\|^2)\|x_1-x_2\|^2, \\
    \end{aligned}
\end{equation*}
where we use the definition of $\Bar{L}_0$ and $\Bar{L}_1$ in the last equality. 
\end{proof}

\section{Additional Experimental Details}
\label{sec:detailexp_hc}
\paragraph{Hyerparameter setting.}
We tune the best hyperparameters for each algorithm, including upper-/lower-level step size, the number of inner loops, momentum parameters, etc. The upper-level learning rate $\eta_{up}$ and lower-level learing rate $\eta_{low}$ are tuned in the range of $[0.001, 0.1]$ for all the baselines on experiments of AUC maximization and data hyper-cleaning, the best $(\eta_{up}, \eta_{low})$ on \textbf{AUC maximization} are summarized as follows: StocBio: ($0.01, 0.001$), TTSA: $(0.005, 0.01)$, SABA: $(0.01, 0.005)$, MA-SOBA: $(0.01, 0.005)$, SUSTAIN: $(0.03, 0.01)$, VRBO: $(0.05, 0.01)$, BO-REP: $(0.001, 0.001)$, AccBO: $(0.005, 0.005)$.  The best learning rate on the experiment of \textbf{data hyper-cleaning} are summarized as follows: Stocbio: $(0.01, 0.002)$, TTSA: $(0.001, 0.01)$, SABA: $(0.05, 0.02)$, MA-SOBA: $(0.01, 0.01)$, SUSTAIN: $(0.05, 0.05)$, VRBO: $(0.1, 0.05)$, BO-REP: $(0.02, 0.01)$, AccBO: $(0.1, 0.1)$. Note that SUSTAIN decays its upper-/lower-level step size with epoch ($t$) by $\eta_{up}= \eta_{up} /(t+2)^{1/3}, \eta_{low}= \eta_{up} /(t+2)^{1/3}$, while other algorithms use a constant learning rate. The number for neumann series estimation in StocBiO and VRBO is fixed to 3, while it is uniformly sampled from $\{1,2,3\}$ in TTSA, SUSTAIN, and AccBO. In AUC maximization, AccBO uses Option I (Option II in data hyper-cleaning) to update the lower-level variable, and sets the Nestrov momentum parameter $\gamma=0.5$, the averaging parameter $\tau=0.5$ ($\gamma=0.1$ and $\tau=0.5$ in data hyper-cleaning).  In AUC maximization,  the batch size is set to be $32$ for all algorithms except VRBO, which uses larger batch size of $64$ (tuned in the range of $\{32, 64, 128, 256, 512\}$) at the checkpoint step and $32$ otherwise. In data hyper-cleaning, the batch size is set to be $128$ for all algorithms except VRBO, which uses larger batch size of $256$ (tuned in the range of $\{63, 128, 256, 512, 1024\}$) at the checkpoint step and $128$ otherwise. AccBO uses Option II in data hyper-cleaning, and the periodical update for low-level variable sets the iterations $N=3$ and update interval $I=2$.  Other hyperparameters setting keep the same in AUC maximization and data hyper-cleaning: The momentum parameter $\beta$ is fixed to $0.9$ in AccBO, MA-SOBA, BO-REP. The warm start steps for lower-level variable in AccBO is set to $3$. The number of inner loops for StocBio is set to $3$. BO-REP uses the periodical update for low-level variable, and set the iterations $N=3$ and the update interval $I=2$.

	%%%%%%%%%%%%%%%%%%%%%%%%%%%%%%%%%%%%%%%%%%%%%%%%%%%%%%%%%%%%
	\newpage
	\section*{NeurIPS Paper Checklist}

	\begin{enumerate}
		
		\item {\bf Claims}
		\item[] Question: Do the main claims made in the abstract and introduction accurately reflect the paper's contributions and scope?
		\item[] Answer: \answerYes{} %\answerTODO{} % Replace by \answerYes{}, \answerNo{}, or \answerNA{}.
		\item[] Justification: Every claim made in the abstract is specified a section of the paper, including algorithm design and analysis in \cref{sec:algorithmdesign} and experiments in \cref{sec:experiment}. %\justificationTODO{}
		
		\item {\bf Limitations}
		\item[] Question: Does the paper discuss the limitations of the work performed by the authors?
		\item[] Answer: \answerYes{} %\answerTODO{} % Replace by \answerYes{}, \answerNo{}, or \answerNA{}.
		\item[] Justification: We discussed the limitations of our work in \cref{sec:conclusion}. %\justificationTODO{}
		
		\item {\bf Theory Assumptions and Proofs}
		\item[] Question: For each theoretical result, does the paper provide the full set of assumptions and a complete (and correct) proof?
		\item[] Answer: \answerYes{} %\answerTODO{} % Replace by \answerYes{}, \answerNo{}, or \answerNA{}.
		\item[] Justification: We provide both assumptions and proofs in \Cref{sec:proof-single-level,sec:proof-sketch,sec:proof-main}. %\justificationTODO{}
		
		\item {\bf Experimental Result Reproducibility}
		\item[] Question: Does the paper fully disclose all the information needed to reproduce the main experimental results of the paper to the extent that it affects the main claims and/or conclusions of the paper (regardless of whether the code and data are provided or not)?
		\item[] Answer: \answerYes{} %\answerTODO{} % Replace by \answerYes{}, \answerNo{}, or \answerNA{}.
		\item[] Justification: The experimental details are fully specified in \Cref{sec:experiment} and \Cref{sec:detailexp_hc}. %\justificationTODO{}

		\item {\bf Open access to data and code}
		\item[] Question: Does the paper provide open access to the data and code, with sufficient instructions to faithfully reproduce the main experimental results, as described in supplemental material?
		\item[] Answer: \answerYes{} %\answerTODO{} % Replace by \answerYes{}, \answerNo{}, or \answerNA{}.
		\item[] Justification: The code and data are attached as a supplement with instructions for reproducibility. %\justificationTODO{}

		\item {\bf Experimental Setting/Details}
		\item[] Question: Does the paper specify all the training and test details (e.g., data splits, hyperparameters, how they were chosen, type of optimizer, etc.) necessary to understand the results?
		\item[] Answer: \answerYes{} %\answerTODO{} % Replace by \answerYes{}, \answerNo{}, or \answerNA{}.
		\item[] Justification: The experimental details are included in \Cref{sec:experiment} and \Cref{sec:detailexp_hc}.  %\justificationTODO{}
		
		\item {\bf Experiment Statistical Significance}
		\item[] Question: Does the paper report error bars suitably and correctly defined or other appropriate information about the statistical significance of the experiments?
		\item[] Answer: \answerNo{} %\answerYes{} %\answerTODO{} % Replace by \answerYes{}, \answerNo{}, or \answerNA{}.
		\item[] Justification: We only run once due to limited computational budget. %\justificationTODO{}
		
		\item {\bf Experiments Compute Resources}
		\item[] Question: For each experiment, does the paper provide sufficient information on the computer resources (type of compute workers, memory, time of execution) needed to reproduce the experiments?
		\item[] Answer: \answerYes{} %\answerTODO{} % Replace by \answerYes{}, \answerNo{}, or \answerNA{}.
		\item[] Justification: The hardware specification is described in \Cref{sec:experiment}. %\justificationTODO{}
		
		\item {\bf Code Of Ethics}
		\item[] Question: Does the research conducted in the paper conform, in every respect, with the NeurIPS Code of Ethics \url{https://neurips.cc/public/EthicsGuidelines}?
		\item[] Answer: \answerYes{} %\answerTODO{} % Replace by \answerYes{}, \answerNo{}, or \answerNA{}.
		\item[] Justification: We have read and conformed to the NeurIPS Code of Ethics. %\justificationTODO{}

		\item {\bf Broader Impacts}
		\item[] Question: Does the paper discuss both potential positive societal impacts and negative societal impacts of the work performed?
		\item[] Answer: \answerNA{} %\answerTODO{} % Replace by \answerYes{}, \answerNo{}, or \answerNA{}.
		\item[] Justification: This paper presents work whose goal is to advance the field of Machine Learning from algorithmic and theoretical aspects. We do not see any direct paths to negative societal impacts. %\justificationTODO{}
		
		\item {\bf Safeguards}
		\item[] Question: Does the paper describe safeguards that have been put in place for responsible release of data or models that have a high risk for misuse (e.g., pretrained language models, image generators, or scraped datasets)?
		\item[] Answer: \answerNA{} %\answerTODO{} % Replace by \answerYes{}, \answerNo{}, or \answerNA{}.
		\item[] Justification: Our paper does not involve the release of any data or models. %\justificationTODO{}
		
		\item {\bf Licenses for existing assets}
		\item[] Question: Are the creators or original owners of assets (e.g., code, data, models), used in the paper, properly credited and are the license and terms of use explicitly mentioned and properly respected?
		\item[] Answer: \answerYes{} %\answerTODO{} % Replace by \answerYes{}, \answerNo{}, or \answerNA{}.
		\item[] Justification: Our paper uses existing text classification datasets and are cited in \Cref{sec:experiment} and their licenses are mentioned. %\justificationTODO{}
		
		\item {\bf New Assets}
		\item[] Question: Are new assets introduced in the paper well documented and is the documentation provided alongside the assets?
		\item[] Answer: \answerNA{} %\answerTODO{} % Replace by \answerYes{}, \answerNo{}, or \answerNA{}.
		\item[] Justification: Our paper does not release new assets. %\justificationTODO{}
		
		\item {\bf Crowdsourcing and Research with Human Subjects}
		\item[] Question: For crowdsourcing experiments and research with human subjects, does the paper include the full text of instructions given to participants and screenshots, if applicable, as well as details about compensation (if any)? 
		\item[] Answer: \answerNA{} %\answerTODO{} % Replace by \answerYes{}, \answerNo{}, or \answerNA{}.
		\item[] Justification: This paper does not involve crowdsourcing or research with human subjects. %\justificationTODO{}
		
		\item {\bf Institutional Review Board (IRB) Approvals or Equivalent for Research with Human Subjects}
		\item[] Question: Does the paper describe potential risks incurred by study participants, whether such risks were disclosed to the subjects, and whether Institutional Review Board (IRB) approvals (or an equivalent approval/review based on the requirements of your country or institution) were obtained?
		\item[] Answer: \answerNA{} %\answerTODO{} % Replace by \answerYes{}, \answerNo{}, or \answerNA{}.
		\item[] Justification: This paper does not involve crowdsourcing or research with human subjects. %\justificationTODO{}
		
	\end{enumerate}

\end{document}